\documentclass[twoside,11pt]{article}

\usepackage[preprint]{jmlr2e}

\usepackage{amsmath}
\usepackage{bm}
\usepackage{mathtools}
\usepackage{amssymb}
\usepackage[dvipsnames]{xcolor}
\usepackage{algorithm2e}
\usepackage{graphicx}
\usepackage{tabulary}
\usepackage{booktabs}
\usepackage[utf8]{inputenc}
\usepackage{caption}
\usepackage{subcaption}
\usepackage{commath}

\usepackage[textsize=footnotesize]{todonotes}

\graphicspath{{./figures/}}

\ShortHeadings{Node Regression on Latent Position Models via local averaging}{Gjorgjevski, Keriven, Barthelmé, De Castro}
\firstpageno{1}

\newtheorem{assumption}[theorem]{Assumption}

\DeclareMathOperator\supp{supp}

\DeclareMathOperator{\Var}{Var}
\DeclareMathOperator{\Bias}{Bias}

\DeclareMathOperator{\ENW}{ENW}
\DeclareMathOperator{\NW}{\mathrm{NW}}
\DeclareMathOperator{\Unif}{Unif}
\DeclareMathOperator{\GNW}{\mathrm{GNW}}
\DeclareMathOperator{\FW}{\mathrm{FW}}
\DeclareMathOperator{\cMDS}{\mathrm{cMDS}}
\DeclareMathOperator{\CV}{\mathrm{CV}}
\DeclareMathOperator{\PTS}{\mathrm{PTS}}
\DeclareMathOperator{\MC}{\mathrm{MC}}
\DeclareMathOperator{\MSE}{\mathrm{MSE}}

\newcommand{\mmat}[1]{\bm{\mathrm{#1}}}
\newcommand{\mvec}[1]{\bm{\mathit{#1}}}

\newcommand\numberthis{\addtocounter{equation}{1}\tag{\theequation}}

\begin{document}

\title{Node Regression on Latent Position Random Graphs via Local Averaging}

\author{\name{Martin Gjorgjevski}\email{martin.gjorgjevski@grenoble-inp.fr} \\
       \addr{GIPSA-Lab, CNRS\\
       11 Rue de Mathématiques\\
       Grenoble, 38000, France}
       \AND\name{Nicolas Keriven}\email{nicolas.keriven@cnrs.fr}\\
       \addr{IRISA, CNRS\\
       263 av. du G\'en\'eral Leclerc\\
       Rennes, 35000, France}
       \AND\name{Simon Barthelmé}\email{simon.barthelme@grenoble-inp.fr} \\ 
       \addr{GIPSA-Lab, CNRS\\
       11 Rue de Mathématiques\\
       Grenoble, 38000, France} 
       \AND\name{Yohann De Castro}\email{yohann.de-castro@ec-lyon.fr}\\ 
       \addr{Institut Universitaire de France
       \\Institut Camille Jordan
       \\École Centrale de Lyon
       \\36 Avenue Guy de Collongue
       \\69134 Écully, France} 
       }
\editor{}
\maketitle

\begin{abstract}%
    Node regression consists in predicting the value of a graph label at a node, given observations at the other nodes. 
    To gain some insight into the performance of various estimators for this task, 
    we perform a theoretical study in a context where the graph is random. 
    Specifically, we assume that the graph is generated by a Latent Position Model, where each node of the graph has a latent position, and the probability 
    that two nodes are connected depend on the distance between the latent positions of the two nodes. 

    In this context, we begin by studying the simplest possible estimator for graph regression, 
    which consists in averaging the value of the label at all neighboring nodes. 
    We show that in Latent Position Models this estimator tends to a Nadaraya-Watson estimator in the latent space, 
    and that its rate of convergence is in fact the same. 

    One issue with this standard estimator is that it averages over a region consisting of all neighbors of a node, 
    and that depending on the graph model 
    this may be too much or too little. 
    An alternative consists in first estimating the ``true'' distances between the latent positions, 
    then injecting these estimated distances into a classical Nadaraya-Watson estimator. 
    This enables averaging in regions either smaller or larger than the typical graph neighborhood.  
    We show that this method can achieve standard nonparametric rates in certain instances even when 
    the graph neighborhood is too large or too small. 
    \end{abstract}

\section{Introduction}
Given an undirected graph with $n+1$ vertices  
and an adjacency matrix $\mmat{A}=[a_{i,j}]$ where all but the $(n+1)$-st node
have labels $y_i$, the node regression problem addresses the 
prediction of the (continuous valued) label $y_{n+1}$ of the remaining node\footnote{Our assumption that the regression node is numbered node $n+1$ 
is made purely out of notational convenience}.
This framework represents a simplified version of the so-called \emph{transductive
Semi-Supervised Learning (SSL)} problem on graphs~\citep{SSLgraphs}, 
where the labels of some nodes on a graph are known and the goal is to predict the labels of the other nodes 
in the same graph. While some SSL theoretical works focus on how to best exploit a \emph{large} quantity of unlabelled nodes, 
this simplified framework with only one unlabelled node is closer to classical Machine Learning, where generalization is computed 
(in expectation) for one unknown sample only. As we will see, this allows 
us to draw parallels between Graph Machine Learning (ML) and ``regular'' ML, and better isolate the effects of the graph structure on the problem. 
Despite the vastness of the Graph ML literature, this simplified framework has rarely been studied. While there are numerous works on 
unsupervised tasks such as node clustering
~\citep{athreya2017statistical,Abbe} and community detection~\citep{Fortunato_2010,com_detection_DCSBM}, supervised tasks have received less attention in this framework
To our knowledge, the only authors to study this framework are~\cite{Tang}, where they use the approximation of some kernel mapping as a node embedding in latent position graphs. Here we will study an even simpler, arguably more foundational approach: a simple $1$-hop averaging, mimicking the classical Nadaraya-Watson estimator in the graph context.

We consider the node regression problem with the goal of establishing generalization bounds in the context of random graphs. Specifically, we work with the \textbf{Latent Position Model} (LPM)~\citep{Hoff}, 
where each node $i$ is associated to a \textit{latent, unknown} variable $\mvec{x}_i\in Q\subseteq \mathbb{R}^d$. 
An edge between nodes $i$ and $j$ occurs with a probability that depends on the distance $||\mvec{x}_i-\mvec{x}_j||$ of the latent positions 
of nodes $i$ and $j$, and occurrences are independent conditionally on the latent positions.
Like often in the literature, our random graph model will essentially depend on a parameter that we call the \textbf{length-scale}~$h_g$.
This represents the ``typical scale'' of the model\footnote{It may be found under other names in the literature, e.g. ``kernel bandwidth''.}: intuitively speaking, nodes with latent positions at distance below length-scale $h_g$ are highly likely to be connected, and vice-versa.  
As mentioned above, in addition to the graph, we observe continuous labels $y_i$ on the first $n$ nodes of the graph. We will assume that the labels $y_i$ are noisy observations of some deterministic
function of the latent positions $\mvec{x}_i$, allowing for a direct comparison between node regression and classical (nonparametric) regression. 

\paragraph{Graphical Nadaraya Watson Estimator (GNW)}
In this paper, we focus on a simple (arguably, the simplest non-trivial) estimator for the missing label of node $n+1$, 
which computes the average of the labels over all of its neighbors, \ i.e.,\
\begin{equation}{\label{eq:meanagg}}
    \hat{y}_{n+1}=\frac{\sum_{j=1}^n{y_j a_{j,n+1}}}
    {\sum_{j=1}^n a_{j,n+1}}
\end{equation}
The estimator~\eqref{eq:meanagg} resembles the \textit{Nadaraya-Watson} (NW) estimator,
a fundamental estimator for nonparametric regression~\citep{Tsybakov}, 
but where the ``soft'' distance kernel $k(\mvec{x}_{n+1}, \mvec{x}_i)$ usually computed in NW is here replaced by the 
graph edges (recall that the $\mvec{x}_i$'s are unknown in our context). Therefore, we decide to call 
estimator~\eqref{eq:meanagg} the
\textbf{Graphical Nadaraya-Watson} (GNW) estimator. Note that, although we had to pick a name for the estimator \eqref{eq:meanagg} because to our knowledge it did not bear any particular name as a standalone estimator, the ``$1$-hop averaging'' principle is of course far from new and appears in many contexts 
(e.g.\ most recently as an aggregation function in Graph Neural Networks~\citep{GCN}).

As we will see, the hidden geometrical structure of the LPM allows us to study and compare 
the GNW estimator with the classical NW estimator with techniques from classical nonparametric regression. 
There is however one major difference between the two. In the classical nonparametric settings,
the statistician is free to select a parameter of NW known as the \emph{bandwidth} $\tau$, which sets the spatial scale over which the NW estimator performs averaging, e.g.\ through a kernel $\phi\left(\frac{\|\mvec{x}_i-\mvec{x}_j\|}{\tau}\right)$ 
with a decreasing function $\phi$. 
Thus, in the classical literature of nonparametric regression,
for a given regression function $f$ and a noise level $\sigma^2$, 
there exists an optimal bandwidth $\tau_{\star}$ that minimizes the
risk. On the contrary, \emph{there are no tunable parameters for GNW}. 
In our setting, the neighborhood ``size'' is imposed by the graph, which depends on a \textbf{length-scale} $h_g$ \emph{that is not user-chosen}.  
In fact, we will show in Sec.~\ref{sec:risk_bound_enw} that the risk of the GNW estimator with length-scale $h_g$ is surprisingly comparable to that of a NW estimator with 
\emph{fixed bandwidth} $\tau \coloneqq h_g$. In other words, replacing a fixed kernel with the corresponding Bernoulli variables does not (asymptotically) degrade the performance of the estimator.
As a  consequence, 
GNW is nearly optimal if the length-scale $h_g$ of the LPM is sufficiently close to the optimal bandwidth $\tau_{\star}$ for 
the corresponding nonparametric regression problem.
On the other hand, if this is not the case, the lack of tunable parameters for GNW is a major limitation: if $h_g$ is far away from $\tau_{\star}$, GNW will perform poorly.

\paragraph{Estimated Nadaraya Watson Estimator (ENW)}
In light of the previous discussion, there are two unfavorable scenarios for GNW -\ the \emph{under averaging regime} $h_g\ll\tau_{\star}$ (averaging is performed on a scale 
significantly smaller than the optimal one) and the \emph{over averaging regime} $h_g\gg\tau_{\star}$ (averaging performed on a scale too large relative 
to the underlying label).
In order to address this problem, in the second part of the paper
we study an estimator in two steps.
Since the most direct obstacle to choosing a bandwidth $\tau$ is arguably the fact that the pairwise distances $\|\mvec{x}_i-\mvec{x}_j\|$ are unknown, 
the first stage is a \textit{distance recovery algorithm} $\mathcal{A}$, that is, an algorithm that estimates the latent distances
based on the observed adjacency matrix $\mmat{A}$.
The second stage simply uses the approximated distances to compute the regular Nadaraya-Watson estimator with tunable bandwidth $\tau$ 
in the hope that, if the estimated distances are sufficiently close to the true ones, then the optimal bandwidth $\tau_{\star}$ 
(approximately known or, in practice, estimated by cross-validation) 
leads to a better result than the previous GNW\@. We call this estimation procedure the $\mathcal{A}$-\textbf{Estimated Nadaraya-Watson} 
estimator ($\mathcal{A}$-ENW), 
where the adjective \textit{estimated} refers to the distances between the latent positions. 
In Section~\ref{ENW_seciton}, our theoretical analysis will decouple these stages, allowing for separate treatment of the two problems.
Our contribution is in regards to the second step, that is, the stability of the Nadaraya-Watson estimator to perturbations of the distances between 
the design points. We provide a risk bound for $\mathcal{A}$-ENW in terms of the probability of the algorithm $\mathcal{A}$ 
to land within a prescribed noise level of the true positions. 
Concerning the algorithm $\mathcal{A}$ itself, we do not make any novel contribution \emph{per se} (as this is slightly out-of-scope here), but we build 
on some existing algorithms $\mathcal{A}$ from the literature and 
point out instances in which $\mathcal{A}$-ENW outperforms GNW both in the \emph{under averaging} ($h_g\ll\tau_{\star}$) and in the 
\emph{over averaging} ($h_g\gg\tau_{\star}$) regimes. In particular, in some instances 
we can achieve \emph{standard nonparametric rates} from classical nonparametric regression literature. 

\subsection{Background on the Latent Position Model}

A random graph consists of a vertex set $V=[n]$ and a \textit{random} edge set $\mathcal{E}\subseteq V\times V$. 
The study of random graphs begins with the Erdös-Renyi model~\citep{Erdos_renyi_graph}, where each edge occurs independently with probability $0\leq p\leq 1$.
However, real world networks do not have the same distributional properties as the Erdös-Renyi model, for example it has been 
observed that the distribution of degrees in real world networks follows a power law~\citep{Albert}. Such observations prompted 
research into models that can better capture the topology of real world networks, yielding richer models of random graphs. 
One such model for studying community structure is called the Stochastic Block Model~\citep{Holland1983StochasticBF}. Here nodes belong to latent communities 
and the probability that two nodes $i,j$ are linked depends only on the communities of the nodes $C_i,C_j$. This model has been studied extensively from a theoretical 
point of view~\citep{Abbe}. Another popular model is the random geometric graph~\citep{PENROSE}, 
where nodes are associated to \textit{latent positions}~$\mvec{x}_i$. 
Here nodes $i$ and $j$ are linked if 
the distance between their latent positions $\mvec{x}_i$ and $\mvec{x}_j$ is within a prescribed threshold $r>0$, i.e. 
if $||\mvec{x}_i-\mvec{x}_j||_2<r$. 
These two models can be unified in the Latent Position Model~\citep{Hoff}. In this model each node $i$ is associated to a latent position\footnote{these positions
may be deterministic or i.i.d.\ draws from some distribution} $\mvec{x}_i\in Q$, where $Q$ is the \emph{latent space}. 
The probability of having an edge between nodes $i$ and $j$ is then given by 
\begin{equation}\label{eq:general_link_function}
   \mathbb{P}\left(a_{i,j}=1 |\mvec{x}_i,\mvec{x}_j \right) = k(\mvec{x}_i,\mvec{x}_j)
\end{equation}
where $k\colon Q\times Q\to [0,1]$ is (in general) a nonparametric \textit{link} function.
\begin{figure}[!ht]
    \centering
    \includegraphics[width=0.495\textwidth]{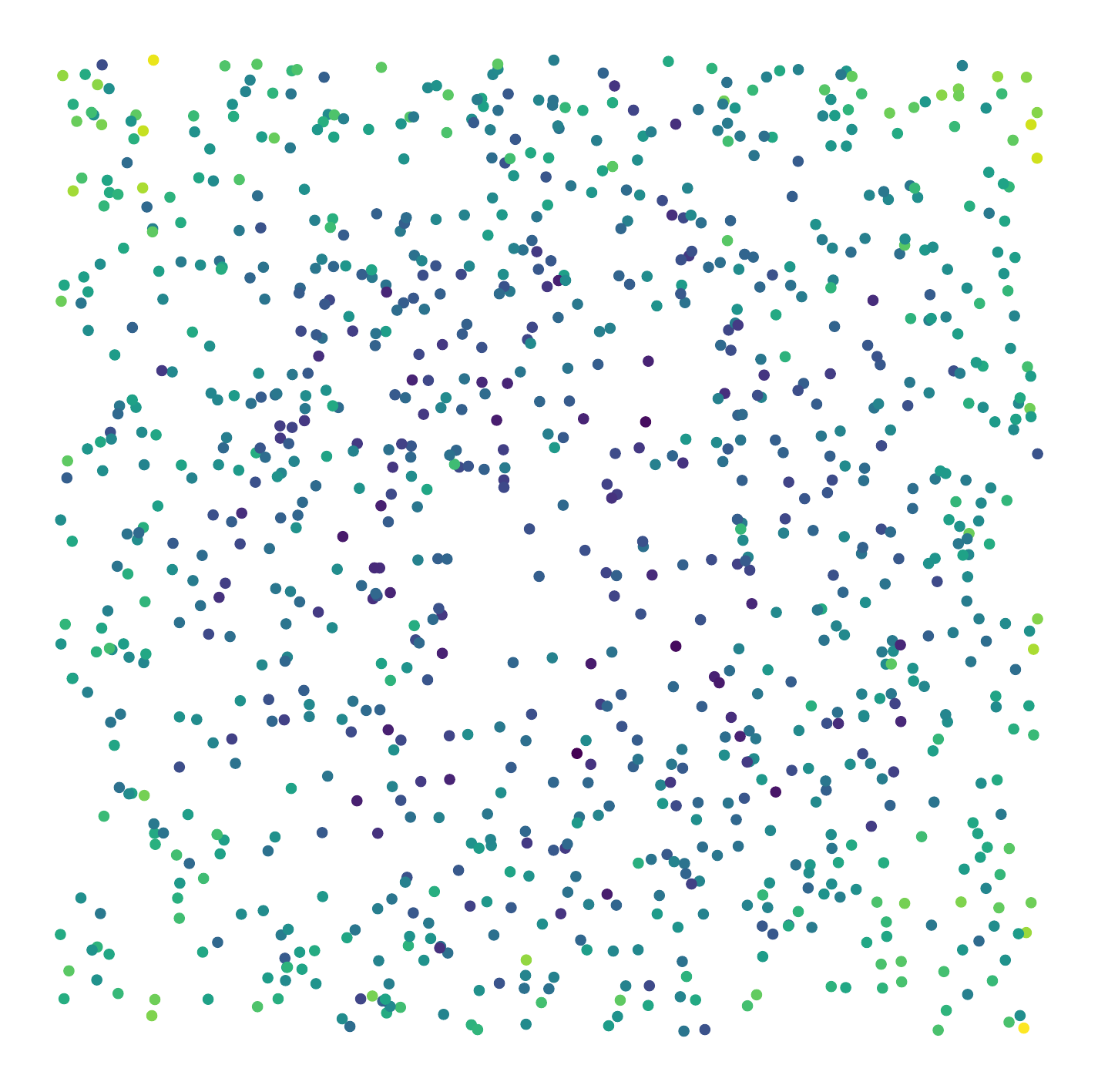}
    \includegraphics[width=0.495\textwidth]{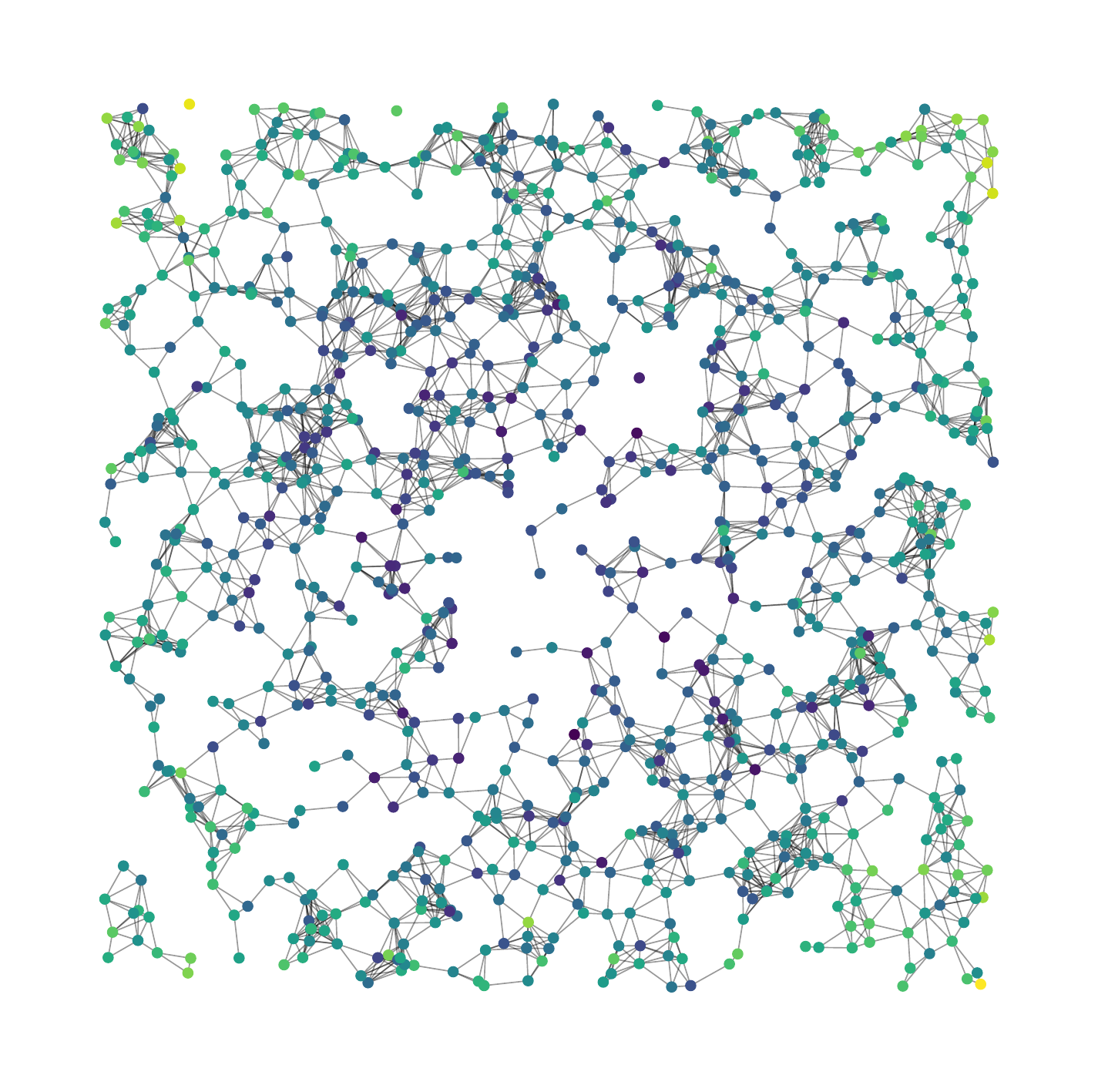}
    \caption{Sampling a LPM\@: Left --- generating uniformly 1000 latent positions on ${[-1,1]}^2$.
    Left: Latent positions. Right: generating a random geometric graph with $h_g=0.1$.
    The color represents the labels --- brighter colors correlate with higher values}
\end{figure}

To relate the node regression problem with the classical theory of (non-parametric) regression, we will suppose that $Q\subseteq\mathbb{R}^d$ and that the link function has the following shape 
\begin{equation}\label{eq:link_function}
    \mathbb{P}\left(a_{i,j}=1| \mvec{x}_i,\mvec{x}_j\right) = k(\mvec{x}_i,\mvec{x}_j) = \alpha K\left(\frac{||\mvec{x}_i-\mvec{x}_j||}{h_g}\right)    
\end{equation}
Here $0<\alpha\leq 1$, $h_g>0$ and $K\colon[0,\infty)\to [0, 1]$ with $K(0)=1$. 
When we observe a LPM graph with link function~\eqref{eq:link_function}, we do not assume to know anything but the graph itself:
we \textbf{do not have access} to the latent positions $\mvec{x}_i,i\in [n]$, nor to the parameters $\alpha,h_g$. The  
function~$K$ is in general assumed to be decreasing and we will add the assumption that it is compactly supported and non-vanishing in a neighborhood 
of $0$ (see Assumption~\ref{ass:K_1}). 
In the literature, the parameters $\alpha,h_g$ are generally used to model \emph{sparsity} in random graphs, 
that is, the relative number of edges with respect to $n$, 
and may thus depend on the number of nodes~$n$. 
When $\alpha, h_g$ are fixed, the expected number of edges is in $\mathcal{O}(n^2)$, and the random graph is said to be \emph{dense}.  
When the number of edges is in $\mathcal{O}(n)$, the graph is \emph{sparse}. In-between those two rates, the graph is \emph{relatively sparse}. 
Most real-world graphs are observed to be relatively sparse or sparse. To model this, taking a decreasing multiplicative factor 
$\alpha$ when~$n$ increases is e.g.\ more common in the SBM literature, while using a decreasing length-scale $h_g$ is more common in the 
geometric graph literature. For the GNW, we take both parameters $h_g$ and $\alpha$ into account, 
and we show that our results hold for any relatively sparse graph, as soon as the expected degrees grow with the number of nodes, 
\emph{even if this growth is arbitrarily slow}. For $\mathcal{A}$-ENW, we fix $\alpha=1$ in order to facilitate the analysis; we 
leave the case of $\alpha$ decreasing with $n$ for future work.

\subsection{Framework and Notation}\label{sec:LPnotations}

We denote the indicator of a set $S$ by $\mathbb{I}\left[S\right]$, the Lebesgue measure on $\mathbb{R}^d$ by $m$ and the volume of the unit ball in $\mathbb{R}^d$ by $v_d$.
The standard Euclidean distance between $\mvec{x},\mvec{z}\in\mathbb{R}^d$ is denoted by $||\mvec{x}-\mvec{z}||$. We use the comparisons operators 
$a_n\lesssim b_n$ when there exists some constant $c>0$ independent of $n$ and $h_g$ such that $ca_n\leq b_n$, and we use $\gtrsim$ in an analogous manner.
We introduce the notation 
\begin{equation}\label{eq:latent_positions}
    \mmat{X}_n=\left[\mvec{x}_1,\dots,\mvec{x}_n\right]\in\mathbb{R}^{d\times n}
\end{equation} 
to denote the matrix that contains the latent positions of nodes $1$ through $n$ (the labeled nodes), and 
\begin{equation}
    \mmat{X}_{n+1}=\left[\mmat{X}_n,\mvec{x}_{n+1}\right]=\left[\mvec{x}_1,\ldots,\mvec{x}_n,\mvec{x}_{n+1}\right]\in\mathbb{R}^{d\times(n+1)}
\end{equation}
to denote the extended matrix that contains the latent position of the regression node, node~$(n+1)$.
The observed label $\mvec{y}={\left[y_1,\ldots,y_n\right]}^t$ is given by 

\begin{equation}
    y_i = f(\mvec{x}_i)+\epsilon_i
\end{equation}
where $f\colon Q\to \mathbb{R}$ is a regression function belongs to a Hölder class (See Assumption~\ref{F_1}) and   
\begin{equation}\label{eq:additive_noise}
    \mvec{\epsilon}={\left[\epsilon_1,\dots,\epsilon_n\right]}^t
\end{equation}
is label additive noise vector with independent entries of finite variance (See Assumption~\ref{ass:additive_noise}).
An LPM graph can be represented by the $(n+1)\times(n+1)$ adjacency matrix $\mmat{A}={\left[a(\mvec{x}_i,\mvec{x}_j)\right]}_{1\leq i, j \leq n+1}$,
where the indicator of an edge between nodes $i$ and $j$ is given by 
\begin{equation}\label{eq:adjacency_with_U}
    a(\mvec{x}_i,\mvec{x}_j)=\mathbb{I}\left[U_{i,j}\leq k(\mvec{x}_i,\mvec{x}_j)\right]
\end{equation}
where
\begin{equation}\label{eq:random_edges_x}
    \mmat{\mathcal{U}}={\left[U_{i,j}\right]}_{1\leq i,j\leq (n+1)}
\end{equation}
are uniform variables on $[0,1]$, for $i\neq j$ and $U_{i,i}=c>1$ (by convention, this prevents self edges). The entries of $\mmat{\mathcal{U}}$ are independent for distinct pairs $(i_1,j_1)$ and $(i_2,j_2)$ with $i_1<j_1$ and $i_2<j_2$,
and satisfying the symmetric constraint $U_{i,j}=U_{j,i}$ (imposed by the symmetry of the adjacency matrix $\mmat{A}$). 
The matrix $\mmat{\mathcal{U}}$ is also
independent from $\mmat{X}_{n+1}$ and $\mvec{\epsilon}$.
Throughout the paper, we will assume that the latent positions are either fixed or they are i.i.d.\ samples with density $p$ with support $Q\subseteq\mathbb{R}^d$. In the latter case,
the local edge density and the local expected degree  at a point $\mvec{x}\in\mathbb{R}^d$ are given by
\begin{equation}\label{eq:local_degree}
    c(\mvec{x})=\int_{\mathbb{R}^d} k(\mvec{x},\mvec{z})p(\mvec{z})dz, \hspace{20pt} d(\mvec{x})=nc(\mvec{x})
\end{equation}
respectively. For $\mvec{x}\in\mathbb{R}^d$, 
we define the operator $S(\cdot,\mvec{x})$ on the set of bounded and measurable functions by
\begin{equation}\label{eq:operator_S}
    S(f,\mvec{x})=\begin{cases}
    \frac{\int f(\mvec{z})k(\mvec{x},\mvec{z})p(\mvec{z})dz}{c(\mvec{x})} \quad &\text{if}\, c(\mvec{x})>0\\
    0 \quad &\text{otherwise}\\
\end{cases}
\end{equation}
Furthermore, we denote by
\begin{equation}\label{eq:latent_distances}
\delta_i = ||\mvec{x}_i-\mvec{x}_{n+1}||
\end{equation}
the distance between the $i$th latent variable and the one of the node of interest $n+1$.

\subsection{Differences between Classical Nonparametric Regression and Node Regression in LPMs}
\subsubsection{Risks}
    The (nonparametric) regression problem in its simplest form can be stated as estimating a \textit{regression}
function $f\colon Q\to\mathbb{R}$ based on a sample  $\mathcal{D}\coloneqq (\mmat{X}_{n},\mvec{y})=\{(\mvec{x}_1,y_1),\dots,(\mvec{x}_n,y_n)|\mvec{x}_i\in Q, y_i\in\mathcal{Y}\subseteq\mathbb{R}\}$
where $\mvec{x}_i$ are either deterministic points from a domain $Q\subseteq \mathbb{R}^d$ or are i.i.d.\ samples from a distribution with density $p$, supported on $Q\subseteq \mathbb{R}^d$ and 
\begin{equation}\label{eq:labels}
        y_i=f(\mvec{x}_i)+\epsilon_i
\end{equation}
with $\epsilon_i$ i.i.d.\ centered, finite variance noise variables. The nonparametric literature uses the nomenclature of  
fixed and random design for the case of deterministic samples and random samples $\mvec{x}_i$, respectively. 
An estimator $\hat{f}=\hat{f}_{\mathcal{D}}$ is any \textit{random (measurable)} function 
$\hat{f}\colon Q\to\mathbb{R}$ that depends on the data~$\mathcal{D}$. 
Traditionally, assuming observations of the form~\eqref{eq:labels}, in the case of fixed design the quality of the 
estimator $\hat{f}$ is measured by the \textit{risk}
\begin{equation*}
    \mathcal{R}\big(\hat{f}(\mvec{x}_{n+1}),f(\mvec{x}_{n+1})\big)=\mathbb{E}_{\mvec{\epsilon}}\left[{\big(\hat{f}(\mvec{x}_{n+1})-f(\mvec{x}_{n+1})\big)}^2\right]
\end{equation*}
Under the random design assumption there are two notions of risks: \emph{point-wise} and \emph{global}. 
For a (non-random) point $\mvec{x}_{n+1}\in Q$, the \textit{point-wise risk} is given by
\begin{equation}\label{eq:point-wise_risk}
    \mathcal{R}\big(\hat{f}(\mvec{x}_{n+1}),f(\mvec{x}_{n+1})\big)=\mathbb{E}_{\mmat{X}_n,\mvec{\epsilon}}\left[{\big(\hat{f}(\mvec{x}_{n+1})-f(\mvec{x}_{n+1})\big)}^2\right]
\end{equation}
which captures statistical information about the particular point $\mvec{x}$ of the domain $Q$. 
The \emph{global risk} is given by
\begin{equation}\label{eq:integrated_risk_nw}
    \mathcal{R}\big(\hat{f},f\big)=\mathbb{E}_{\mmat{X}_{n+1},\mvec{\epsilon}}\left[{\big(\hat{f}(\mvec{x}_{n+1})-f(\mvec{x}_{n+1})\big)}^2\right]
\end{equation}
where $\mvec{x}_{n+1}$ is an out of sample example not used in the training process. Note that the global and the point-wise risk are related by the 
equation 
\begin{equation}\label{eq:global-to-local-risk}
    \mathcal{R}\big(\hat{f},f\big)=\int \mathcal{R}\big(\hat{f}(\mvec{x}),f(\mvec{x})\big)p(\mvec{x})d\mvec{x}
\end{equation}
The global risk~\eqref{eq:integrated_risk_nw} can be interpreted as the average of the point-wise risk~\eqref{eq:point-wise_risk} over the data distribution $p$.
    
In contrast, when considering a node regression estimator, we will consider the risk taken with respect to the randomness of the edges, the additive noise 
on the label, and the latent positions (when they are treated as random variables). 
In other words if $\hat{f}$ is a node regression estimator (i.e.\ it depends on $\mathcal{D}_g = (\mmat{A},y)$,  
the adjacency matrix $\mmat{A}$ and the graph label $\mvec{y}$),
we define the pointwise risk
\begin{equation}\label{eq:node_reg_risk_pointwise}
    \mathcal{R}_{g}(\hat{f}(\mvec{x}_{n+1}),f(\mvec{x}_{n+1}))=\mathbb{E}_{\mmat{\mathcal{U}},\mvec{\epsilon}}\left[{\left(\hat{f}(\mvec{x}_{n+1})-f(\mvec{x}_{n+1})\right)}^2\right]
\end{equation}
and, in the case of random latent positions $\mmat{X}_{n+1}$, the global risk, as 
\begin{equation}\label{eq:node_reg_risk_global}
    \mathcal{R}_{g}(\hat{f},f)=\mathbb{E}_{\mmat{X}_{n+1},\mmat{\mathcal{U}},\mvec{\epsilon}}\left[{\left(\hat{f}(\mvec{x}_{n+1})-f(\mvec{x}_{n+1})\right)}^2\right]
\end{equation}
In Equation~\eqref{eq:node_reg_risk_global} the expectation is taken as before over the latent positions $\mmat{X}_{n+1}$ (which include the latent position of the regression node), the additive label noise $\mvec{\epsilon}$, 
but also \emph{the random matrix} $\mmat{\mathcal{U}}$ which, we recall, is used along with $\mmat{X}_{n+1}$ to generate the random adjacency matrix $\mmat{A}$ through~\eqref{eq:adjacency_with_U}. 
It is often convenient to write the expectation this way instead of the conditional expectation $\mathbb{E}_{\mmat{X}_{n+1}}\mathbb{E}_{\mmat{A}|\mmat{X}_{n+1}}$, since $\mmat{X}_{n+1}$ and $\mmat{\mathcal{U}}$ 
are independent. Sometimes we will adopt the shortcut $\mvec{x} = \mvec{x}_{n+1}$ in the notation above, with the understanding that random edges ``link'' the point $\mvec{x}$ with all the others using the last column of~$\mmat{\mathcal{U}}$ as before: $a(\mvec{x}_i,\mvec{x}) = \mathbb{I}[U_{i,n+1} \leq k(\mvec{x}_i,\mvec{x})]$. 
Again, the risk~\eqref{eq:node_reg_risk_global} is the pointwise risk~\eqref{eq:node_reg_risk_pointwise} integrated with respect to $\mvec{x}_{n+1}$.

\subsubsection{Estimators}
\paragraph{Nadaraya-Watson} A classical approach for the regression problem is the weighted average
\textit{Nadaraya-Watson} estimator, for which a modern theoretical analysis may be found in~\citep{Tsybakov,Gyofri}
\begin{equation}\label{eq:NW}
    \hat{f}_{\NW,\tau}(\mvec{x})=\begin{cases}
        \frac{\sum_{i=1}^n y_i \phi\left(\frac{||\mvec{x}-\mvec{x}_i||}{\tau}\right)
        }{\sum_{i=1}^n \phi\left(\frac{||\mvec{x}-\mvec{x}_i||}{\tau}\right)} \quad &\text{if}\, \sum_{i=1}^n \phi\left(
        \frac{||\mvec{x}-\mvec{x}_i||}{\tau}        
        \right)\neq 0\\
        0 \quad &\text{otherwise}\\
    \end{cases}
\end{equation}
Here, $\phi\colon\mathbb{R}\to\mathbb{R}$ is called a \textit{kernel} function. Some popular choices for $\phi$ are the \emph{rectangular} kernel 
($\phi\left(z\right)=\mathbb{I}\left[|z|\leq 1\right]$), \emph{Gaussian} kernel ($\phi\left(z\right)=e^{-z^2}$), the 
\emph{sinc} kernel $(\phi\left(z\right)=\frac{\sin(\pi z)}{\pi z})$. Other common choices are discussed in~\citep{Tsybakov}. 
The parameter $\tau>0$ is called the \textit{bandwidth} and it controls the scale on which the data is being averaged. 
This parameter needs to be chosen carefully, as too small values of $\tau$ produce estimates of high variance (overfitting), while too large values of $\tau$ give highly biased estimators (underfitting), 
an instance of the \textit{Bias-Variance tradeoff}, a well known phenomenon in statistics and classical machine learning. 
The Nadaraya-Watson estimator is a \textit{local} estimator, 
in that a prediction for a point $\mvec{x}$ will depend on the distances of the samples $\mvec{x}_i\in\mathcal{D}$  
from the point of interest $\mvec{x}$; NW is averaging the observations $\{y_i| \mvec{x}_i\in \mathcal{D}\}$, 
giving higher weights to observations $y_i$ with covariates $\mvec{x}_i$ close to the point $\mvec{x}$.
Therefore, the NW is a reasonable estimator for regression functions that vary smoothly across the domain $Q$. 
More precisely, a natural class for the regression function $f$ is the \textit{Hölder class} $\Sigma(a,L)$~\citep{Tsybakov} given by
\begin{equation}\label{eq:holderclass}
    \Sigma(a,L)=\Big\{g\colon\mathbb{R}^d\to\mathbb{R}|\hspace{5pt} \textit{for all} \hspace{5pt} \mvec{x},\mvec{z}\in\mathbb{R}^d, \hspace{5pt} |g(\mvec{x})-g(\mvec{z})|\leq L||\mvec{x}-\mvec{z}||^{a}\Big\}\,,
\end{equation} 
for $a\in[0,1],L>0$. The larger the parameter $a$, the smoother the function is. Indeed, for $a=1$, one recovers the class of \textit{Lipschitz} functions.

\paragraph{Minimax rates of NW}
The standard nonparametric rate in terms of the bandwidth is 
\begin{equation}\label{eq:NW_rate}
    \mathcal{R}\left(\hat{f}_{\NW,\tau},f\right)\leq C_1\tau^{2a}+\frac{C_2}{n\tau^{d}}
\end{equation}
where $C_1,C_2>0$ depend on the variance of the label $\sigma^2$, and the Hölder constant $L$, but 
not on the sample size $n$.
In the large $n$ regime, optimizing this rate in terms of $\tau$, one gets that
\begin{equation}\label{eq:optimal_nw_rate}
    \inf_{\tau>0}\mathcal{R}\left(\hat{f}_{\NW,\tau},f\right)\leq C_{\star}n^{-\frac{2a}{2a+d}}
\end{equation}
obtained for bandwidth $\tau_{\star}$ of order $n^{-\frac{1}{2a+d}}$~\citep{Tsybakov,Gyofri}.
It can also be shown that the rate~\eqref{eq:optimal_nw_rate} is optimal in a minimax sense for the Hölder class $\Sigma(a,L)$~\citep{Tsybakov}, i.e\
given the prior that the regression function $f$ belongs in the Hölder class $\Sigma(a,L)$, 
asymptotical improvements are only possible on the multiplicative constant $C_{\star}$, but not on the rate~\eqref{eq:optimal_nw_rate}. 
In presence of additional smoothness of the regression function~$f$, one can improve upon 
the rate~\eqref{eq:optimal_nw_rate}.

\paragraph{Graphical NW} In this paper, we do \emph{not} observe the positions $\mvec{x}_i$, but we observe instead a random graph 
with $n+1$ nodes sampled according to a LPM with kernel function~\eqref{eq:link_function}. 
We assume that for all 
but the last node there is a label of the form~\eqref{eq:labels}. 
Denoting $\mvec{x} = \mvec{x}_{n+1}$ for convenience, we introduce the (random) \textit{empirical} degree
\begin{equation}
\label{empirical_degree}
    \hat{d}(\mvec{x})=\sum_{i=1}^n a(\mvec{x},\mvec{x}_i)
\end{equation}
where we recall that $a(\mvec{x},\mvec{x}_i)$ are the random edges between the nodes $n+1$ and $i$ taken as~\eqref{eq:adjacency_with_U}.
With this notation, the GNW estimator~\eqref{eq:meanagg} is given by
\begin{equation}
\label{gnw_def}
\hat{f}_{\GNW}(\mvec{x})=\begin{cases}
    \frac{1}{\hat{d}(\mvec{x})}\sum_{i=1}^n y_i a(\mvec{x},\mvec{x}_i) \quad &\text{if}\hspace{3pt} \hat{d}(\mvec{x}) > 0\\
    0 \quad &\text{otherwise}\\
\end{cases}
\end{equation}
Note that the only edges of interest for the Graphical Nadaraya-Watson estimator are those adjacent to node $(n+1)$, 
so in the discussion of GNW we will not be concerned with the remaining edge variables $a(\mvec{x}_i,\mvec{x}_j)$ for $1\leq i,j\leq n$.
Since the edges $a(\mvec{x},\mvec{x}_i)$ are Bernoulli variables with expectation $\alpha K\left(\frac{||\mvec{x}-\mvec{x}_i||}{h_g}\right)$,
GNW can be considered as a quite noisy version of NW, where the true weights $\alpha K\left(\frac{||\mvec{x}-\mvec{x}_i||}{h_g}\right)$ are replaced 
by $1$ with probability $\alpha K\left(\frac{||\mvec{x}-\mvec{x}_i||}{h_g}\right)$ and by $0$ with complementary probability. 
Given the potentially high variance introduced by such nonlinear perturbations, 
it is somewhat surprising that GNW achieves the NW-rate~\eqref{eq:NW_rate} 
for $\tau\coloneqq h_g$, as we will show in Section~\ref{sec:GNW_concentration}.

\paragraph{Estimated NW} As mentioned in the introduction, in order to address some shortcomings of GNW, 
we also study a broad family of node regression estimators that are built using
a plug-in estimator of the latent distances. Specifically, we use an estimator
of either latent distances or latent positions, then plug those estimates in a classical NW estimator. In addition to the observed 
label $\mvec{y}$ on the first $n$ nodes~\eqref{eq:labels}, namely $\mvec{y}={\left[y_1,\dots,y_n\right]}^t$ 
and the adjacency matrix $\mmat{A}$, we also assume that there exist an algorithm $\mathcal{A}$ that takes in the observed graph with 
adjacency matrix $\mmat{A}$ as an input\footnote{potentially depending on some hyperparameters.}
and outputs a vector $\mvec{\tilde{\delta}}=\left[\tilde{\delta}_1,\dots,\tilde{\delta}_{n}\right]$, an \textit{estimation of the distances} 
$\mvec{\delta} = \left[ \delta_1,\dots,\delta_n \right]$ where $\delta_i$ is given by~\eqref{eq:latent_distances}. The $\mathcal{A}$-Estimated Nadaraya-Watson 
is given by 
\begin{equation}
    \hat{f}^{\mathcal{A}}_{\ENW,\tau}(\mvec{x}_{n+1}) = \begin{cases}
       \displaystyle\frac{\sum_{i=1}^n y_i\phi\left(\frac{\tilde{\delta}_i}{\tau}\right)}{\sum_{i=1}^n\phi\left(\frac{\tilde{\delta}_i}{\tau}\right)}  
 \quad &\text{if}\quad \, \displaystyle \sum_{i=1}^n\phi\left(\frac{\tilde{\delta}_i}{\tau}\right)>0\\
   0 \quad &\text{otherwise}\\
\end{cases}
\end{equation}
where $\phi\colon[0,\infty)\to\mathbb{R}$ and $\tau>0$ are \emph{user chosen}. The theoretical analysis of $\mathcal{A}$-ENW is conducted in 
Sec.~\ref{ENW_seciton}.

\subsection{Related work}

The node regression problem has been thoroughly studied in non-random edge graphs.~\cite{kovac2009regression} studies a penalized least square method,
 where the penalization is in terms of $l_1$ norm 
over the edges of the graph. The authors in~\citep{1326716} provide
generalization bounds by assuming that edges are weighted and depend on the latent variables, the generalization result is 
over the randomness of the latent positions. Our analysis includes 3 sources of randomness:
(potential) randomness of latent positions $\mmat{X}_{n+1}$, randomness of edges $\mmat{\mathcal{U}}$, and additive noise randomness $\mvec{\epsilon}$.
As far as we know, one of the only work to study this framework is~\citep{Tang}, where the authors draw connections with kernel methods and Reproducing Kernel Hilbert Spaces (RKHS).

On random graphs, there is a significant literature on unsupervised learning, e.g.\ for
clustering in SBMs~\citep{Snijders,Abbe}. As large graphs in the real world tend to be sparse~\citep{Albert}, 
a significant effort in the community detection literature is dedicated to
understanding statistical properties of graphs with low expected degrees~\citep{Oliviera,Lei_2015,Levina-Vershynin}.
Another vast line of work in LPMs studies algorithms for recovering latent positions or
latent distances~\citep{Arias-Castro,nearperfect,Giraud-Verzelen,dani}. 
We will leverage some of the results established in this literature to demonstrate that $\mathcal{A}$-ENW achieves 
the standard nonparametric rate~\eqref{eq:optimal_nw_rate} in certain under-averaging $(h_g\ll \tau_{\star})$ and over-averaging
$(h_g\gg \tau_{\star})$ regimes.

\subsection{Outline}
 In Section~\ref{sec:GNW_concentration} we show that under classical assumptions on the regression function $f$ and the kernel $K$, the 
 \textit{Graphical Nadaraya-Watson} (GNW) estimator achieves the same risk rates 
 as those of the \textit{Nadaraya-Watson} estimator. A precise formulation of this statement can be found in Theorems~\ref{thm:final_result} and~\ref{thm:final_result_holder}. 
 We follow an approach inspired by the classical bias-variance decomposition,
 but we use instead two quantities which we call \textit{bias and variance
   proxies}, which are close but not equal to the exact bias and variance. The
 bias and variance proxies have simpler expressions and are easier to study.  
Under minimal assumptions on the additive noise, we show that the \textit{variance proxy} of GNW is inversely proportional to the 
expected degree; a precise statement is in Sec.~\ref{sharp_variance_bounds}. In
Sec.~\ref{sec:bias} we study the bias proxy. To do so, we require more
assumptions on the kernel $K$ as well as the distribution of latent positions $p$. Finally, in Sec.~\ref{irisk}, we conclude the GNW analysis by combining the 
bias and variance analysis.

In Section~\ref{ENW_seciton} we study the two-stage estimator that consists in \textit{estimating} the latent distances by some user-chosen algorithm $\mathcal{A}$ and 
then plugging those estimated distances into the classical NW (often with a bandwidth parameter $\tau_{\CV}$ that is chosen by cross-validation by the user).
Our analysis treats these two steps separately. 
In Sec~\ref{sec:risk_bound_enw} we show that Nadaraya-Watson with bandwidth $\tau$ and maximum perturbation of the distances $\Delta\geq 0$ preserves the classical 
rate~\eqref{eq:NW_rate} as long as $\Delta\lesssim\tau$ (see Theorem~\ref{thm:perturbed_nw_thm}) and, building on that result,    
we prove a bound on the risk of $\mathcal{A}$-ENW in terms of the \emph{probability
of success} of the Algorithm $\mathcal{A}$ (see Theorem~\ref{thm_enw_fixed_design}). In Sec.~\ref{sec:position_estimation_algos} 
we give several examples of existing literature on distance estimation algorithms $\mathcal{A}$ 
in LPMs to derive risk bounds on $\mathcal{A}$-ENW\@. We point out certain \emph{under-averaging} and \emph{over-averaging} regimes in which 
distance recovery can yield optimal nonparametric rates for $\mathcal{A}$-ENW\@. 

In Section~\ref{simulations} we corroborate on our theoretical results by numerical experiments. We consider two 
simple position recovery algorithms 
that achieve (sometimes only empirically) optimality in the under-averaging and over-averaging regimes respectively.

\section{The Graphical Nadaraya Watson (GNW) estimator}\label{sec:GNW_concentration}

In this section we adopt the random design setting, i.e.\ we assume that 
the latent positions~$\mmat{X}_{n+1}$ are i.i.d.\ samples with density $p$ supported on $Q\subseteq\mathbb{R}^d$.
We will work conditionally on node $n+1$ having latent position $\mvec{x}_{n+1}=\mvec{x}$.
The goal of this section is to provide a bound on the \textit{global risk of GNW}~\eqref{eq:node_reg_risk_global}.
The approach we take is thus to provide an upper bound of~\eqref{eq:node_reg_risk_pointwise} and then to integrate it in order 
to obtain a bound on the global risk~\eqref{eq:node_reg_risk_global}. There will often be a need to take expectations 
with respect to the random matrix $\mmat{\mathcal{U}}$ that generates the random edges~\eqref{eq:random_edges_x}, the latent positions 
$\mmat{X}_n$~\eqref{eq:latent_positions} and the additive noise $\mvec{\epsilon}$~\eqref{eq:additive_noise}. 
In lieu of writing $\mathbb{E}_{\mmat{X}_n,\mmat{\mathcal{U}},\mvec{\epsilon}}\left[\cdot\right]$, we will simply 
use the notation $\mathbb{E}\left[\cdot\right]$.

Recall that the labels are given by $y_i = f(\mvec{x}_i)+\epsilon_i$. We make the following two general assumptions on the regression problem.

\begin{assumption}\label{ass:bounded_f}
    There exists $B>0$ such that 
    \begin{equation*}
        ||f||_{\infty}\coloneqq \sup_{\mvec{z}\in Q} |f(\mvec{z})|\leq B < \infty
    \end{equation*}
\end{assumption}

\begin{assumption}\label{ass:additive_noise}
    The additive noise $\mvec{\epsilon}$ is such that its entries are independent random variables and
    \begin{equation*}
        \mathbb{E}\left[\epsilon_i\right]=0 \hspace{5pt} 
        \text{and} 
        \hspace{5pt} \max_{i\in [n]}\mathbb{E}\left[\epsilon_i^2\right]\leq\sigma^2<\infty
    \end{equation*}
\end{assumption}

Assumption~\ref{ass:bounded_f} is somewhat restrictive but it holds in various settings. 
For example, if the domain $Q$ is compact and there is a continuity assumption on $f$, then Assumption~\ref{ass:bounded_f} is satisfied.
The classical setup in~\citep{Gyofri} includes this assumption.
Assumption~\ref{ass:additive_noise} is the most general assumption under model~\eqref{eq:labels}: 
while it is classical to assume stronger tail control of the distribution of the noise, here we just assume that it has finite variance. 

We will follow a bias-variance decomposition inspired approach.
For $\mvec{x}\in Q$ we introduce a \textit{variance proxy} and a \textit{bias proxy} at $\mvec{x}$:
\begin{equation}
\label{eq:variance_term}
    v(\mvec{x})=\mathbb{E}\left[{\left(\hat{f}_{\GNW}(\mvec{x})-S(f,\mvec{x})\right)}^2\right]
\end{equation}
\begin{equation}
\label{eq:bias_term}
b(\mvec{x})=S(f,\mvec{x})-f(\mvec{x})
\end{equation}
where $S(f,\mvec{x})$ is the operator given by~\eqref{eq:operator_S}. 
We remark that these variance and bias proxies
do \emph{not} correspond 
to the exact variance and bias, but are simpler to manipulate. 
In fact, for the true bias and variance, we have the following result.
\begin{proposition}\label{prop:proxies}   
    Let $\Bias\left[\hat{f}_{\GNW}(\mvec{x})\right]$ and $\Var\left[\hat{f}_{\GNW}(\mvec{x})\right]$ denote the standard bias and variance of $\hat{f}_{\GNW}(\mvec{x})$, i.e.
\begin{equation*}
\begin{split}
    &\Bias\left[\hat{f}_{\GNW}(\mvec{x})\right]=\mathbb{E}\left[\hat{f}_{\GNW}(\mvec{x})\right]-f(\mvec{x}) \hspace{3pt} \text{and}\\
    &\Var\left[\hat{f}_{\GNW}(\mvec{x})\right]=\mathbb{E}\left[{\left(\hat{f}_{\GNW}(\mvec{x})-\mathbb{E}\left[\hat{f}_{\GNW}(\mvec{x})\right]\right)}^2\right]
\end{split}
\end{equation*} 
If Assumptions~\ref{ass:bounded_f} and~\ref{ass:additive_noise} hold, then  
\begin{equation*}
    0\leq {\left[b(\mvec{x})-\Bias(\hat{f}_{\GNW}(\mvec{x}))\right]}^2\leq B^2\exp(-2d(\mvec{x}))
\end{equation*}
and 
\begin{equation*}
    0\leq v(\mvec{x})-\Var\left[\hat{f}_{\GNW}(\mvec{x})\right]\leq B^2\exp(-2d(\mvec{x})) 
\end{equation*}
\end{proposition}

The rest of this Section is dedicated to bounding the variance and bias proxies, in order to obtain a bound on the global risk.
The variance proxy~\eqref{eq:variance_term} governs the statistical fluctuation of $\hat{f}_{\GNW}(\mvec{x})$ around the quantity $S(f,\mvec{x})$. 
Its analysis is relying principally
on probability techniques such as concentration inequalities. We provide a bound of this term in Sec.~\ref{sharp_variance_bounds}. 
The main result of the \textit{sharp variance bound} given in Theorem~\ref{thm:variance_theorem}, 
which states that $v(\mvec{x})$ behaves like $1/d(\mvec{x})$.

The bias proxy~\eqref{eq:bias_term} on the other hand, measures the proximity of quantity $S(f,\mvec{x})$ towards the regression 
function $f(\mvec{x})$. Unlike the variance proxy~\eqref{eq:variance_term} which can be
universally controlled by concentration inequalities, the bias proxy~\eqref{eq:bias_term} must be controlled on a case-by-case basis.
To this end, we will focus on radial kernels~\eqref{eq:link_function}. 
In this scenario, for compactly supported link functions, 
$S(f,\mvec{x})$ approximates $f(\mvec{x})$ uniformly within precision 
$\mathcal{O}(h_g^{a})$, where $0<a\leq 1$ is the Hölder exponent of the regression function $f$ and $h_g$ is the length-scale of the random graph kernel. 
This dependence is described in Sec.~\ref{sec:bias}.

\subsection{A Variance Bound}\label{sharp_variance_bounds}

The goal of this subsection is to bound the variance proxy $v(\mvec{x})$~\eqref{eq:variance_term} in terms the expected degree $d(\mvec{x})$~\eqref{eq:local_degree}.
Theorem~\ref{thm:variance_theorem} shows that $v(\mvec{x})$ is of order $\mathcal{O}\left(\frac{1}{d(\mvec{x})}\right)$. Later, in Section~\ref{sec:bias} we will 
show how the local degree depends $d(\mvec{x})$~\eqref{eq:local_degree} depends on the parameters $h_g$ and $\alpha$ in the kernel function~\eqref{eq:link_function}.

\begin{theorem}{(\textbf{Sharp Variance Bound})}\label{thm:variance_theorem}
    \newline
    Suppose that Assumptions~\ref{ass:bounded_f} and~\ref{ass:additive_noise} hold. Then
    \begin{equation*}
         v(\mvec{x})\leq \frac{9B^2+2\sigma^2}{d(\mvec{x})}
    \end{equation*}
\end{theorem}

In large random graphs, where the number of nodes grows to infinity, we may consider an asymptotic regime 
where $d(\mvec{x})$ depends on the number of vertices $n$. Theorem~\ref{thm:variance_theorem} 
shows that $\hat{f}_{\GNW}$ concentrates towards $S(f,\mvec{x})$ \emph{as soon as the local degree~\eqref{eq:local_degree} grows to infinity}, even arbitrarily slowly.
In comparison, most methods in the literature require a certain growth rate in order to provide a theoretical guarantee: for instance, a classical threshold is logarithmic degrees, $d(\mvec{x})\gtrsim\log(n)$~\citep{Oliviera,Lei_2015}.

\medskip

\begin{proof}{\textbf{of Theorem~\ref{thm:variance_theorem}}}
For convenience of notation, let 
\begin{equation}\label{eq:deg_shorthand}
    Z\coloneqq\mathbb{I}\left[\hat{d}(\mvec{x})>0\right]
\end{equation}
Note that by Definition~\eqref{gnw_def} and Equation~\eqref{eq:deg_shorthand}
\begin{equation}\label{gnw_no_edges}
 (1-Z)\hat{f}_{\GNW}(\mvec{x})=0
\end{equation}
or equivalently
\begin{equation}
\label{gnw_edges}
   Z\hat{f}_{\GNW}(\mvec{x})=\hat{f}_{\GNW}(\mvec{x}) 
\end{equation}
Keeping in mind that $Z$ is $\{0,1\}$-valued variable signifying the occurrence of an edge 
incident to node $(n+1)$, we have 
\begin{equation}\label{eq:prob_of_no_edges}
    \mathbb{E}\big[1-Z\big]=\mathbb{P}\big(a(\mvec{x},\mvec{x}_1)=a(\mvec{x},\mvec{x}_2)=\cdots=a(\mvec{x},\mvec{x}_n)=0\big)={(1-c(\mvec{\mvec{x}}))}^n
\end{equation}
Additionally, we have $Z^2=Z$ 
and ${(1-Z)}^2=1-Z$, so using Equations~\eqref{gnw_no_edges} and~\eqref{gnw_edges}, we get
\begin{align*}\
    v(\mvec{x})&=\mathbb{E}\left[{\left(\hat{f}_{\GNW}(\mvec{x})-S(f,\mvec{x})\right)}^2Z\right]+\mathbb{E}\left[{\left(\hat{f}_{\GNW}(\mvec{x})-S(f,\mvec{x})\right)}^2 (1-Z)\right]\\
    &=\mathbb{E}\left[{\left(\hat{f}_{\GNW}(\mvec{x})Z-S(f,\mvec{x})Z\right)}^2\right]+\mathbb{E}\left[{\left(\hat{f}_{\GNW}(\mvec{x})(1-Z)-S(f,\mvec{x})(1-Z)\right)}^2\right]\\
    &=\mathbb{E}{\left[{\left(\hat{f}_{\GNW}(\mvec{x})-S(f,\mvec{x})Z\right)}^2\right]}+S^2(f,\mvec{x})\mathbb{E}\left[1-Z\right]\\
    &=\mathbb{E}{\left[{\left(\hat{f}_{\GNW}(\mvec{x})-S(f,\mvec{x})Z\right)}^2\right]}+S^2(f,\mvec{x}){\left(1-c(\mvec{x})\right)}^n \numberthis\label{eq:variance_decomp} 
\end{align*}
where we used Equation~\eqref{eq:prob_of_no_edges} in line~\eqref{eq:variance_decomp}.
In particular, we get 
\begin{equation}\label{eq:final_bound_plug}
    v(\mvec{x})\leq \mathbb{E}{\left[{\left(\hat{f}_{\GNW}(\mvec{x})-S(f,\mvec{x})Z\right)}^2\right]}+e^{-d(\mvec{x})}S^2(f,\mvec{x})
\end{equation}
From Equation~\eqref{eq:final_bound_plug} it follows that we only need to focus on control of 
\begin{equation}\label{eq:variance_technical_term}
    \mathbb{E}{\left[{\left(\hat{f}_{\GNW}(\mvec{x})-S(f,\mvec{x})Z\right)}^2\right]}
\end{equation}
We will show that the term~\eqref{eq:variance_technical_term} is of order $\mathcal{O}(\frac{1}{d(\mvec{x})})$, and hence, in 
Equation~\eqref{eq:final_bound_plug} we may substitute $e^{-d(\mvec{x})}$ with $\frac{1}{d(\mvec{x})}$.
The key insight is that the expression within the expectation of Equation~\eqref{eq:variance_technical_term} 
has a representation as a sum of identically distributed, uncorrelated variables whose variance is easy to compute.
We now derive this representation, using a method that we title the \textbf{decoupling trick}.

\paragraph{The decoupling trick}
For $I\subseteq [n]$, let
\begin{equation*}
\label{eqn_R_empty}
R_I(\mvec{x})= \begin{cases}
    \frac{1}{|I|+\sum_{j\notin I}a(\mvec{x},\mvec{x}_j)}, \hspace{3pt} I\neq\emptyset \\
        \frac{1}{\hat{d}(\mvec{x})}, \hspace{3pt} I=\emptyset \hspace{3pt} \text{and} \hspace{3pt} \hat{d}(\mvec{x})>0\\
        0, \hspace{3pt} \text{otherwise}\\
\end{cases}
\end{equation*}
For $I\neq\emptyset$, consider the graph $\mathcal{G}_I$ obtained by adding the edges $\{(n+1,i)|i\in I\}$ to the original LPM graph $\mathcal{G}$
(of course, not all edges of the form $(n+1,i),i\in I$ need to exist in the original graph). 
Then $1/R_I(\mvec{x})=|I|+\sum_{j\notin I}a(\mvec{x},\mvec{x}_j)$ can be thought of as  
counting the number of neighbors of node $n+1$ in the modified graph $\mathcal{G}_I$.
For $I=\emptyset$, observe that when $\hat{d}(\mvec{x})>0$, $R_{\emptyset}(\mvec{x})=1/\hat{d}(\mvec{x})$ whereas 
when $\hat{d}(\mvec{x})=0$, $R_{\emptyset}(\mvec{x})=0$, and hence one can easily see that 

\begin{equation}\label{GNW_as_sum}
    \hat{f}_{\GNW}(\mvec{x})=\sum_{i=1}^n y_i a(\mvec{x},\mvec{x}_i)R_{\emptyset}(\mvec{x})
\end{equation}
At this point we placed the inconvenience of having a bracket in the definition~\eqref{gnw_def} into the variable $R_{\emptyset}(\mvec{x})$.
For convenience of notation we define
\begin{equation}\label{shorthand}
    R_i(\mvec{x})\coloneqq R_{\{i\}}(\mvec{x}) \hspace{5pt}
\end{equation}
Taking into account the fact that $a(\mvec{x},\mvec{x}_i)$ is a Bernoulli variable, i.e.\ it takes values in $\{0,1\}$, 
it follows that for all $i\in [n]$ 
\begin{equation}
\label{eqn_R_single}
R_{\emptyset}(\mvec{x})a(\mvec{x},\mvec{x}_i)=R_{i}(\mvec{x})a(\mvec{x},\mvec{x}_i)
\end{equation}
Indeed, if $a(\mvec{x},\mvec{x}_i)=0$ then both sides of Equation (\ref{eqn_R_single}) are $0$. Otherwise $a(\mvec{x},\mvec{x}_i)=1$ and both sides in Equation (\ref{eqn_R_single}) equal $R_{i}(\mvec{x})$. 
Moreover, $R_i(\mvec{x})$ is independent from $a(\mvec{x},\mvec{x}_i)$.
More generally we have the following observation.
\begin{lemma}{\textbf{(Decoupling trick)}}\label{lemma:decoupling}
    For all pairs of \textbf{disjoint} subsets $I,J\subseteq$[n] we have
    \begin{equation*}
    R_J(\mvec{x})\prod_{i\in I}a(\mvec{x},\mvec{x}_i)=R_{I\cup J}(\mvec{x})\prod_{i\in I}a(\mvec{x},\mvec{x}_i)
    \end{equation*}
    and $R_{I\cup J}(\mvec{x})$ is independent from $\{a(\mvec{x},\mvec{x}_i)|i\in I\}$.
\end{lemma}
\begin{proof}
If $\prod_{i\in I}a(\mvec{x},\mvec{x}_i)=0$ then there is nothing to prove. If $\prod_{i\in I}a(\mvec{x},\mvec{x}_i)\neq 0$, then by the fact that $a(\mvec{x},\mvec{x}_i)$ are Bernoulli variables we get $a(\mvec{x},\mvec{x}_i)=1$ for all $i\in I$.
As $I\subseteq[n] \setminus J$, we have
\begin{equation*}
\label{decoupling_trick}
    R_{J}(\mvec{x})=\frac{1}{|J|+\sum_{i\notin J}a(\mvec{x},\mvec{x}_i)}=\frac{1}{|I|+|J|+\sum_{i\notin I\cup J}a(\mvec{x},\mvec{x}_i)}=R_{I\cup J}(\mvec{x})
\end{equation*}
The second part of the lemma follows from modeling assumptions.
\end{proof}

\pagebreak[3]

\noindent
Plugging in Equation~\eqref{eqn_R_single} into Equation~\eqref{GNW_as_sum}, we get
\begin{equation}
    \label{eqn_for_gnw_Ri}
        \hat{f}_{\GNW}(\mvec{x})=\sum_{i=1}^n y_i a(\mvec{x},\mvec{x}_i)R_i(\mvec{x})
\end{equation}
Moreover, summing Equation~\eqref{eqn_R_single} over $i\in [n]$ gives
        
\begin{equation*}
    \begin{split}
    \sum_{i=1}^n a(\mvec{x},\mvec{x}_i)R_i(\mvec{x})&=\sum_{i=1}^n a(\mvec{x},\mvec{x}_i)R_{\emptyset}(x)\\
    &=\hat{d}(\mvec{x})R_{\emptyset}(\mvec{x})\\
    &=\mathbb{I}\left[\hat{d}_n(\mvec{x})>0\right]=Z
    \end{split}
\end{equation*}
we get
\begin{equation}\label{eqn_for_Ri}
       Z=\sum_{i=1}^n a(\mvec{x},\mvec{x}_i)R_i(\mvec{x})
\end{equation}
Using Equations~\eqref{eqn_for_gnw_Ri} and~\eqref{eqn_for_Ri} we have
\begin{equation}
\label{decomp}
\begin{split}
    \hat{f}_{\GNW}(\mvec{x})-S(f,\mvec{x})Z&=
    \sum_{i=1}^n y_i a(\mvec{x},\mvec{x}_i)R_i(\mvec{x})-S(f,\mvec{x})\sum_{i=1}^n a(\mvec{x},\mvec{x}_i)R_i(\mvec{x})\\
    &=\sum_{i=1}^n(y_i-S(f,\mvec{x}))a(\mvec{x},\mvec{x}_i)R_i(\mvec{x})
\end{split}
\end{equation}
In Appendix~\ref{proof_for_ref_1} we show that the summands in the right hand side of Equation~\eqref{decomp} are uncorrelated and 
consequently we obtain the following expression for the quantity~\eqref{eq:variance_technical_term}.

\begin{lemma}\label{lemma:tech_1}
    We have
    \begin{equation*}
        \mathbb{E}{\left[{\left(\hat{f}_{\GNW}(\mvec{x})-S(f,\mvec{x})Z\right)}^2\right]}=\sum_{i=1}^n \mathbb{E}\left[{\left(y_i-S(f,\mvec{x})\right)}^2a(\mvec{x},\mvec{x}_i)R_i^2(\mvec{x})\right]
    \end{equation*}
\end{lemma}
The proof of Lemma~\ref{lemma:tech_1} may be found in the Appendix; it uses the decoupling trick~\eqref{lemma:decoupling} along with the 
fact\footnote{This fact is the reason why we work directly with the random design; verbatim analysis 
for the fixed design does not satisfy this.} that $\mathbb{E}\left[\left(y_i-S(f,\mvec{x})\right)a(\mvec{x},\mvec{x}_i)\right]=0$. Since
\begin{equation}\label{S_bound}
    |S(f,\mvec{x})|\leq ||f||_{\infty}\,
\end{equation}
one can deduce from Assumption~\eqref{ass:bounded_f} that $|S(f,\mvec{x})|\leq B$. 
For $i\in[n]$, we have 
\begin{equation}\label{eq:eqn_for_clarity}
\begin{split}
    \mathbb{E}_{\mvec{\epsilon}}\left[{(y_i-S(f,\mvec{x}))}^2\right]
    &=\left[{\left(f(\mvec{x}_i)-S(f,\mvec{x})\right)}^2\right]
    +\mathbb{E}_{\mvec{\epsilon}}\left[\epsilon_i^2\right]\\
    &\leq 4B^2+\sigma^2
\end{split}  
\end{equation}
Plugging in the bound~\eqref{eq:eqn_for_clarity} into Lemma~\ref{lemma:tech_1}, we get  

\begin{equation}\label{ubv_1}
\begin{split}
    \mathbb{E}\left[{\left(\hat{f}_{\GNW}(\mvec{x})-S(f,\mvec{x})Z\right)}^2\right] 
    &\leq (4B^2+\sigma^2)\mathbb{E}\left[\sum_{i=1}^{n}a(\mvec{x},\mvec{x}_i)R^2_i(\mvec{x})\right]\\
    &= (4B^2+\sigma^2)\mathbb{E}\left[\sum_{i=1}^n a(\mvec{x},\mvec{x}_i)R^2_{\emptyset}(\mvec{x})\right]\\
    &= (4B^2+\sigma^2)\mathbb{E}\left[\frac{1}{\hat{d}(\mvec{x})}\mathbb{I}\left[\hat{d}(\mvec{x})>0\right]\right]
\end{split}
\end{equation}
Applying Lemma 4.1 in~\citep{Gyofri} gives that 
\begin{equation*}
    \mathbb{E}\left[\frac{1}{\hat{d}(\mvec{x})}\mathbb{I}\left[\hat{d}(\mvec{x})>0\right]\right]\leq \frac{2}{d(\mvec{x})}
\end{equation*}
Finally, 
\begin{equation}\label{new_var_final_b}
    \mathbb{E}{\left[{\left(\hat{f}_{\GNW}(\mvec{x})-S(f,\mvec{x})Z\right)}^2\right]}\leq \frac{8B^2+2\sigma^2}{d(\mvec{x})}
\end{equation}
Plugging the bounds~\eqref{S_bound} and~\eqref{new_var_final_b} into Equation~\eqref{eq:final_bound_plug}, we get the desired result.
\end{proof}
Theorem~\ref{thm:variance_theorem} is essentially tight, at least in the presence of additive noise,
 i.e.\ we have the following lemma.
\begin{lemma}\label{lemma:variance_lower_bound}
     Suppose that $\min_{i\in [n]}\mathbb{E}[\epsilon_i^2]\geq \sigma_0^2>0$. Then 

    \begin{equation*}
        v(\mvec{x})\geq \frac{\sigma_0^2{\left(1-e^{-d(\mvec{x})}\right)}^2}{d(\mvec{x})} 
    \end{equation*}
    
\end{lemma}
The proof of Lemma~\ref{lemma:variance_lower_bound} can be found in the appendix.

We remark that the Theorem~\ref{thm:variance_theorem} holds for Latent Position Models 
with general nonparametric kernel functions $k\colon Q\times Q \to [0,1]$, 
as long as the condition $d(\mvec{x})>0$ holds. Indeed, the proof of Theorem~\ref{thm:variance_theorem} 
is independent of the form of $k$. However, $S(f,\mvec{x})$~\eqref{eq:operator_S} depends on $f$ 
and on the kernel function $k$~\eqref{eq:general_link_function}. 
When $k(\mvec{x}_i,\mvec{x}_j)$ depends on the distance $||\mvec{x}_i-\mvec{x}_j||$ as in~\eqref{eq:link_function},
$S(f,\mvec{x})$ is a good approximant of $f(\mvec{x})$, as we show in Sec.~\ref{sec:bias}.

\subsection{Bias and Risk of GNW}\label{sec:bias} 
In Sec.~\ref{sharp_variance_bounds} we considered a LPM graph with general kernel function $k$. 
As mentioned before in \eqref{eq:link_function}, we will suppose that the kernel is \textit{radial}, i.e. 
\begin{equation}\label{radial_kernel}
    k(\mvec{x},\mvec{z})=\alpha K\left(\frac{||\mvec{x}-\mvec{z}||}{h_g}\right)
\end{equation}
with $0<\alpha \leq 1$ and $h_g>0$. These two parameters are considered to be \textbf{unknown and fixed}.
There are two important questions that need to be addressed. First, under which conditions on $\alpha$ and $h_g$ is 
$S(f,\mvec{x})$ a good approximation of $f(\mvec{x})$? In other words, how does the bias proxy~\eqref{eq:bias_term} depend on $\alpha$ and $h_g$?
Second, our bound for the variance proxy~\eqref{eq:variance_term} is in terms of the local degree $d(\mvec{x})$. Therefore it is important to 
understand how the local degree $d(\mvec{x})$ depends on the parameters $\alpha$ and $h_g$.
We address these questions in this section. Proofs for this Section can be found in the Appendix~\ref{bias_risk_proofs}. 

In order to control the bias proxy~\eqref{eq:bias_term} we will need
to assume regularity conditions on the regression function $f$, the kernel function $K$ and on the density $p$. 
\begin{assumption}{\textbf{(Box assumption)}}\label{ass:K_1}
\newline
There exists $M_1,M_2>0$ s.t.\ for all $t\in[0,\infty)$
\begin{equation*}
    \frac{1}{2}\mathbb{I}\left[t\leq M_1\right]\leq K(t)\leq \mathbb{I}\left[t\leq M_2\right]
\end{equation*}
\end{assumption}
\begin{assumption}{\textbf{(Regularity of the regression function)}}\label{F_1}
    \newline    
    There exist $0<a\leq 1$ and $L>0$ such that for all $\mvec{x},\mvec{z}\in Q$
    \begin{equation*}
        |f(\mvec{x})-f(\mvec{z})|\leq L||\mvec{x}-\mvec{z}||^a
    \end{equation*}
\end{assumption}

\begin{assumption}{\textbf{(Regularity of the domain)}}\label{ass:measure_retention} 
    \newline
    There exist $r_0,c_0>0$ such that for all $\mvec{x} \in Q=\supp{(p)}$, and all $r\leq r_0$, 

    \begin{equation*}
        m\left(Q\cap B_r\left(\mvec{x}\right)\right)\geq c_0m\left(B_r\left(\mvec{x}\right)\right)
    \end{equation*}
    Here $m$ is the Lebesgue measure on $\mathbb{R}^d$.
\end{assumption}
Finally, Assumptions~\ref{eq:density_ass} and~\ref{ass:hcd_condition} cover different type of distributions for the latent positions. 

\begin{assumption}{\textbf{(Density Assumption 1)}}\label{eq:density_ass}
    \newline
    There exists $p_0>0$ such that for all $\mvec{x}\in Q$
    \begin{equation*}
        p(\mvec{x})\geq p_0
    \end{equation*}
\end{assumption}

\begin{assumption}{\textbf{(Density Assumption 2)}}\label{eq:density_ass_2}
    \newline
    There exist $0<b\leq 1$ and $S>0$
    such that $p\in\Sigma(b,S)$ and\label{ass:hcd_condition}
        \begin{equation*}
            \int p^{1/2}(\mvec{x})dx<\infty
        \end{equation*}
\end{assumption}
Assumptions~\ref{ass:K_1} and~\ref{F_1} are rather classical in the context of the NW estimator. 
The NW estimator performs poorly in low density regions and near the boundary of the support of the data distribution. The 
intuitive explanation for this behavior is that because there are on average
fewer observations in such a region, the variance of the estimator is greater.
Under Assumptions~\ref{ass:measure_retention},~\ref{eq:density_ass}
and~\ref{ass:hcd_condition}, the problematic regions are not too large. 
Assumption~\ref{ass:measure_retention} is the most technical one, but it it satisfied in many instances considered in the classical regression setting. Clearly, $\mathbb{R}^d$ satisfies Assumption~\ref{ass:measure_retention} with $r_0=\infty$, $c_0=1$ 
and it is not difficult to show that the Cube $Q_d={[-1,1]}^d$ satisfies the regularity Assumption~\ref{ass:measure_retention} with $r_0=1$, $c_0=\frac{1}{2^d}$ and so does every closed 
and convex subset\footnote{when compact, such sets are called convex bodies} of $\mathbb{R}^d$ (for some $r_0,c_0>0$). Another broad class of sets which satisfy this property and are used in the regression context in $\mathbb{R}^d$ 
are those that satisfy \emph{interior cone condition}~\citep{Wendland}. A set $Q$ satisfies an interior cone condition with cone $C$ if for all points $\mvec{x}\in Q$, one can rotate and translate $C$
to a cone $C_{\mvec{x}}$ with a vertex in $\mvec{x}$ such that $C_{\mvec{x}}\subseteq Q$. 
A typical example of Assumption~\ref{eq:density_ass} is the uniform distribution (over a convex body), whereas Assumption~\ref{ass:hcd_condition} covers 
non-compactly supported, but smooth distributions such as the Gaussian.

Under Assumptions~\ref{ass:K_1} and~\ref{F_1}, the bias proxy~\eqref{eq:bias_term} is uniformly bounded over $Q$ by Lemma~\ref{lemma:bias_control}.

\begin{lemma}{\textbf{(Bias control lemma)}}\label{lemma:bias_control}
    Suppose that Assumptions~\ref{ass:K_1} and~\ref{F_1} hold. Then
        \begin{equation*}
        \sup_{\mvec{x}\in Q}|S(f,\mvec{x})-f(\mvec{x})|\leq 2LM_2^{a}h_g^{a}
    \end{equation*}
\end{lemma}

The problematic vertices for GNW are those whose latent positions fall in 
a low density region or are near the boundary of the support~$Q$. Our next lemma
lower-bounds the expected degree of a node, to control the risk:
\begin{lemma}{\textbf{(Local degree bound)}}\label{lemma:local_degree_bound} 
    \newline
    Suppose that Assumption~\ref{ass:K_1} and~\ref{ass:measure_retention} hold. 
    If $M_1h_g<r_0$ and $\mvec{x}\in Q$ is such that 
    \begin{equation}
    \label{lbd}
    p_0(\mvec{x})\coloneqq\inf\limits_{\substack{\mvec{z}\in Q\\||\mvec{x}-\mvec{z}||\leq M_1h_g}} p(\mvec{z})>0
    \end{equation}
    Then 
    \begin{equation*}
        \frac{1}{d(\mvec{x})}\leq \frac{2}{c_0v_d M_1^d n\alpha h_g^d p_0(\mvec{x})}
    \end{equation*}
where, we recall, $v_d$ is the volume of the $d-$dimensional unit ball.
\end{lemma}
This Lemma in combination with Theorem~\ref{thm:variance_theorem} and Lemma~\ref{lemma:bias_control} 
gives a bound on the point-wise risk~\eqref{eq:node_reg_risk_pointwise}.
Having established bounds on the bias~\eqref{eq:bias_term} and variance~\eqref{eq:variance_term} proxies, we are ready to 
provide a bound on the point-wise risk~\eqref{eq:node_reg_risk_pointwise}.
\begin{theorem}{(\textbf{Pointwise risk bound})}\label{thm:pwriskthm}
    \newline
    Suppose that Assumptions~\ref{ass:K_1},~\ref{F_1},~\ref{ass:measure_retention} hold.
    Furthermore, suppose that $\mvec{x}\in Q$ is s.t.\ $p_0(\mvec{x})>0$ where $p_0(\mvec{x})$ is given by~\eqref{lbd}. 
    If $M_1h_g\leq r_0$ then
    \begin{equation*}
            \mathcal{R}_g\left(\hat{f}_{\GNW}(\mvec{x}),f(\mvec{x})\right)\leq 4L^2M_2^{2a}h_g^{2a}+\frac{36B^2+8\sigma^2}{c_0 v_d M_1^d n\alpha h_g^d p_0(\mvec{x})}
    \end{equation*}
\end{theorem}
In the next section we follow up by giving bounds on the global risk~\eqref{eq:node_reg_risk_global} 
in terms of the parameters $\alpha $ and $h_g$.
\subsection{Global risk}\label{irisk}
Finally to bound the global risk~\eqref{eq:node_reg_risk_global} of GNW, we would like to integrate the inequality given in 
Theorem~\ref{thm:pwriskthm}. Unfortunately the right hand side of this inequality depends on $p_0(\mvec{x})$, a quantity that depends non trivially 
on the behavior of $p$ around the point $\mvec{x}\in Q$, so a direct integration does not work. However, Assumption~\ref{eq:density_ass} 
allows us to conclude that $p_0(\mvec{x}) \geq p_0$ for all $\mvec{x}\in Q$ and hence yields the following result.

\begin{theorem}{\textbf{(Risk bound 1)}}\label{thm:final_result}
\newline 
Suppose that Assumptions~\ref{ass:K_1},~\ref{F_1},~\ref{ass:measure_retention} and~\ref{eq:density_ass} hold.
If $M_1h_g<r_0$, we have
\begin{equation*}
    \mathcal{R}_g\left(\hat{f}_{\GNW},f\right)\leq C_1h_g^{2a}+\frac{C_2}{n\alpha h_g^d}
\end{equation*}
where $C_1=4L^2M_2^{2a}$, $C_2=\frac{36B^2+8\sigma^2}{p_0 c_0 v_d M_1^d}$.
\end{theorem}
Theorem~\ref{thm:final_result} matches the classical rate~\eqref{eq:NW_rate} with $\tau\coloneqq h_g$. In this sense,
GNW with length-scale $h_g$ behaves like a classical NW estimator with \textbf{fixed bandwidth} $\tau\coloneqq h_g$. In particular,
Assumptions~\ref{ass:measure_retention} and~\ref{eq:density_ass} apply for latent positions with compactly supported distribution $p$,
that is also lower bounded by a positive constant, in some sense a relaxation of the uniform distribution.
Theorem~\ref{thm:final_result} does not cover distributions supported on all of $\mathbb{R}^d$, or any set of infinite 
Lebesgue measure more generally. For example, the Gaussian distribution over~$\mathbb{R}^d$ is not covered by Assumption~\ref{eq:density_ass}. 
Such density functions must achieve arbitrary small values and hence it is not possible to control $p_0(\mvec{x})$ globally in the same 
way as it was done with Assumption~\ref{eq:density_ass}. However, under Assumption~\ref{eq:density_ass_2} we get the following result. 

\begin{theorem}{\textbf{(Risk bound 2)}}\label{thm:final_result_holder}
\newline     
    Suppose that Assumptions~\ref{ass:K_1},~\ref{F_1},~\ref{ass:measure_retention} and~\ref{ass:hcd_condition} hold. If $h_g<\min{(r_0/M_1,1)}$ then 
    \begin{equation*}
        \mathcal{R}_g\left(\hat{f}_{\GNW},f\right)\leq C_1h_g^{\min{\left(2a,b/2\right)}}+\frac{C_2}{n\alpha h_g^{d+b}}
    \end{equation*}
where $C_1=4L^2M_2^{2a}+(8B^2+2\sigma^2)S^{1/2}M_1^{b/2}\int p^{1/2}(\mvec{x})dx$ and $C_2=\frac{36B^2+8\sigma^2}{c_0 v_d S M_1^{d+b}}$.
\end{theorem}
Assumptions~\ref{ass:measure_retention} and~\ref{ass:hcd_condition} extend the class of distributions of the latent 
points to Hölder continuous density with non-compact support, with the caveat that the support still needs
to be geometrically regular. The cost of these assumptions is an increase both in the bias and the variance proxies. For example, when $b=1$
and $h_g$ is sufficiently small, $h_g^{2a}\lesssim h_g^{\min{\left(2a,1/2\right)}}$ and equality holds only when $a\leq 1/4$, i.e.\ the regression function $f$ is fairly hard to learn.
On the other hand, $\frac{1}{n\alpha h_g^d}\lesssim \frac{1}{n\alpha h_g^{d+1}}$, and hence the variance term is always worse under Assumption~\ref{ass:hcd_condition}.

The main idea behind the proof of Theorem~\ref{thm:final_result_holder} is to split the risk over a high density region 
i.e.\ where the density is $p(\mvec{x})\gtrsim h^{b}_g$, and its complement. Due to the integrability condition in Assumption~\ref{eq:density_ass_2}, and the fact 
that the point-wise risk is bounded by a constant, the low density region can be handled. For the high-density region, the risk is controlled by Lemma~\ref{thm:pwriskthm}.  

\subsection{Discussion}
The GNW estimator is computationally extremely cheap (with runtime
$\mathcal{O}(n)$) and it has the same convergence rate\footnote{up to a
  multiplicative constant} as the (fixed-bandwidth) NW estimator. In order 
to find the range of values $h_g$ for which the GNW risk converges for large 
LPMs, we conduct a simplified asymptotic analysis of GNW\@. We suppose that 
the scaling factor $\alpha=1$ in~\eqref{eq:link_function}.
In order to conclude that the risk $\mathcal{R}_g\left(\hat{f}_{\GNW},f\right)$
converges to 0, we need both terms in the upper bound of
Theorem~\ref{thm:final_result}, namely $h_g^{2a}$ and $1/nh_g^d$ to go to $0$.
Note that this is equivalent to having the expected local
degree~\eqref{eq:local_degree} $d(\mvec{x})\to\infty$ and the local edge
density~\eqref{eq:local_degree} $c(\mvec{x})\to 0$ at the same time. Moreover,
elementary calculus shows that the ideal bias-variance tradeoff is achieved for
$h_g=\tau_{\star}\coloneqq c_{a,d}n^{-\frac{1}{d+2a}}$ and the associated rate
for the risk is $C_{a,d}n^{-\frac{2a}{d+2a}}$, where $c_{a,d}$ and $C_{a,d}$
also depend on the various parameters that appear in the constants $C_1$ and
$C_2$ in Theorem~\ref{thm:final_result}. Similar analysis may be conducted for
Theorem~\ref{thm:final_result_holder}. In summary, $\hat{f}_{\GNW}$ behaves
reasonably well as soon as $h_g\to 0$ and $nh_g^d\to\infty$, with the risk
depending on where this parameter $h_g$ happens to fall.

The problem with using the GNW estimator arises when the length-scale $h_g$ is
either too small or too large. When the length-scale $h_g$ is too small
(relative to $\tau_{\star}$), GNW averages labels over a neighborhood that is
too small, meaning it will have low bias but high variance. We call this case the \emph{under-averaging
  regime} ($h_g \ll \tau_{\star}$). On the other hand, a large length-scale
$h_g$ (relative to $\tau_{\star}$) will result in averaging on a window of size
$h_g$, larger than the optimal window of size $\tau_{\star}$. This leads to low
variance but high bias, and we call this case
the \emph{over-averaging regime} ($h_g\gg \tau_{\star}$).

We reemphasize the fact that length-scale $h_g$ and the optimal bandwidth
$\tau_{\star}$ are \textbf{not} user chosen parameters. The length-scale $h_g$
is inherent to the generative process of the graph: it influences the size of neighborhoods in the latent space and the sparsity
of the graph. The optimal bandwidth $\tau_{\star}$ depends primarily on the
sample size $n$, the smoothness of the regression function $f$, namely the
Hölder constant and exponent $L$, $a$, as well as on the variance of the
additive noise $\sigma^2$. The optimal bandwidth $\tau_{\star}$ determines the
size of the window in the latent space which achieves optimal performance for
the label $\mvec{y}$ given by~\eqref{eq:labels}. In the remainder of this paper
we will focus on the following question: For what pairs of values
$(h_g,\tau_{\star})$ can we construct a node regression estimator that achieves
the optimal risk rate~\eqref{eq:NW_rate} with $\tau\coloneqq\tau_{\star}$?

Theorem~\ref{thm:final_result} states that GNW achieves this for $h_g=\tau_{\star}$, and less formally, 
that the risk is nearly optimal when $h_g$ is in the vicinity of $\tau_{\star}$. In the following section we consider 
the \emph{Estimated Nadaraya Watson} estimator, which, as we will see, achieves \emph{standard minimax risk rates} in certain 
under-averaging and over-averaging regimes.

\section{Nadaraya-Watson on estimated positions (ENW)}\label{ENW_seciton} 

In this section we are going to consider an estimator that 
works on a broader range of length-scales, that still relies on a local averaging approach. 
In particular, the goal is to construct node 
regression estimators that will outperform the GNW estimator in the under and over-averaging regimes, 
when $h_g\ll \tau_{\star}$ and $h_g\gg \tau_{\star}$ respectively, with the ultimate goal of
achieving optimal rate~\eqref{eq:NW_rate} for $\tau=\tau_{\star}$.
We consider a LPM with kernel function~\eqref{radial_kernel} with $\alpha=1$.

As mentioned in the introduction, since the main drawback of GNW is that the latent positions $\mvec{x}_i$ are unknown and the bandwidth $\tau$ cannot be adjusted, we suggest to combine two classical approaches from the literature: \emph{first} estimate the latent positions -- or, more precisely, estimate the \emph{distances} from the regression node $(n+1)$ to all the others -- \emph{then} use the standard 
Nadaraya-Watson estimator with these estimated distances, allowing the user to freely tune the bandwidth as usual.

As there are many approaches for the first step in the literature, we suppose that the user chooses some \textit{latent distance estimation algorithm} $\mathcal{A}$ that takes as an input the 
observed graph (potentially with some other hyper-parameters) and returns an \textit{estimate for the distance between the latent positions} of node $(n+1)$ and all other nodes $i$ in the graph.
The algorithm $\mathcal{A}$ needs to be deterministic, in the sense that the only random components on which it acts are the adjacency matrix $\mmat{A}$ and the observed label $\mvec{y}$,
i.e.\ we do not cover random algorithms which require additional randomness in their execution such as additional random walks used in DeepWalk~\citep{DeepWalk} or Node2Vec~\citep{node2vec}.
By plugging these estimated distances in a NW estimator, we end up with a prediction for the 
regression node. We call the resulting estimator the \textbf{$\mathcal{A}$-Estimated Nadaraya-Watson} ($\mathcal{A}$-ENW).

The analysis of the performance of $\mathcal{A}$-ENW may be broken down into two problems: analyzing the precision of the distance estimation algorithm $\mathcal{A}$, and analyzing the \emph{stability of NW} to using
estimated distances instead of exact ones. The first error depends on the choice of $\mathcal{A}$ and has been studied extensively in the literature. We will give several examples in Section~\ref{sec:position_estimation_algos}. 

For the stability of NW, in Section~\ref{sec:risk_bound_enw} we  
provide a risk bound when the algorithm $\mathcal{A}$ estimates  
the distances $\mvec{\delta}$~\eqref{eq:latent_distances} with an additive error $\Delta$. 
We show that in this case $\mathcal{A}$-ENW achieves (up to a multiplicative constant) 
the classical NW rate~\eqref{eq:NW_rate} \emph{as long as} $\tau\gtrsim\Delta$. 
This result is formally stated in Theorem~\ref{thm:perturbed_nw_thm}. In particular, given a problem for which the optimal bandwidth is $\tau_{\star}$,
$\mathcal{A}$-ENW can achieve the optimal NW-rate~\eqref{eq:NW_rate} with $\tau\coloneqq\tau_{\star}$ provided that $\tau_{\star}\gtrsim \Delta$.

In Section~\ref{sec:position_estimation_algos} we give several examples of 
algorithms $\mathcal{A}$ in the literature and their respective $\Delta$.
We find instances $(h_g,\tau_{\star})$, both in the under-averaging and the over-averaging regime, for which 
there exist algorithms $\mathcal{A}$ such that $\mathcal{A}$-ENW achieves the optimal NW-rate~\eqref{eq:NW_rate} with $\tau\coloneqq\tau_{\star}$.  

\subsection{Risk bound on Estimated Nadaraya Watson}\label{sec:risk_bound_enw}
In this section, in addition to the observed labels $\mvec{y}$ on the first $n$ nodes~\eqref{eq:labels} $\mvec{y}={\left[y_1,\dots,y_n\right]}^t$ 
and the adjacency matrix $\mmat{A}$,
we also assume that there exists an algorithm $\mathcal{A}$ that takes the observed graph with adjacency matrix $\mmat{A}$ as 
input
and outputs a vector $\mvec{\tilde{\delta}}=\left[\tilde{\delta}_1,\dots,\tilde{\delta}_{n}\right]$, an \textit{estimation of the distances} 
$\mvec{\delta} = \left[ \delta_1,\dots, \delta_n \right]$ where $\delta_i$ is the distance between the $(n+1)$th and the $i$th latent variables~\eqref{eq:latent_distances}. 
We remark that such an algorithm should be 
equivariant, i.e.\ if we relabel the nodes $[n]$ with some permutation $\pi\colon [n]\to [n]$, the algorithm $\mathcal{A}$ 
will permute its outputs by that same permutation. We suppose that the latent positions $\mmat{X}_{n+1}$ are fixed;
although our analysis can be easily extended to the random design case as well.
We will measure the quality of the estimator $\mathcal{A}$ by

\begin{equation}\label{eq:DELTA_distances}
\mathrm{\Delta}\left(\mathcal{A},\mmat{X}_{n+1}\right)\coloneqq ||\mvec{\tilde{\delta}}-\mvec{\delta}||_{\infty} =\max_{i\in [n]}|\tilde{\delta}_i-\delta_i|
\end{equation}
where $\delta_i$ is given by~\eqref{eq:latent_distances}.
Even though we state our results in terms of latent distance estimation, 
we can easily adapt them to \emph{position estimation} algorithms, as described 
in the following remark.
\begin{remark}{\textbf{(Position Estimation Algorithms)}}
    If $\mathcal{B}\coloneqq{\{0,1\}}^{(n+1)\times(n+1)}\to \mathbb{R}^{d\times (n+1)}$ is a position estimation algorithm with 
    $\mathcal{B}(\mmat{A})=[\mvec{\tilde{x}}_1,\dots,\mvec{\tilde{x}}_n,\mvec{\tilde{x}}_{n+1}]$, where $\mvec{\tilde{x}}_i$ is an 
    estimate of the latent position $\mvec{x}_i$, then one can consider the induced distance estimation algorithm $\mathcal{A}_{\mathcal{B}}$ given by 

    \begin{equation*}
        \tilde{\delta}_i = ||\mvec{\tilde{x}}_i-\mvec{\tilde{x}}_{n+1}||
    \end{equation*}
    The triangle inequality implies that 

    \begin{equation}
    \Delta(\mathcal{A}_{\mathcal{B}},\mmat{X}_{n+1})=\max_{i\in [n]}|\delta_i-\tilde{\delta}_i|\leq2\max_{i\in[n+1]}||\mvec{\tilde{x}}_i-\mvec{x}_i||         
    \end{equation}
    Hence in the case of position estimation algorithms, one can 
    replace the metric $\Delta(\mathcal{A}_{\mathcal{B}},\mmat{X}_{n+1})$ by $\mathrm{D}(\mathcal{B},\mmat{X}_{n+1})\coloneqq 2\max_{i\in[n+1]}||\mvec{\tilde{x}}_i-\mvec{x}_i||$. 
    For position estimation algorithms $\mathcal{B}$, we use the slightly abusive notation and write $\mathcal{B}$-ENW instead of $\mathcal{A}_{\mathcal{B}}$-ENW\@. 
\end{remark}
We will analyze the performance of the Nadaraya-Watson estimator with estimated distances $\mvec{\tilde{\delta}}$ in terms of the metric~\eqref{eq:DELTA_distances}.
In contrast to the Graphical Nadaraya-Watson estimator, the distance estimation approach allows for a choice of a kernel function $\phi$ as well as a bandwidth~$\tau$.
Throughout this section, we will make the following two assumptions.

\begin{assumption}\label{ass:phi_conditions}
    The kernel function $\phi\colon [0,\infty)\to[0,1]$ is non-negative, compactly supported and non-vanishing in a neighborhood of $0$,
    i.e.\ there are $M_1,M_2>0$ such that

    \begin{equation*}
    \frac{1}{2}\mathbb{I}\left[t\leq M_1\right]\leq \phi\big(t\big)\leq \mathbb{I}\left[t\leq M_2\right]
    \end{equation*}  

\end{assumption}

\begin{assumption}\label{ass:number_of_points}
    The positions $\mmat{X}_{n+1}$ and $\tau>0$ are such that the number of points in $\mmat{X}_n = [\mvec{x}_1,\dots,\mvec{x}_n]$ in the $\frac{M_1\tau}{2}$ 
        window around $\mvec{x}_{n+1}$ satisfies
        \begin{equation}\label{eq:TRUE_latent_distance_count}
            M(\tau)\coloneqq \sum_{i=1}^n \mathbb{I}\left(\delta_i\leq \frac{M_1\tau}{2}\right)\geq k_0n\tau^d
        \end{equation}
    for some $k_0>0$.

\end{assumption}
Assumption~\ref{ass:phi_conditions} is the same as Assumption~\ref{ass:K_1}, the major difference being that $\phi$ is user chosen whereas $K$ is implicit in the definition of the LPM\@.
In particular, one can \textit{choose} the constants $M_1,M_2$ as well.
 For example, 
the function $\phi_0\colon [0,\infty)\to [0,1]$ given by $\phi_0(t)=\mathbb{I}(t\leq 1)$ with $M_1=M_2=1$ is a valid choice for ENW averaging. 
Assumption~\ref{ass:number_of_points} is also standard in the NW literature. Loosely speaking, it guarantees that the points are sufficiently 
scattered across their support, relative to the bandwidth $\tau$. For example,
in a random sample of i.i.d.\ points with distribution $p$ satisfying Assumptions~\ref{ass:measure_retention} and~\ref{eq:density_ass},
Assumption~\ref{ass:number_of_points} fails to hold with probability $\mathcal{O}(e^{-n\tau^d})$. 

Our aim is to prove a guarantee on the $\mathcal{A}$-Estimated Nadaraya Watson estimator with \textit{estimated} distances $\mvec{\tilde{\delta}}=\mathcal{A}(\mmat{A})$ given by
\begin{equation}\label{eq:enw_definition}
    \hat{f}^{\mathcal{A}}_{\ENW,\tau}(\mvec{x}_{n+1}) = \begin{cases}
       \displaystyle\frac{\sum_{i=1}^n y_i\phi\left(\frac{\tilde{\delta}_i}{\tau}\right)}{\sum_{i=1}^n\phi\left(\frac{\tilde{\delta}_i}{\tau}\right)}  
 \quad &\text{if}\quad \, \displaystyle \sum_{i=1}^n\phi\left(\frac{\tilde{\delta}_i}{\tau}\right)>0\\
   0 \quad &\text{otherwise}\\
\end{cases}
\end{equation}
Note that even though $\mvec{x}_{n+1}$ appears in the notation of~\eqref{eq:enw_definition}, 
the latent position $\mvec{x}_{n+1}$ \emph{is not} fed into $\mathcal{A}$-ENW\@. This is only 
a convention in order to stick close to the notation of the classical ML regression literature.
If the algorithm $\mathcal{A}$ is sufficiently accurate in the estimation of the latent distances, we claim  
that the error in the subsequent NW procedure will preserve the classical rate~\eqref{eq:NW_rate}. 
A precise statement of this claim is given in the following theorem.
\begin{theorem}\label{thm:perturbed_nw_thm}
Suppose that $\mmat{X}_{n+1}$ and $\tau>0$ satisfy Assumption~\ref{ass:number_of_points} and that    
the regression function $f$ satisfies the regularity Assumption~\ref{F_1}, i.e.\ it is
Hölder regular with exponent $0<a\leq 1$.
Additionally, suppose that
\begin{equation*}
    \mathrm{\Delta}(\mathcal{A},\mmat{X}_{n+1})\leq\frac{M_1\tau}{2}
\end{equation*}
Then    
    \begin{equation*}
  \mathbb{E}_{\mvec{\epsilon}}\left(|\hat{f}^{\mathcal{A}}_{\ENW,\tau}(\mvec{x}_{n+1})-f(\mvec{x}_{n+1})|^2\right) \leq C_1\tau^{2a}+\frac{4\sigma^2}{k_0n\tau^d}   
    \end{equation*}
where $C_1 = 2L^2\left[{(\frac{M_1}{2}+M_2)}^{2a}\right]$
\end{theorem}

\begin{figure}
\centering
\includegraphics[width=0.6\textwidth]{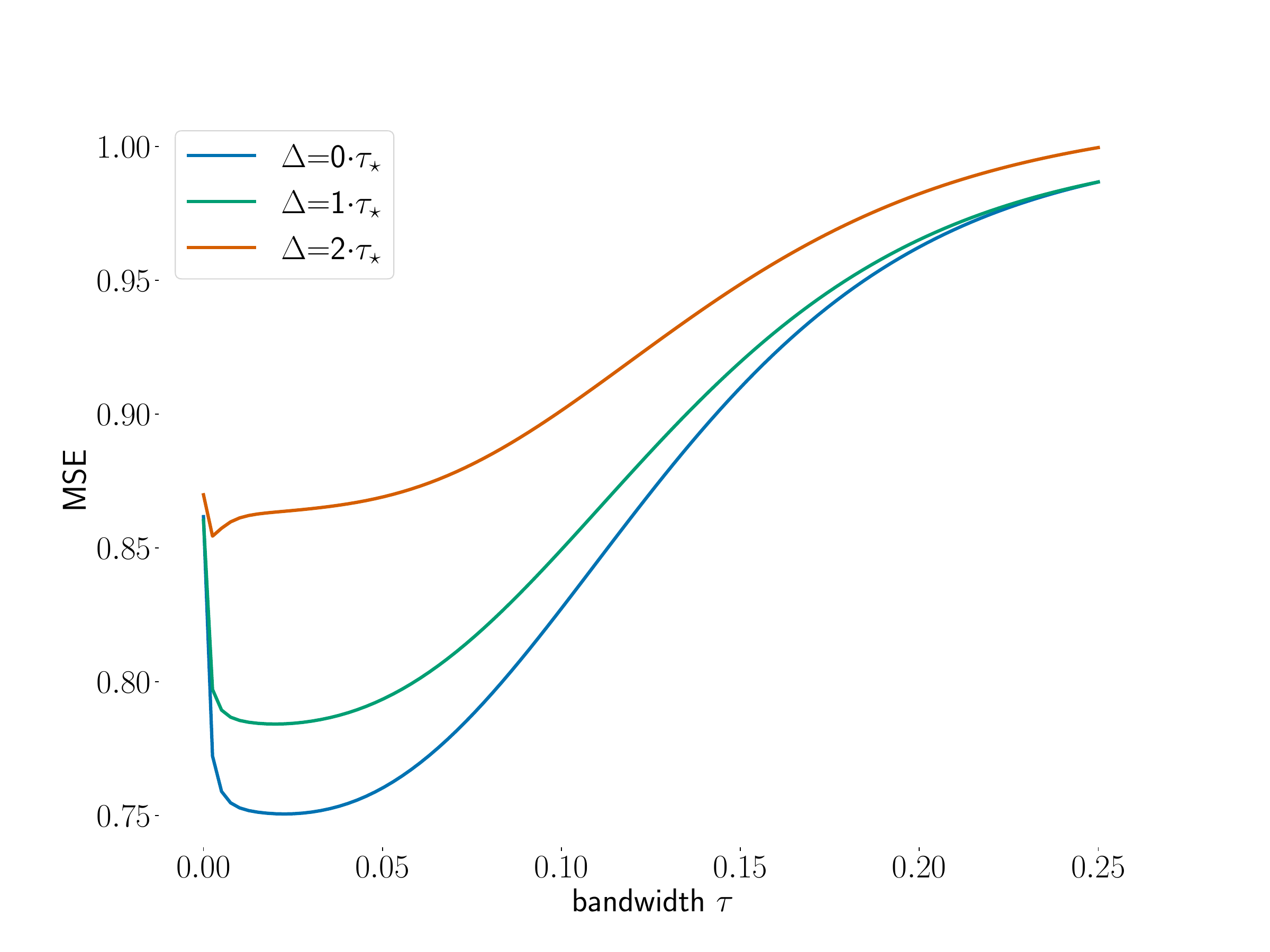}
\caption{Bias Variance Tradeoff Curves for NW under perturbation. Sample size $n=500$, label $y_i = \sin(4\pi\mvec{x}_i)+\epsilon_i$ with 
$\epsilon_i\sim \mathcal{N}(0,1.5)$ and $\mvec{x}_i\sim \Unif[0,1]$ }\label{fig:perturbed_nw_errors}
\end{figure}

Theorem~\ref{thm:perturbed_nw_thm} states that if the algorithm $\mathcal{A}$ has distance estimation  
error $\mathrm{\Delta}(\mathcal{A},\mmat{X}_{n+1})$~\eqref{eq:DELTA_distances} below $\frac{M_1\tau}{2}$,
$\mathcal{A}$-ENW averaged over the additive noise~\eqref{ass:additive_noise} achieves the classical NW rate~\eqref{eq:NW_rate} 
(up to a multiplicative constant). However, the algorithm $\mathcal{A}$ acts 
on the random matrix $\mmat{A}$ and hence even in the case when $\mmat{X}_{n+1}$ are treated as fixed, $\mathrm{\Delta}(\mathcal{A},\mmat{X}_{n+1})$ is
a random variable that depends on the edge variables $\mmat{\mathcal{U}}$. When the algorithm $\mathcal{A}$ fails to estimate distances within 
precision $\frac{M_1\tau}{2}$, we do not expect that the subsequent NW averaging procedure will yield interesting results.

We justify this claim by numerical evidence. Let us consider univariate positions $\mvec{x}_i\in[0,1]$, and perturbations given by 
$\mvec{x}_{i,\Delta}=\mvec{x}_{i}+\Delta\mvec{u}_i$, where $\mvec{u}_i\in [-1,1]$ are uniform variables and $\Delta = k\tau_{\star}$, $k=0,1,2$.
We compute \emph{smoothed} values 

\begin{equation}\label{eq:smoothed_nw_vals}
    \hat{f}_{\Delta,\tau}(\mvec{x}_i)=\frac{\sum_{j=1}^n y_j\phi\left(\frac{|\mvec{x}_{j,\Delta}-\mvec{x}_{i,\Delta}|}{\tau}\right)}{\sum_{j=1}^n \phi\left(\frac{|\mvec{x}_{j,\Delta}-\mvec{x}_{i,\Delta}|}{\tau}\right)}   
\end{equation}
which can be interpreted as predictions for the value $f(\mvec{x}_i)$ based on a Nadaraya-Watson 
estimator with design $\mmat{X}_{\Delta}=[\mvec{x}_{1,\Delta},\ldots,\mvec{x}_{n,\Delta}]$, 
labels $\mvec{y}=f(\mmat{X}_{n})+\mvec{\epsilon}$ and bandwidth $\tau$.
We then compute the (smoothed) Mean Squared Error

\begin{equation*}
    \MSE\left(\hat{f}_{\Delta,\tau},f\right)=\frac{1}{n}\sum_{i=1}^n {\left(\hat{f}_{\Delta,\tau}(\mvec{x}_i)-f(\mvec{x}_i)\right)}^2
\end{equation*}
and we plot it as a function of $\tau$ 
(see Figure~\ref{fig:perturbed_nw_errors}). 
We see that for $k=2$, i.e.\@ when the precision is $2\tau_{\star}$, the bias-variance tradeoff curve does not exhibit the
same behavior as for $k=1$, which is closer to the bias-variance tradeoff curve with exact positions ($k=0$).

Therefore, we opt to control the probability of the failure of the algorithm $\mathcal{A}$ defined as
\begin{equation}\label{algo_rec_prob_failure}
   p_{\tau}(\mathcal{A},\mmat{X}_{n+1})=\mathbb{P}_{\mmat{\mathcal{U}}}\left(\mathrm{\Delta}\left(\mathcal{A}\left({\mmat{A}}\right),\mmat{X}_{n+1}\right)>\frac{M_1\tau}{2}\right)
\end{equation}
The following result provides a bound on the pointwise risk~\eqref{eq:node_reg_risk_pointwise}, which we recall 
takes expectation both over additive label noise and the random edges in the graph.
\begin{theorem}{\textbf{(ENW risk rate (deterministic design))}}\label{thm_enw_fixed_design}
     Suppose that $\mmat{X}_{n+1}$ and $\tau>0$ satisfy Assumption~\ref{ass:number_of_points} and the regression function $f$ satisfies Assumption~\ref{F_1}. 
Then 
\begin{equation*}
    \mathcal{R}_{g}\left(\hat{f}^{\mathcal{A}}_{\ENW,\tau}(\mvec{x}_{n+1}),f(\mvec{x}_{n+1})\right)\leq C_0p_{\tau}(\mathcal{A},\mmat{X}_{n+1}) 
+ C_1\tau^{2a}+\frac{4\sigma^2}{k_0n\tau^d}
\end{equation*}
where $C_0 = 4B^2+2\sigma^2$, $C_1 = 2L^2\left[{(\frac{M_1}{2}+M_2)}^{2a}\right] $
\end{theorem}

From Theorem~\ref{thm_enw_fixed_design} it follows that if $p_{\tau}(\mathcal{A},\mmat{X}_{n+1})$ is smaller than the classical NW-rate~\eqref{eq:NW_rate},
$\mathcal{A}$-ENW will achieve, up to a multiplicative constant, the same rate~\eqref{eq:NW_rate} in~$\tau$.
In particular, if this is true for the optimal bandwidth $\tau_{\star}$, 
then $\mathcal{A}$-Estimated Nadaraya Watson \emph{can achieve the optimal non-parametric NW rate} $n^{-\frac{2a}{2a+d}}$.

The LPM literature on position
estimation~\citep{Arias-Castro,dani,Giraud-Verzelen} typically establish rate of
convergence for $\Delta(\mathcal{A},\mmat{X}_{n+1})$, i.e.\ there exist several
results which provide rates $r_n>0$, such that
$p_{r_n}(\mathcal{A},\mmat{X}_{n+1})$ is overwhelmingly small, typically of
order $\mathcal{O}(1/n)$ (much lower than the rate~\eqref{eq:NW_rate} for any
$\tau>0$). Since $p_{\tau}(\mathcal{A},\mmat{X}_{n+1})$ is decreasing function
in $\tau$, it follows that we can match the classical NW-rate~\eqref{eq:NW_rate}
for any $\tau\gtrsim r_n$.

Several results~\citep{Arias-Castro},~\citep{dani} indicate that the sparsity of the graph as 
dictated by the length-scale $h_g$ plays 
a key role in establishing the rate $r_n$. To the best of our knowledge, theoretical
understanding of the relationship between the length-scale $h_g$ and the rate
$r_n$ in the general compactly supported kernel LPM setting is incomplete. As a consequence 
we cannot characterize completely the pairs of values $(h_g,\tau_{\star})$ for which optimal rates~\eqref{eq:optimal_nw_rate}
are achievable. In
the next section we consider two approaches for the under-averaging and
over-averaging regimes respectively, and point out several instances of LPMs for
which we can get optimal NW performance by using $\mathcal{A}$-ENW\@.

\subsection{Distance and Position estimation algorithms}\label{sec:position_estimation_algos}
In this section we have a glance at the existing literature on distance and position estimation algorithms and discuss some
implications for the node regression problem. There are several estimators in the context of LPMs,
but the theoretical analysis remains limited. We will go through a few examples, focusing on consequences for the 
node regression estimator $\mathcal{A}$-ENW\@. Given a LPM with kernel of the shape~\eqref{eq:link_function}, the length-scale $h_g$ will play an important role in the 
probability of failure~\eqref{algo_rec_prob_failure}. Recall that in the classical regression setting, for label $\mvec{y}=f(\mmat{X}_n)+\mvec{\epsilon}$ with $f$ satisfying~\ref{F_1}, 
and $\mvec{\epsilon}$ satisfying~\ref{ass:additive_noise}, the rate~\eqref{eq:NW_rate} achieves a minimal value $C_{\star}n^{-\frac{2a}{2a+d}}$ for 
$\tau_{\star}=c_{\star}n^{-\frac{1}{2a+d}}$, where $c_{\star},C_{\star}>0$ depend on $L$ and $\sigma^2$. 

\begin{subsubsection}{The Shortest Path Algorithm}\label{subsec:shortest_path_algo}

    The Shortest Path Algorithm $\mathcal{A}_{sp}$ is the simplest and cheapest approach to distance estimation. 
    The idea behind it is that the shortest path distance in the graph between nodes~$i$ and $j$ 
    should approximate (up to a scaling factor $h_g$) the distance of the latent positions~$\mvec{x}_i$ and $\mvec{x}_j$. 
    This algorithm is the subject of study of~\citep{Arias-Castro}. In this work it is shown that for \textit{any} distance estimator $\mvec{\hat{d}}$
    based on the adjacency matrix $\mmat{A}$, there will be some configuration of points $\mmat{X}_{n+1}$ such that 
    for a random geometric graph with length-scale $h_g$, at least half of the quantities $|\delta_i-\hat{\delta}_i|$~\eqref{eq:latent_distances} are of order $\Omega(h_g)$. 
    In particular, $\Delta(\mathcal{A},\mmat{X}_{n+1})$ is of order $\Omega(h_g)$ and hence, using our results (Theorem~\ref{thm_enw_fixed_design}), 
    this particular approach could yield optimal rates only in the under-averaging regime $h_g\ll \tau_{\star}$.
    In order to explore the implications for the downstream task of node regression, we will use the more general result Theorem 3 of~\citep{Arias-Castro}.  
    In our notation, they show if the latent positions~$\mmat{X}_{n+1}$ are uniformly spread out 
    on a convex body $Q$, i.e.\

    \begin{equation}\label{eq:arias-castro-density}
        \Lambda(\mmat{X}_n)\coloneqq\sup_{\mvec{x}\in Q}\min_{i\in [n]} ||\mvec{x}_i-\mvec{x}|| \leq \epsilon       
    \end{equation}
    and the kernel function $K$~\eqref{eq:link_function} is supported on $[0,1]$ and satisfies $K(t)\geq C_0{(1-t)}^{A}$ for some $C_0>0,A\geq 0$, then 
    there exists $C_2>0$ (depending only on $A$ and $C_0$) such that
    \begin{equation}\label{eq:arias_castro_proper_kernel}
        \mathbb{P}_{\mmat{\mathcal{U}}}\left(\mathrm{\Delta}(\mathcal{A}_{sp},\mmat{X}_{n+1}) > C_2\left(h_g+{\left(\frac{\epsilon}{h_g}\right)}^{\frac{1}{1+A}}\right)\right)\leq\frac{1}{n}
    \end{equation}
    Building on their result, we get the following corollary. 
    For the sake of simplicity we omit some constants in the analysis (in particular $\sigma^2$ and $B$), which amounts to asymptotic study with large $n$, where $h_g$ is allowed to depend on $n$.
    
    \begin{corollary}\label{col:arias-castro-col} Suppose that $d=1$, $K$ is supported on $[0,1]$ with $K(t)\geq C_0{(1-t)}^{A}$ 
        for some $C_0>0,0\leq A < 1$, and $\mmat{X}_{n+1}$ are i.i.d.\ with density $p$ that satisfies 
        Assumptions~\ref{ass:measure_retention} and~\ref{eq:density_ass}.
        Furthermore, suppose that the regression function $f$ satisfies Assumption~\ref{F_1} with $a>\frac{1+A}{2}$. Finally, suppose that $\mmat{A}$ is an adjacency matrix of a LPM random graph
        with link function~\ref{eq:link_function} s.t. 
        \begin{equation*}
            \log(n)n^{-\frac{2a-A}{1+2a}}\lesssim h_g\lesssim n^{-\frac{1}{1+2a}}.
        \end{equation*}        
        Then, with high probability over $\mmat{X}_{n+1}$,

        \begin{equation*}
            \inf_{\tau>0}\mathcal{R}_{g}\left(\hat{f}^{\mathcal{A}}_{\ENW,\tau}(\mvec{x}_{n+1}),f(\mvec{x}_{n+1})\right)\lesssim n^{-\frac{2a}{2a+1}}
        \end{equation*}

      \end{corollary}
    The time complexity of the shortest path algorithm is $\mathcal{O}\left(n\log(n)nh_g^d\right)$, in the growing degree regime $nh_g^d = \Omega(1)$ and 
    $\mathcal{O}\left(n\log(n)\right)$ in the bounded degree regime $nh_g^d = \Theta(1)$. 
\paragraph{Generalization to the Random Geometric Graphs in dimension $d\geq 2$}
    Note that for the special case of random geometric graph,
  $p_{\tau}(\mathcal{A},\mmat{X}_{n+1})\in\{0,1\}$, since there is no edge
  randomness $\mmat{\mathcal{U}}$.
    The approach of~\citep{dani} consists in refining the shortest path distances by taking into account the number of common neighbors of the nodes. They propose 
    a distance estimation algorithm, and building on that algorithm, they construct a position recovery algorithm $\mathcal{B}_{rgg}$. In our notation,
setting $Q={[0,1]}^d$ to be the $d$-dimensional cube with latent positions following a uniform distribution on $Q$,
they obtain the following bound (with high probability over the sampled points $\mmat{X}_n$)

\begin{equation}\label{eq:breaking_omega_paper}
    \mathrm{D}\left(\mathcal{B}_{rgg},\mmat{X}_{n+1}\right)\leq C_{d}\begin{cases}
        {(nh_g^d)}^{-\frac{2}{d+1}} \quad &\text{if}\, \quad 1<nh_g^d<n^\frac{d+1}{2d}\\
        \sqrt{\log(n)}n^{-\frac{1}{d}} \quad &\text{if}\, \quad n^\frac{d+1}{2d}\leq nh_g^d < n\\
    \end{cases}
\end{equation}
Note that in the over-averaging regime, for $nh_g^d$ sufficiently large, in particular for $nh_g^d \geq n^{\frac{d+1}{2d}}$, 
we have the rate $r_n = \sqrt{\log(n)}n^{-\frac{1}{d}}$, which does not depend on the length-scale $h_g$, whereas for $nh_g^d<n^{\frac{d+1}{2d}}$, we have the rate 
$r_n = {(nh_g^d)}^{-\frac{2}{d+1}}$, which depends on the length-scale $h_g$. In particular it shows convergence of the algorithm even for sparse 
graphs\footnote{In LPMs with compactly supported kernel functions, the degree scales like $nh_g^d$}. However, in order to obtain nonparametric rates
using this algorithm, we need to be in a certain density regime. More precisely, based on the bound~\eqref{eq:breaking_omega_paper}, we have the following result. 

    \begin{corollary}\label{col:dani-corollary}
        Suppose that $d\geq 2$ and that $\mmat{A}$ is adjacency matrix of a Random Geometric Graph. If $nh_g^d\gtrsim n^{\frac{d+1}{2(d+2a)}}$ 
        then with high probability over the samples $\mmat{X}_{n+1}$, we have
    
        \begin{equation*}
            \inf_{\tau>0}\mathcal{R}_{g}\left(\hat{f}^{\mathcal{A}}_{\ENW,\tau}(\mvec{x}_{n+1}),f(\mvec{x}_{n+1})\right)\lesssim n^{-\frac{2a}{2a+d}}
        \end{equation*}

    \end{corollary}

    The algorithm $\mathcal{B}_{rgg}$ runs in $\mathcal{O}(n^{\omega}\log(n))$, where $\omega<2.373$ is the matrix multiplication constant~\citep{Alman}.

\end{subsubsection}

\begin{subsubsection}{Localize and Refine: Optimal recovery on the Sphere}
Another instance in which optimal recovery is possible is provided by~\citep{Giraud-Verzelen}. 
This work concerns position recovery in the large length-scale regime $h_g \geq c_0>0$ on the sphere $\mathbb{S}^1\subseteq\mathbb{R}^2$.
Our model slightly differs from theirs, as we work with data supported 
on convex bodies in $\mathbb{R}^d$, which excludes surfaces like the sphere, although our results easily 
generalize to the case of smooth manifolds. They propose an algorithm titled \emph{Localize and Refine} 
$\mathcal{B}_{LaR}$ for position recovery on the sphere with  
\begin{equation}\label{eq:localize-and-refine}
    \mathrm{D}(\mathcal{B}_{LaR},\mmat{X}_{n+1}) \leq C\sqrt{\frac{\log(n)}{n}}
\end{equation}
Furthermore, they show that this rate is minimax 
optimal in their setting. Corollary 4.3 in their paper provides a specific link function $K$ s.t.
\begin{equation*}
    p_{\tau}(\mathcal{B}_{LaR},\mmat{X})\leq\frac{9}{n^2}
\end{equation*}
whenever $\tau\geq C\sqrt{\frac{\log(n)}{n}}$. 
In particular, when $\tau_{\star}\gtrsim \sqrt{\frac{\log(n)}{n}}$, $\mathcal{B}_{LaR}$-ENW achieves 
optimal non-parametric rates.
This happens for $\frac{1}{1+2a}<\frac{1}{2}$, or equivalently for $a>1/2$. Thus, in the over-averaging regime $\frac{h_g}{\tau_{\star}}\gg 1 $, 
there are algorithms $\mathcal{B}$ such that $\mathcal{B}$-ENW achieves optimal nonparametric rates for sufficiently regular functions (Hölder exponent $a>1/2$). This is somewhat surprising, as such graphs are extremely dense,
with every degree having order $\geq c_0 n$. The time complexity of this algorithm is polynomial in the number of nodes $n$. 
\end{subsubsection}

\section{Numerical Experiments}\label{simulations}
We study empirically two \emph{position recovery} algorithms based on the ideas of~\citep{Arias-Castro} and~\citep{Giraud-Verzelen}, 
that are intended to treat the under-averaging and over-averaging regime, respectively (see Figure~\ref{fig:delta_vs_hg}). We restrict our attention 
to the one dimensional case, i.e, it is assumed that the latent positions are univariate, uniform i.i.d.\ variables on $Q=[0,1]$. Throughout this section, 
the Kernel Function $K$~\eqref{eq:link_function} is taken to be the gaussian $K(t)=e^{-t^2}$ and $\alpha=1$. 

The first algorithm $\mathcal{B}_{sp}$
is based on the shortest path algorithm, which yields an approximation for the distances~\eqref{eq:latent_distances}. We convert these distances into 
position estimates by using classical Multi Dimensional Scaling~\citep{Torgerson}, abbreviated as cMDS\@. Our implementation computes the 
graph distances using the Floyd-Warshwall algorithm $\FW$, which computes all graph distances in time complexity $\mathcal{O}(n^3)$. For disconnected graphs,
this algorithm will correctly calculate the graph distance of nodes $i$ and $j$ in different connected components to be infinite. However,
since we perform cMDS on the graph distances, we require that all graph distances are finite. Indeed, one can apply cMDS on each connected component, 
providing embeddings for different components which are not comparable with one another. Instead, we opt to implement position recovery algorithms only for \emph{connected 
graphs}.
Algorithm $\mathcal{B}_{sp}$ is expected to outperform GNW in the under-averaging regime $\frac{h_g}{\tau_{\star}}=o(1)$. 
The shortest path $\mathcal{B}_{sp}$ algorithm is given as follows.
\begin{algorithm}[h!]
    \caption{Shortest Path Position Recovery Algorithm $\mathcal{B}_{sp}$}\label{alg:shortest_path}
    \KwIn{Adjacency matrix $\mmat{A}$}
    Compute $\FW(\mmat{A})$;\\
    Return $\cMDS(\FW(\mmat{A}))$;\\
    \KwOut{Positions $\mmat{\hat{X}}\in\mathbb{R}^{n\times1}$}
\end{algorithm}

    The second algorithm $\mathcal{B}_{spectral}$ is more empirical in nature. 
    Recall that in the over-averaging regime we have $h_g\gg \tau_{\star}$. The main idea is to ``shrink the length scale'', i.e.\ to produce 
    a new adjacency matrix $\mmat{A}_{q}$ that indicates if two points are within distance $\tau_{q}$, where $\tau_{q}\ll h_g$.
    In order to achieve this goal, we will \emph{denoise} the adjacency matrix $\mmat{A}$ in the hope to get a more accurate estimate $\hat{\mmat{K}}$ of $\mmat{K}$.
    Keeping in mind that $\hat{\mmat{K}}_{i,j}$ is an estimate of $\mmat{K}_{i,j}$~\eqref{eq:matrix_entries_K} 
    where $K$~\eqref{eq:link_function} is a decreasing function, we can construct $\mmat{A}_q$ by 

    \begin{equation}\label{eq:ls_shrinked}
        {\left[\mmat{A}_q\right]}_{i,j}=\mathbb{I}\left[\hat{\mmat{K}}_{i,j}>q\right]   
    \end{equation}
    Once $\mmat{A}_q$ is constructed, we run the shortest-path algorithm $\mathcal{B}_{sp}$~\eqref{alg:shortest_path} on $\mmat{A}_q$.

    We now explain the construction of the denoised matrix $\hat{\mmat{K}}$. This construction is based on empirical observations,
    and theoretical analysis is out of the scope of this paper.  
    Empirically, we observe that the eigenvalue distribution of the adjacency matrix $\mmat{A}$ admits a Wigner-like semicircular law~\citep{Anderson_Guionnet_Zeitouni_2009}. 
    Indeed, $\mmat{A}=\mmat{K}+\mmat{E}$ where 
    
    \begin{equation}\label{eq:matrix_entries_K}
        {\left[\mmat{K}\right]}_{i,j}=K\left(\frac{||\mvec{x}_i-\mvec{x}_j||}{h_g}\right)
    \end{equation}
     and 
$\mmat{E}$ is a random matrix with centered and independent entries (conditionally on the latent positions).
The spectrum of $\mmat{A}$ is formed from a \emph{bulk} of eigenvalues 
coming from $\mmat{E}$ and only a few eigenvalues of $\mmat{K}$ are separated from this bulk (See Figure~\ref{fig:specdecay_genfig}).
Moreover, the eigenvectors of the adjacency matrix $\mmat{A}$ associated with these eigenvalues separated from the bulk tend to be a very good approximation
for the corresponding eigenvectors of the kernel matrix $\mmat{K}$, whereas as soon as an eigenvalue of $\mmat{A}$ enters in the bulk, its associated 
eigenvector is overwhelmed with noise (see Figure~\ref{fig:eigvecs_genfig}).
\begin{figure}[h]
     
    \centering

    \begin{subfigure}[b]{0.49\textwidth}
        \centering
        \includegraphics[width=\textwidth]{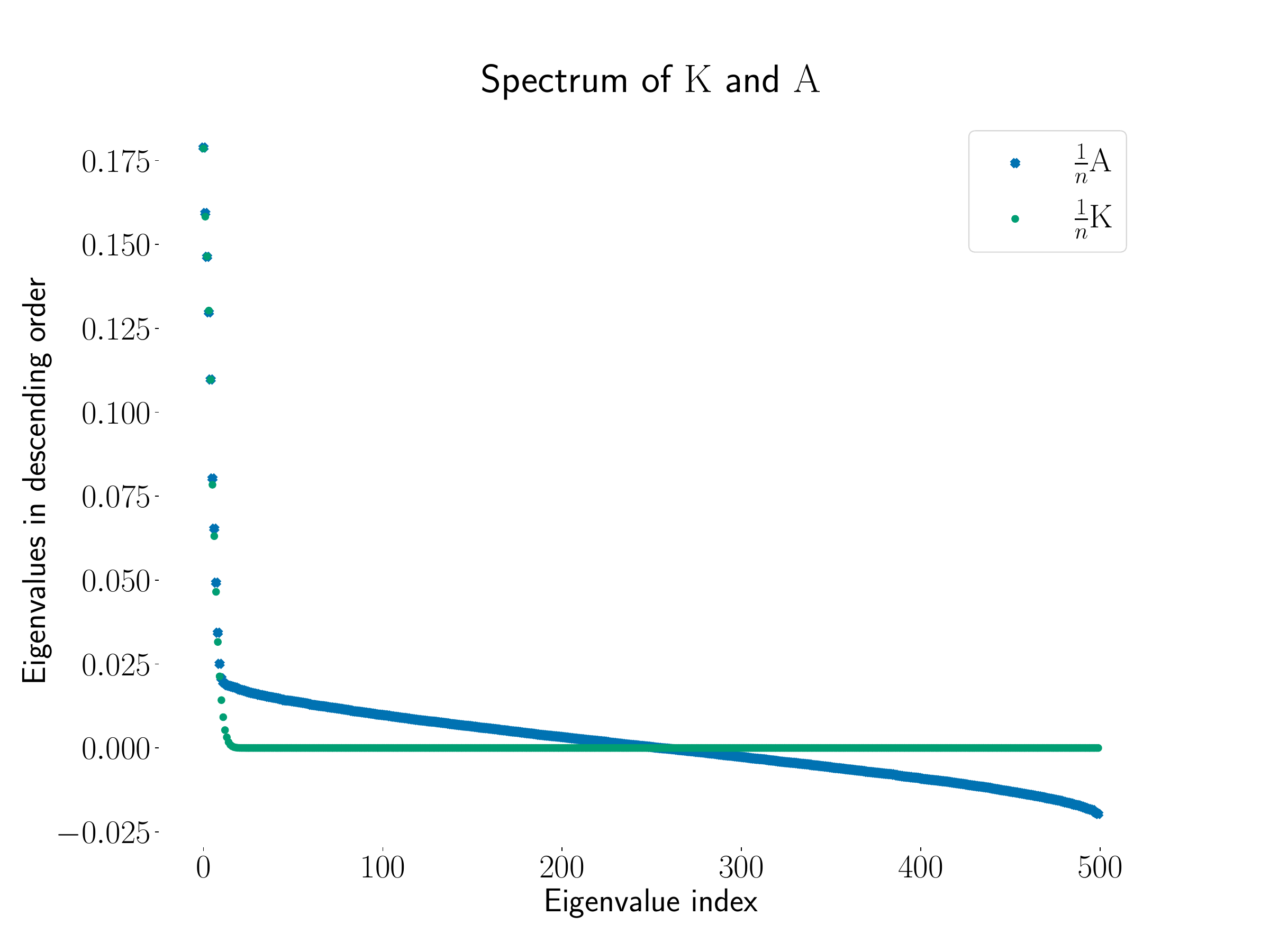}
        \caption{}\label{fig:spectral_decay}
    
    \end{subfigure}
    \hfill
    \begin{subfigure}[b]{0.49\textwidth}
        \centering
        \includegraphics[width=\textwidth]{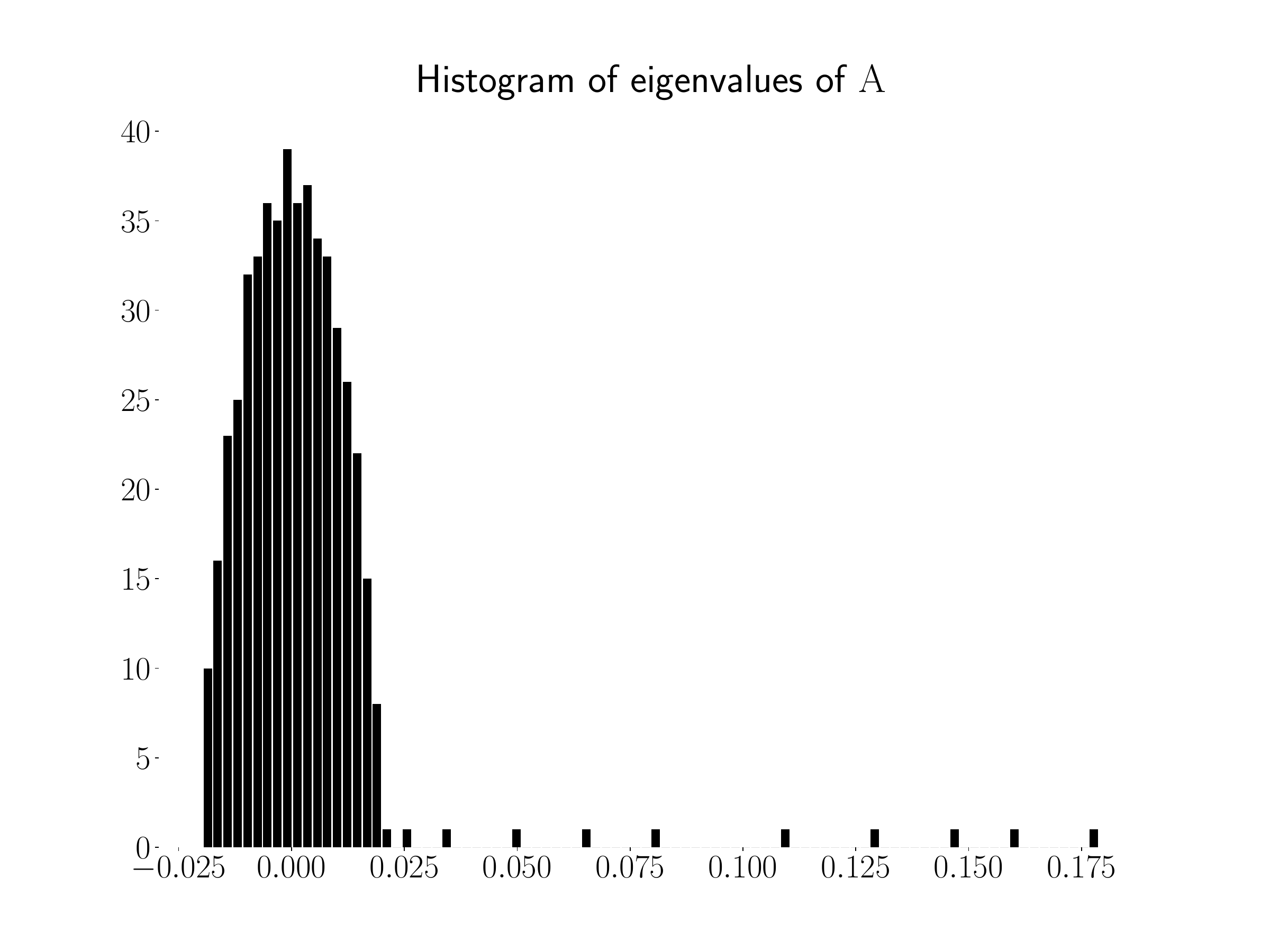}
        \caption{}\label{fig:histogram}
    \end{subfigure}
    \caption{ Illustration for a LPM with sample size $n=500$ and length-scale $h_g=0.1$. 
        Figure~\ref{fig:spectral_decay} shows a scatter plot of descending eigenvalues of $\mmat{K}$ and $\mmat{A}$.
        Interestingly, the first few eigenvalues of $\mmat{A}$ are very close to the corresponding ordered eigenvalues of $\mmat{K}$.
        Figure~\ref{fig:histogram} shows a histogram of Eigenvalues of $\mmat{A}$. The top several eigenvalues of $\mmat{A}$ are well separated 
        from the rest, which fall in the semicircular \textit{bulk}.
    }\label{fig:specdecay_genfig}

\end{figure}

\begin{figure}[h]
     
    \centering
    \begin{subfigure}{0.49\textwidth}
        \centering
        \includegraphics[width=\textwidth]{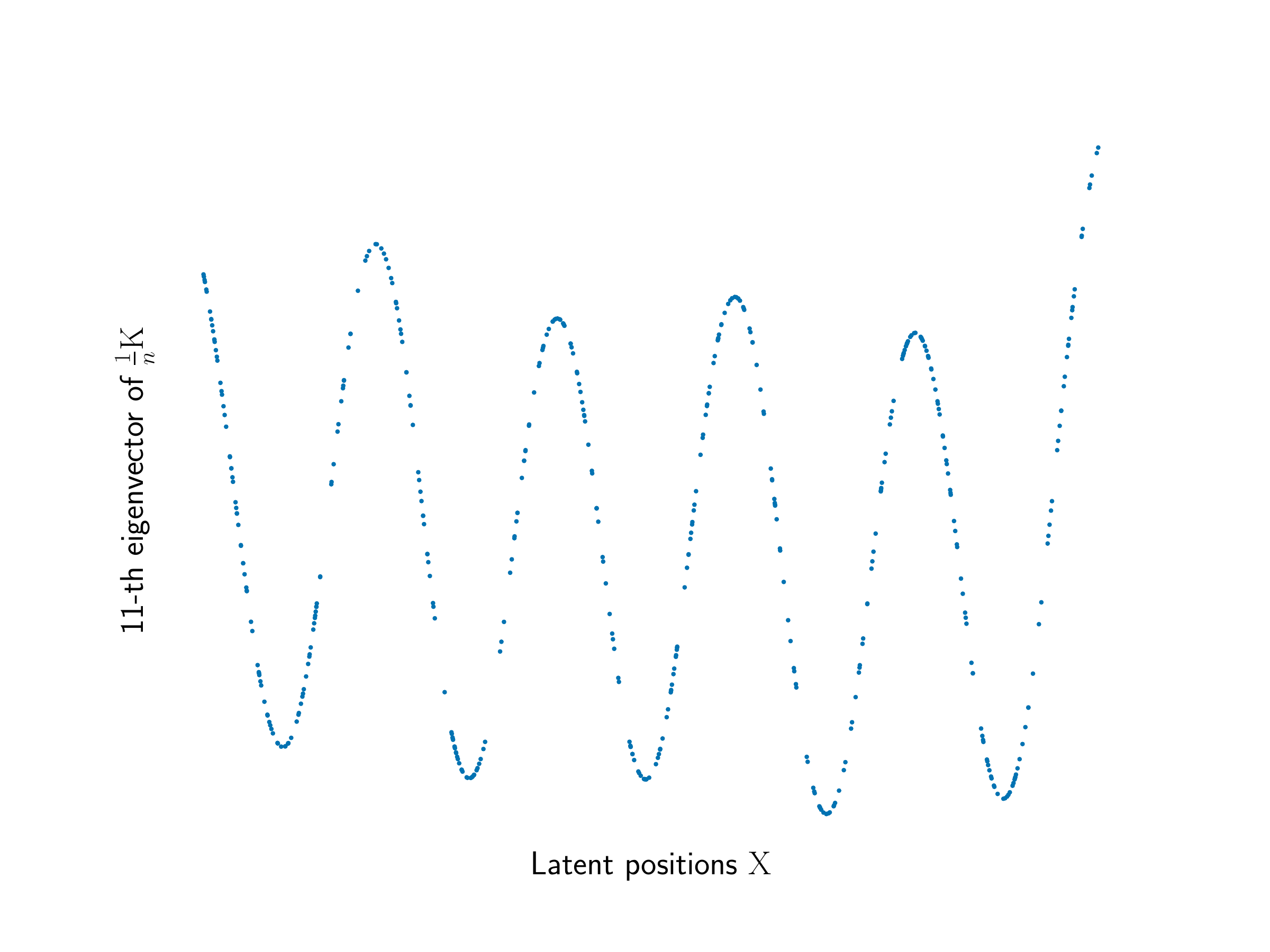}
        \caption{}\label{fig:eigvecs_K_sep}
    \end{subfigure}
    \hfill
    \begin{subfigure}{0.49\textwidth}
        \centering
        \includegraphics[width=\textwidth]{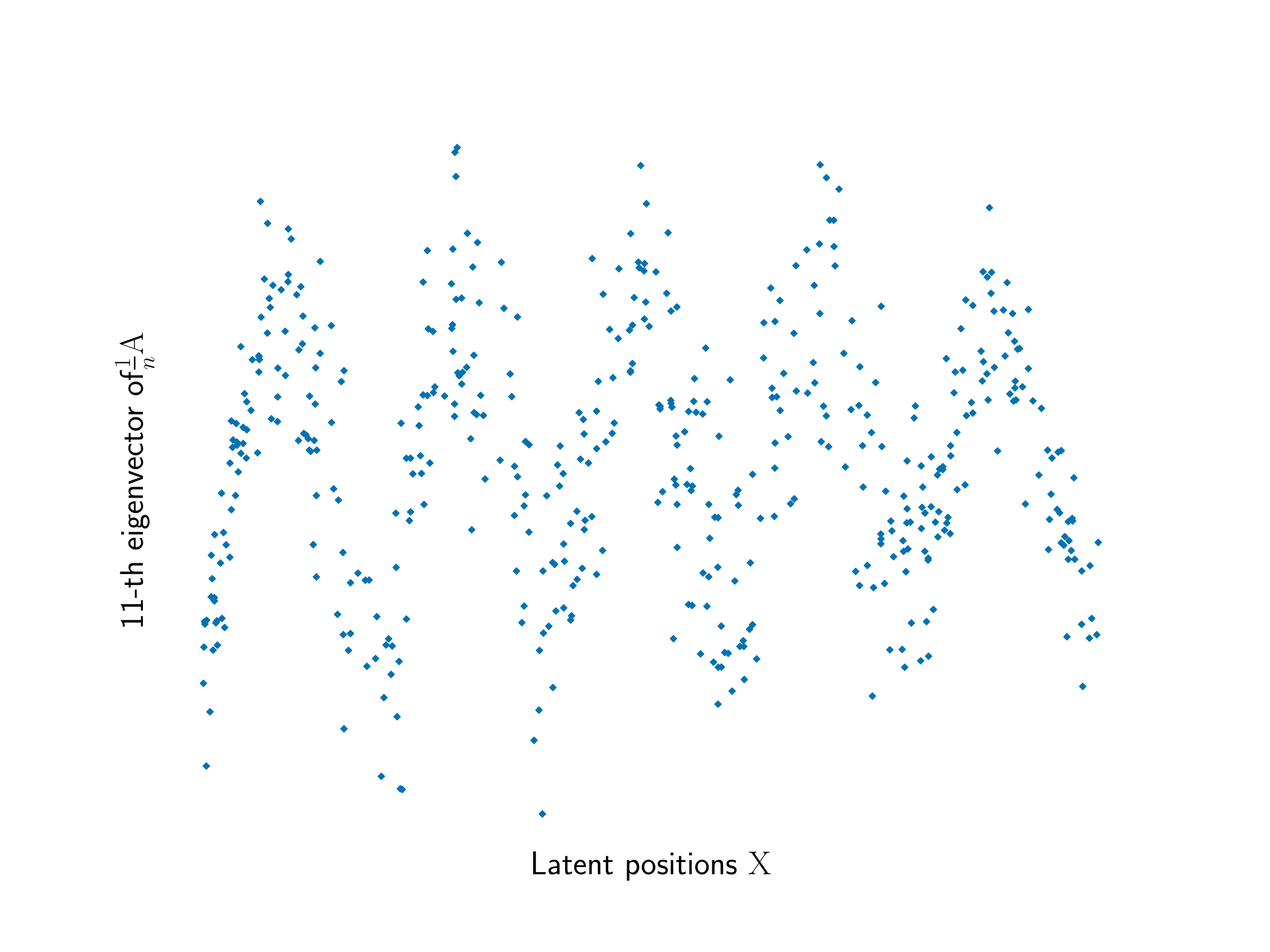}
        \caption{}\label{fig:eigvecs_A_sep}
    \end{subfigure}
    \vfill
    \begin{subfigure}{0.49\textwidth}
        \centering
        \includegraphics[width=\textwidth]{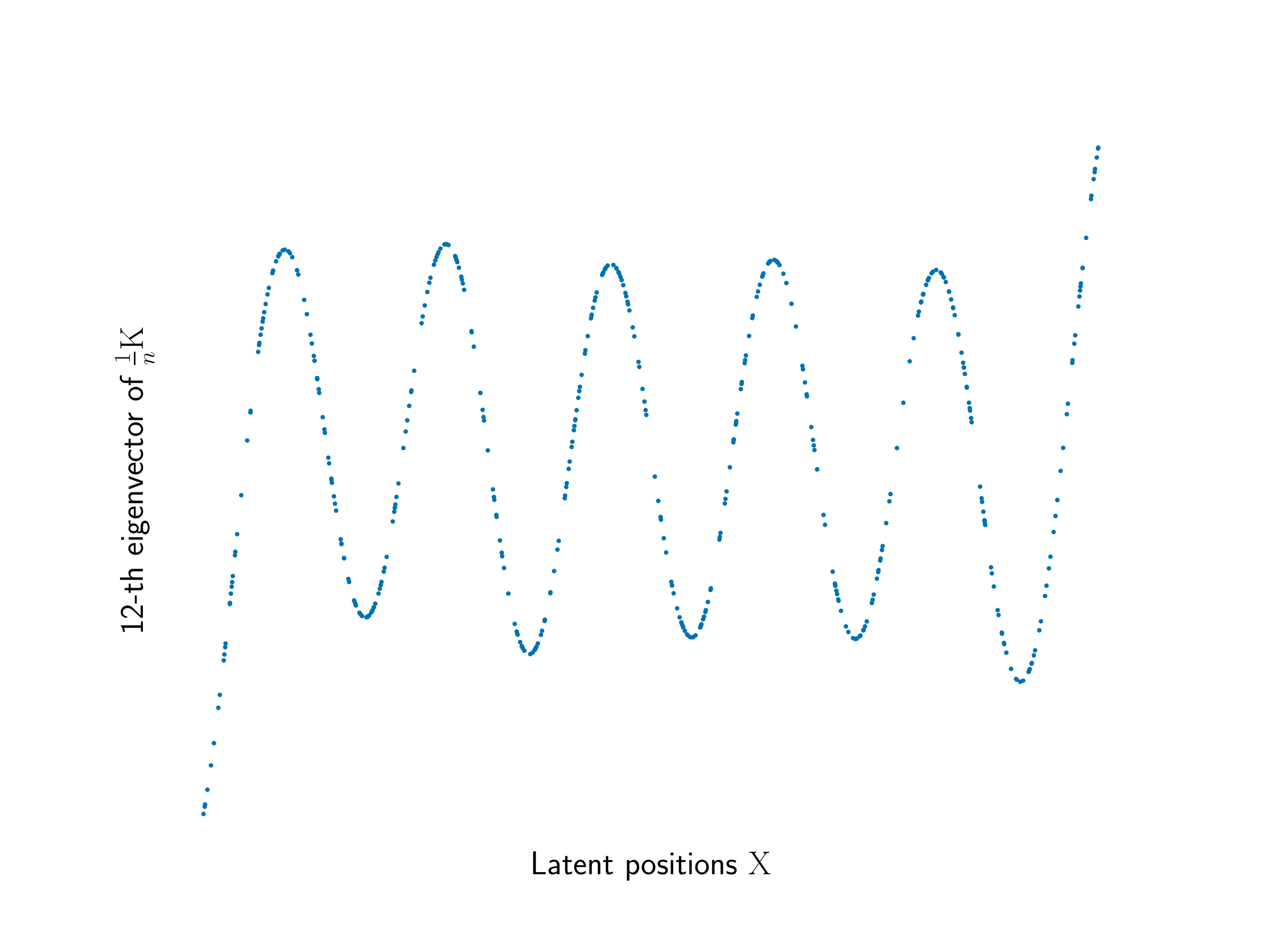}
        \caption{}\label{fig:eigvecs_K_bulk}
    \end{subfigure}
    \hfill
    \begin{subfigure}{0.49\textwidth}
        \centering
        \includegraphics[width=\textwidth]{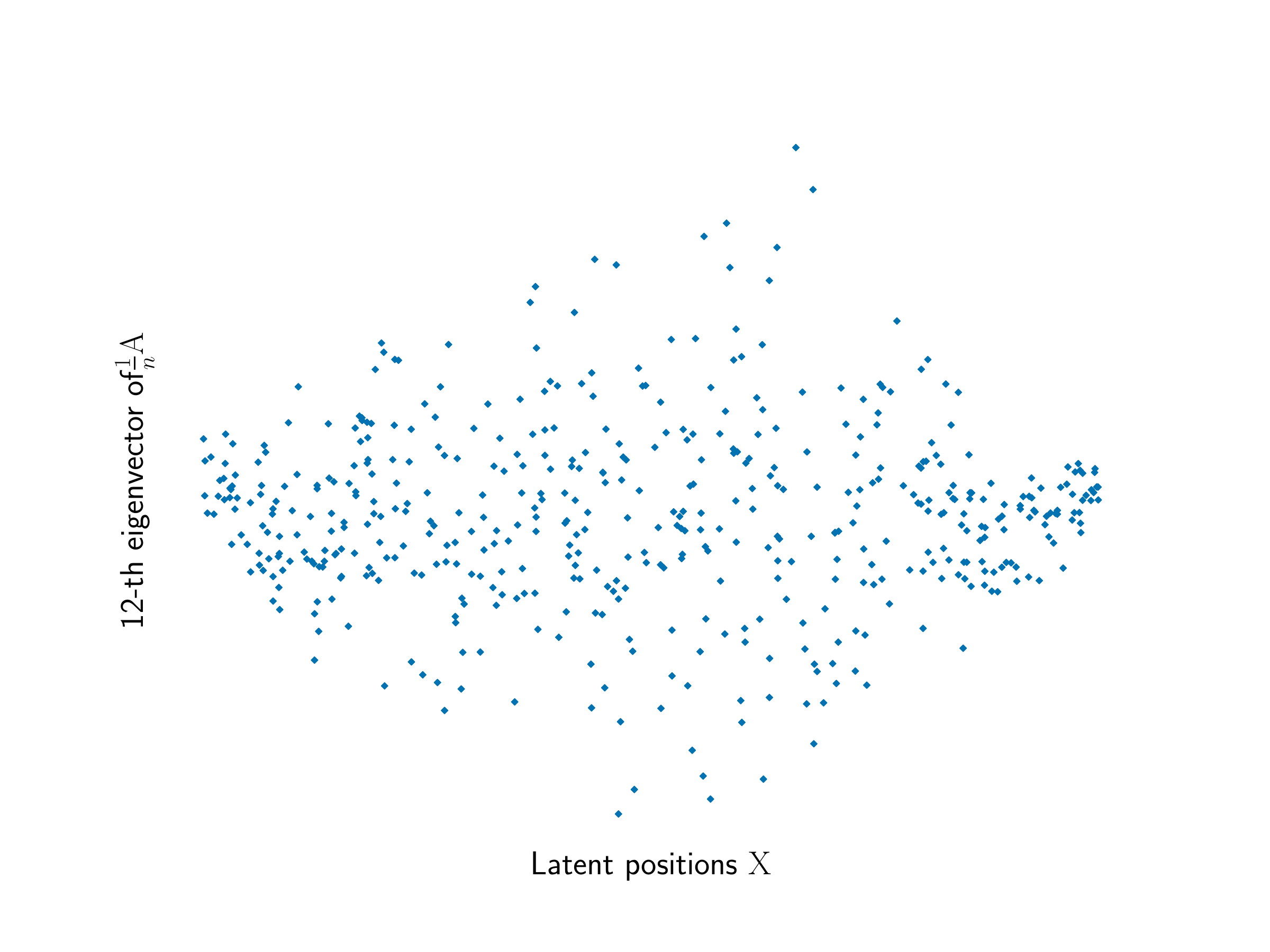}
        \caption{}\label{fig:eigvecs_A_bulk}
    \end{subfigure}
    \caption{Scatter plots of $(\mmat{X}_n,\mvec{u}_j)$ and $(\mmat{X}_{n},\mvec{v}_j)$.
        Figures~\ref{fig:eigvecs_K_sep} and~\ref{fig:eigvecs_A_sep} show a scatter plots of 
        the 11th eigenvector of the matrices $\mmat{K}$ and $\mmat{A}$ respectively. This is the index of the last 
        eigenvalue that separates from the bulk. Figures~\ref{fig:eigvecs_K_bulk} and~\ref{fig:eigvecs_A_bulk} demonstrate the same 
        scatter plot, for the 12th eigenvector of the matrices $\mmat{K}$ and $\mmat{A}$, respectively. This is the index of the 
        first eigenvalue that belongs in the bulk.
    }\label{fig:eigvecs_genfig}
\end{figure}

Let $\mmat{A}=\sum_{i=1}^n \sigma_i\mvec{v}_i\mvec{v}^t_i$ and 
$\mmat{K}=\sum_{i=1}^n \lambda_i \mvec{u}_i\mvec{u}^t_i$ be the spectral decompositions of $\mmat{A}$ and $\mmat{K}$, respectively, where 
the sequences $(\lambda_i)$ and $(\sigma_i)$ are decreasing. For Positive Semi Deffinite (PSD) matrices $\mmat{K}$ our heuristic observation (See Figures~\ref{fig:specdecay_genfig} and~\ref{fig:eigvecs_genfig}) suggests that the low-rank matrix 
$\hat{\mmat{K}}=\sum_{i=1}^r \sigma_i\mvec{v}_i\mvec{v}_i^t$ is a good approximation of $\mmat{K}$, where $r$ is the number of eigenvalues that are ``out of the bulk''. While there are many methods to estimate the bulk, here we resort to a simple \emph{symmetrization trick}:
due to the symmetry of the bulk and the fact that $\mmat{K}$ is PSD, the \emph{most negative} eigenvalue should be a good indication of the the size of the bulk. More precisely, we keep the eigenvalues $\sigma_i$ and their corresponding eigenvectors whenever $\sigma_i>-\sigma_n$, i.e.\ $r$ is the index such that 
$\sigma_r>-\sigma_n\geq \sigma_{r+1}$, where $\sigma_n$ is the smallest (most negative) eigenvalue of $\mmat{A}$. If $\mmat{K}$ was not PSD, one could for instance consider the spacings of the eigenvalues in order to select the threshold for the spectrum. 
The algorithm $\mathcal{B}_{spectral}$ is described in Algorithm~\ref{alg:spectral}.

\begin{algorithm}
    \caption{Spectral Position Recovery Algorithm $\mathcal{B}_{spectral}$}\label{alg:spectral}
    \KwIn{Adjacency matrix $\mmat{A}$, threshold parameter $q$, eigenvalue tolerance $\rho_0$}
    Compute eigenvalue threshold: $ r = \#\{1\leq i\leq n: \sigma_i>-(1+\rho_0)\sigma_n\}$;\\
    Compute low rank matrix $\hat{\mmat{K}} = \sum_{i=1}^r \sigma_i \mvec{v}_i\mvec{v}^t_i $;\\
    Construct new adjacency matrix $\mmat{A}_{q}$ with ${[\mmat{A}_{q}]}_{i,j}=\mathbb{I}\left[\hat{\mmat{K}}_{i,j}\geq q\right]$;\\ 
    Run $\mathcal{B}_{sp}$ on $\mmat{A}_{r,q}$;\\
    \KwOut{Estimated Positions $\mmat{\hat{X}}\in\mathbb{R}^{n\times1}$}
\end{algorithm}

\begin{figure}[h]
    \centering
    \includegraphics[width=0.5\textwidth]{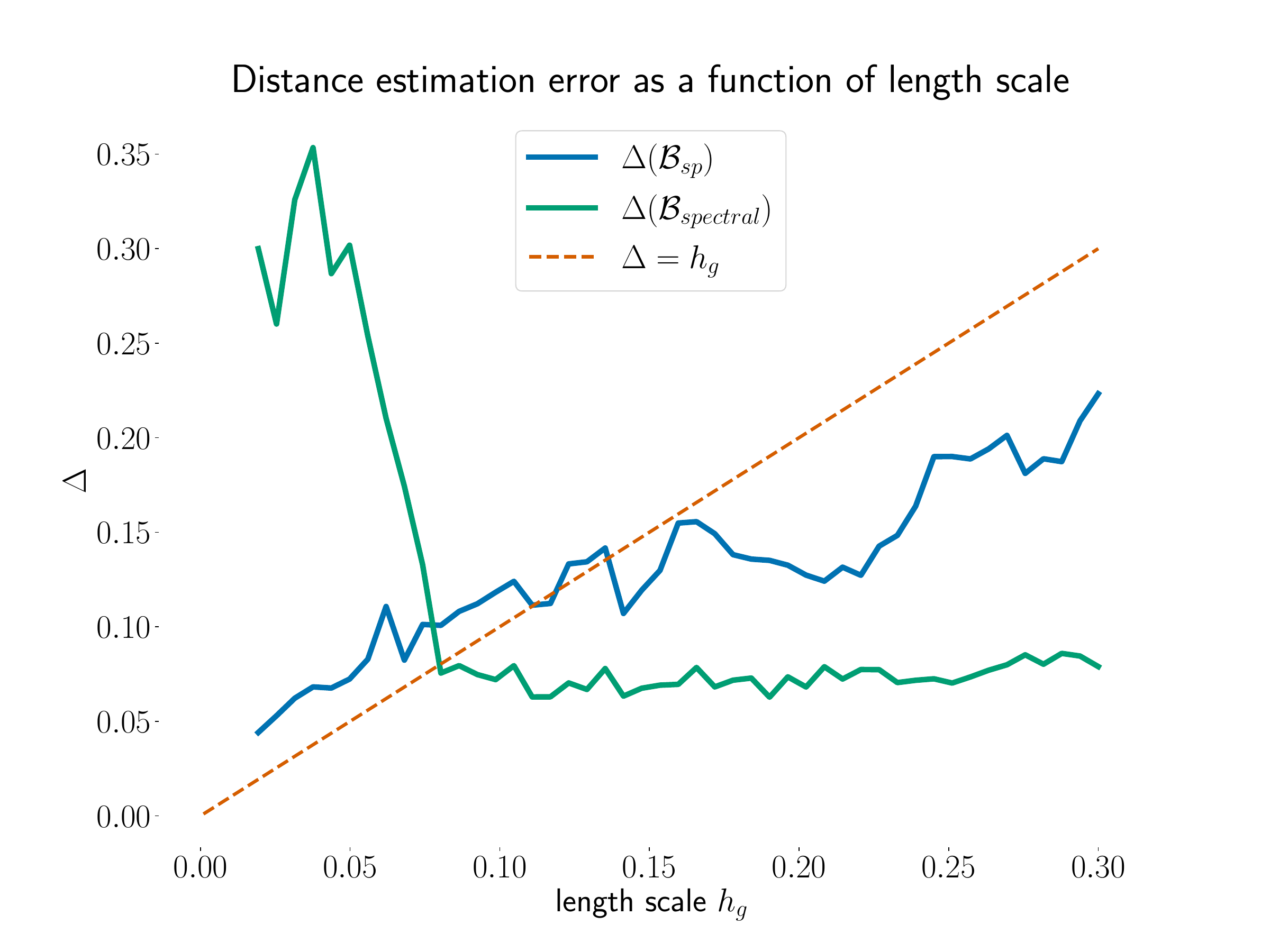}
    \caption{Empirical error of $\mathcal{B}_{sp}$ and $\mathcal{B}_{spectral}$ as a function of the 
    length-scale $h_g$ of the LPM\@.}\label{fig:delta_vs_hg}
\end{figure}

\subsection{Bias-Variance trade-off curves}

We compute the Bias-Variance trade-off curves for GNW and ENW for a wide range of length-scales $h_g$. 
Given parameters $n\in \mathbb{N}$, $m>0$ and $\sigma^2>0$, we consider the model 
\begin{equation}\label{eq:toy_label}
    y_i = \sin(2m\pi\mvec{x}_i)+\epsilon_i
\end{equation}
where $\mmat{X}_n=[\mvec{x}_1,\dots,\mvec{x}_n]$ are i.i.d.\ uniform univariate variables on $[0,1]$ and $\epsilon_i$ are i.i.d.\ gaussian variables
with variance $\sigma^2$. As $m$ increases, so does the number of oscillations of the sine, 
and therefore the optimal bandwidth $\tau_{\star}$ needs to shrink to compensate for the irregularity of the label $\mvec{y}$. 

For a given set of parameters $n,m,\sigma^2$, we approximate $\tau_{\star}$ for the label $\mvec{y}$ given by~\eqref{eq:toy_label} by 
cross validation, i.e.\ we compute $\tau_{\CV}$. Then we consider the grid $G$ of $\mathrm{NUM}_{\PTS}\in\mathbb{N}$ regularly spaced points 
in a window $W_{\tau_{\CV}}$ around $\tau_{\CV}$ of two orders of magnitude, i.e. $W_{\tau_{\CV}} = [0.1\tau_{\CV},10\tau_{\CV}]$. 
The neighborhood of the leftmost endpoint of this interval could be considered as the narrow length-scale regime, whereas 
the neighborhood of the rightmost endpoint could be considered as the over-averaging regime. For each point $p$ in the grid $G$, 
we generate a LPM with length-scale $h_g=p$, on which we run GNW and ENW with the algorithms~\ref{alg:shortest_path} and~\ref{alg:spectral}. 
We report the mean squared error for each algorithm and for each point $p$. This computation is repeated $\mathrm{NUM}_{\MC}\in\mathbb{N}$ times to reduce the variance
due to random edges. Results are displayed in Figure~\ref{fig:enw_gnw_comp}.

\begin{figure}[t]
     
    \centering
    \begin{subfigure}{0.49\textwidth}
        \centering
        \includegraphics[width=1\textwidth]{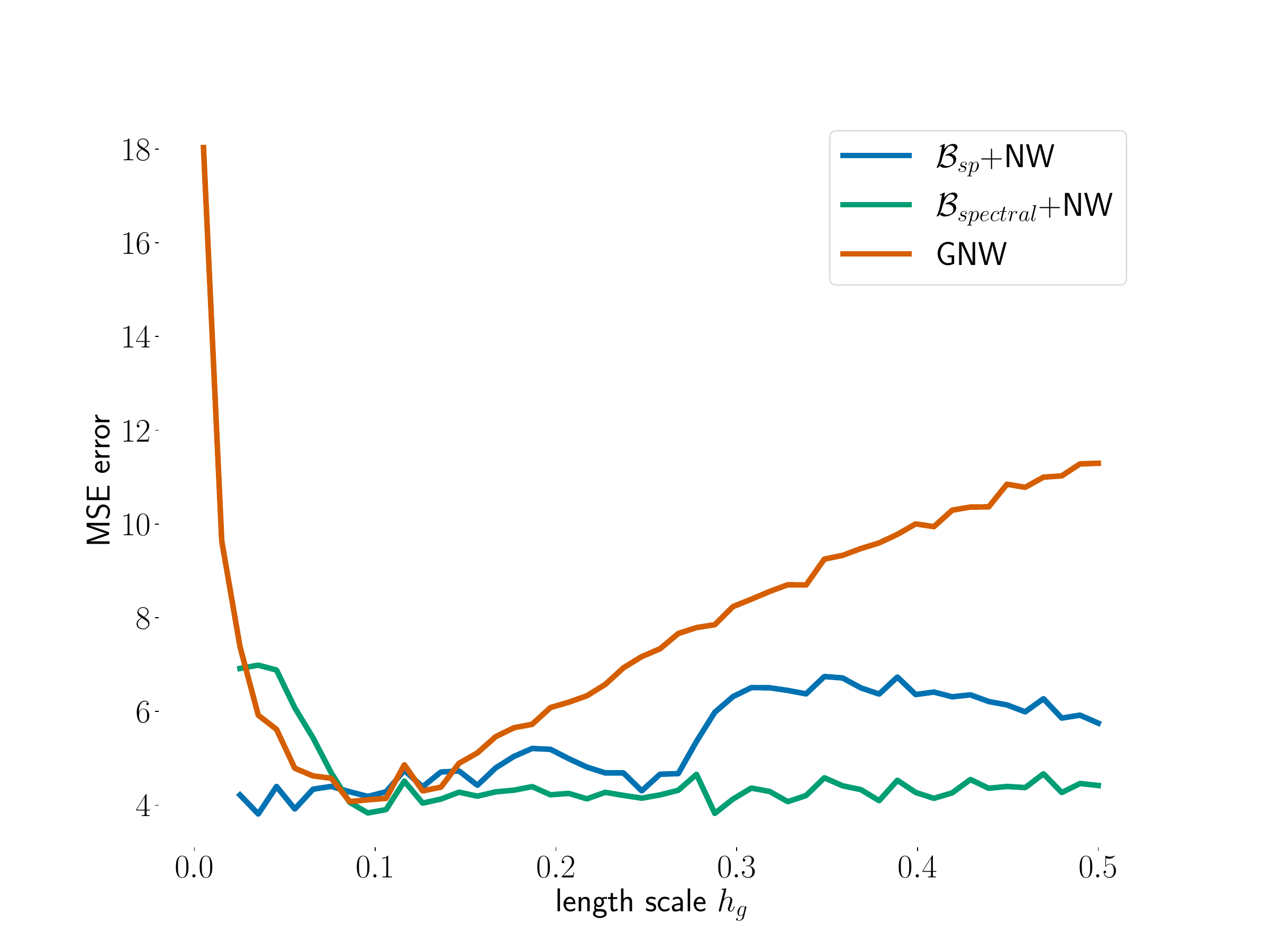}
        \caption{$m=1$, $\sigma^2 = 1.5$}\label{fig:enw_a}
    \end{subfigure}
    \hfill
    \begin{subfigure}{0.49\textwidth}
        \centering
        \includegraphics[width=\textwidth]{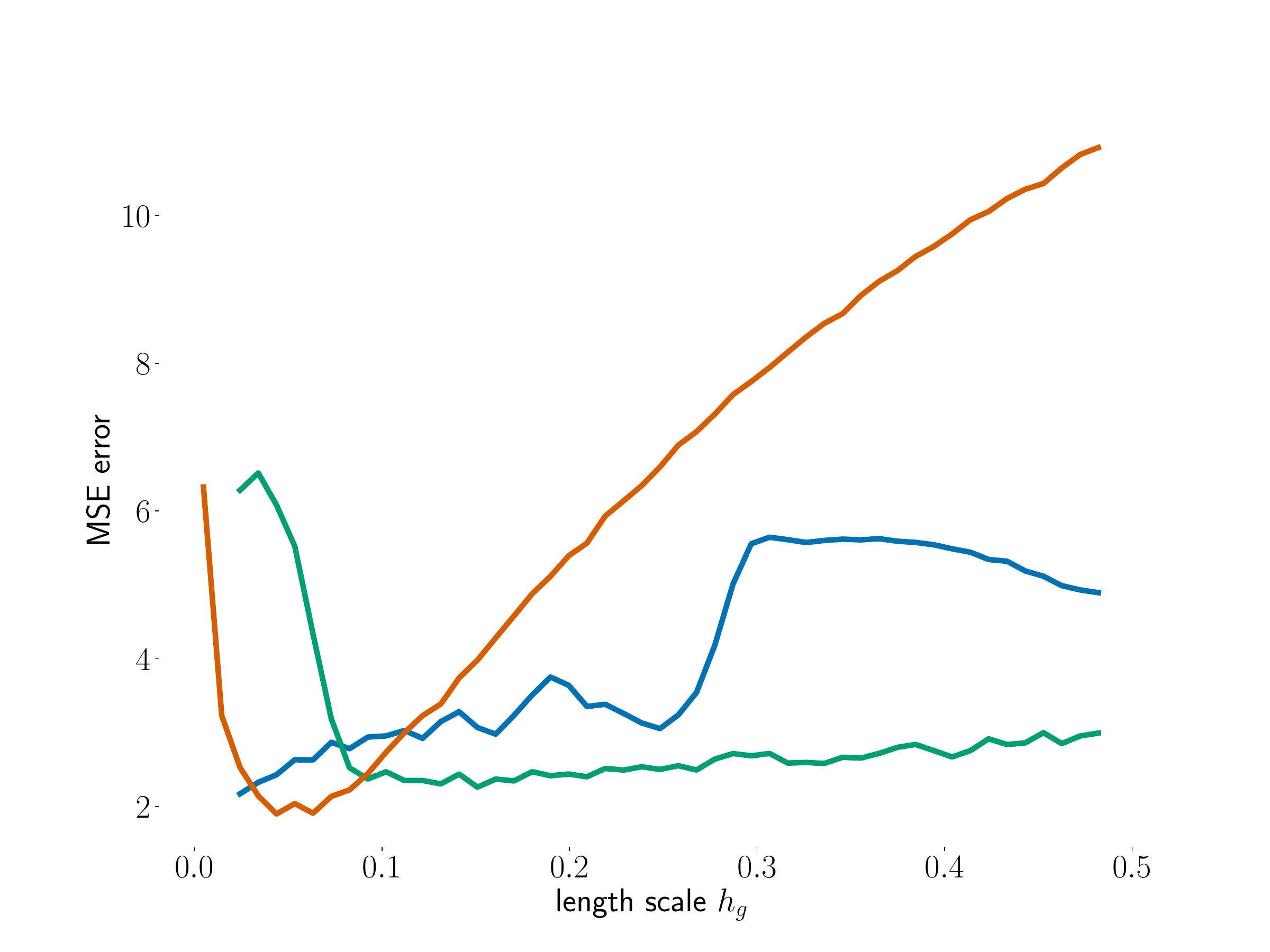}
        \caption{$m=1$, $\sigma^2 = 0.5$}\label{fig:fig:enw_b}
    \end{subfigure}
    \vfill
    \begin{subfigure}{0.49\textwidth}
        \centering
        \includegraphics[width=\textwidth]{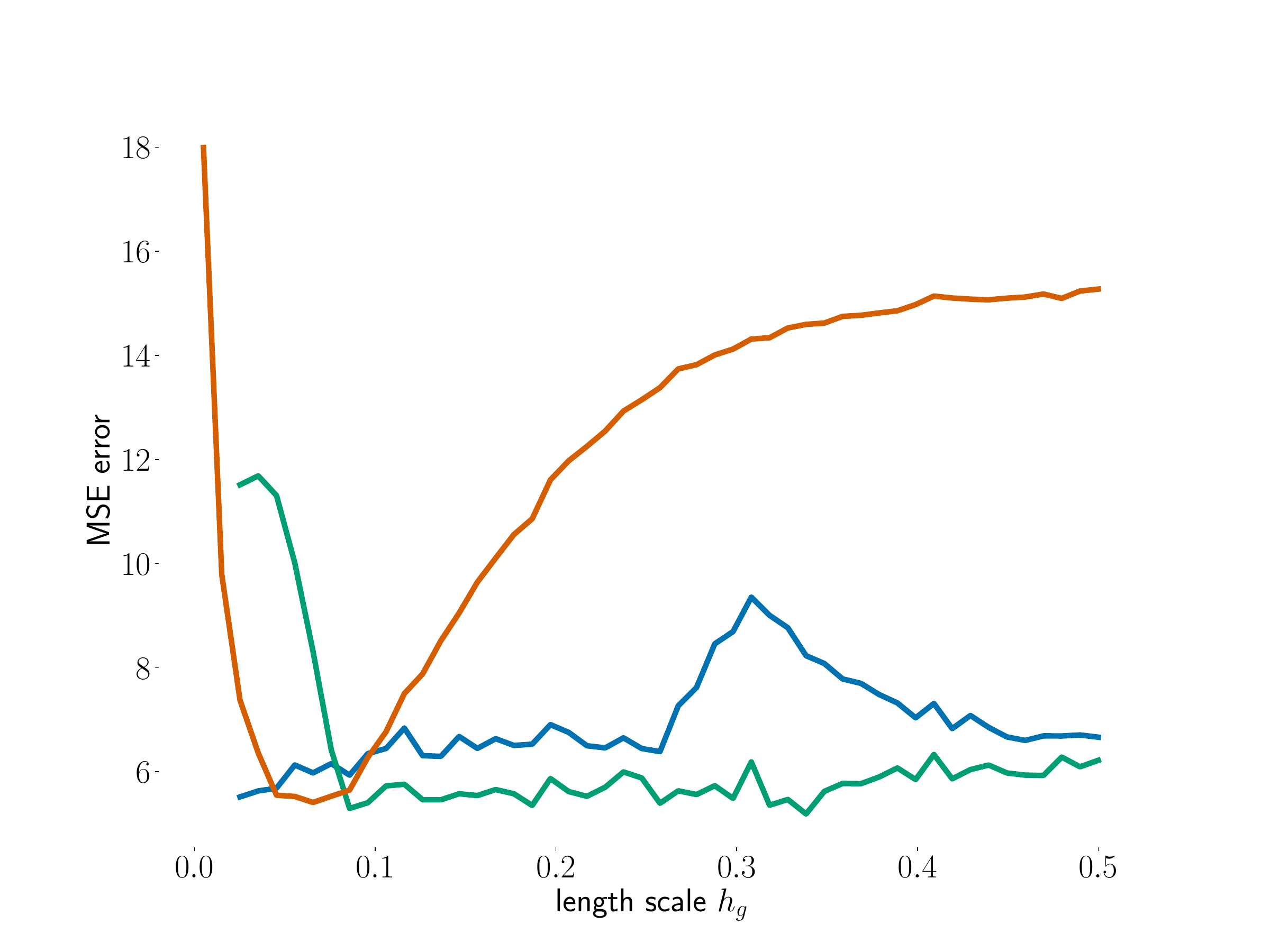}
        \caption{$m=2$, $\sigma^2 = 1.5$}\label{fig:enw_c}
    \end{subfigure}
    \hfill
    \begin{subfigure}{0.49\textwidth}
        \centering
        \includegraphics[width=\textwidth]{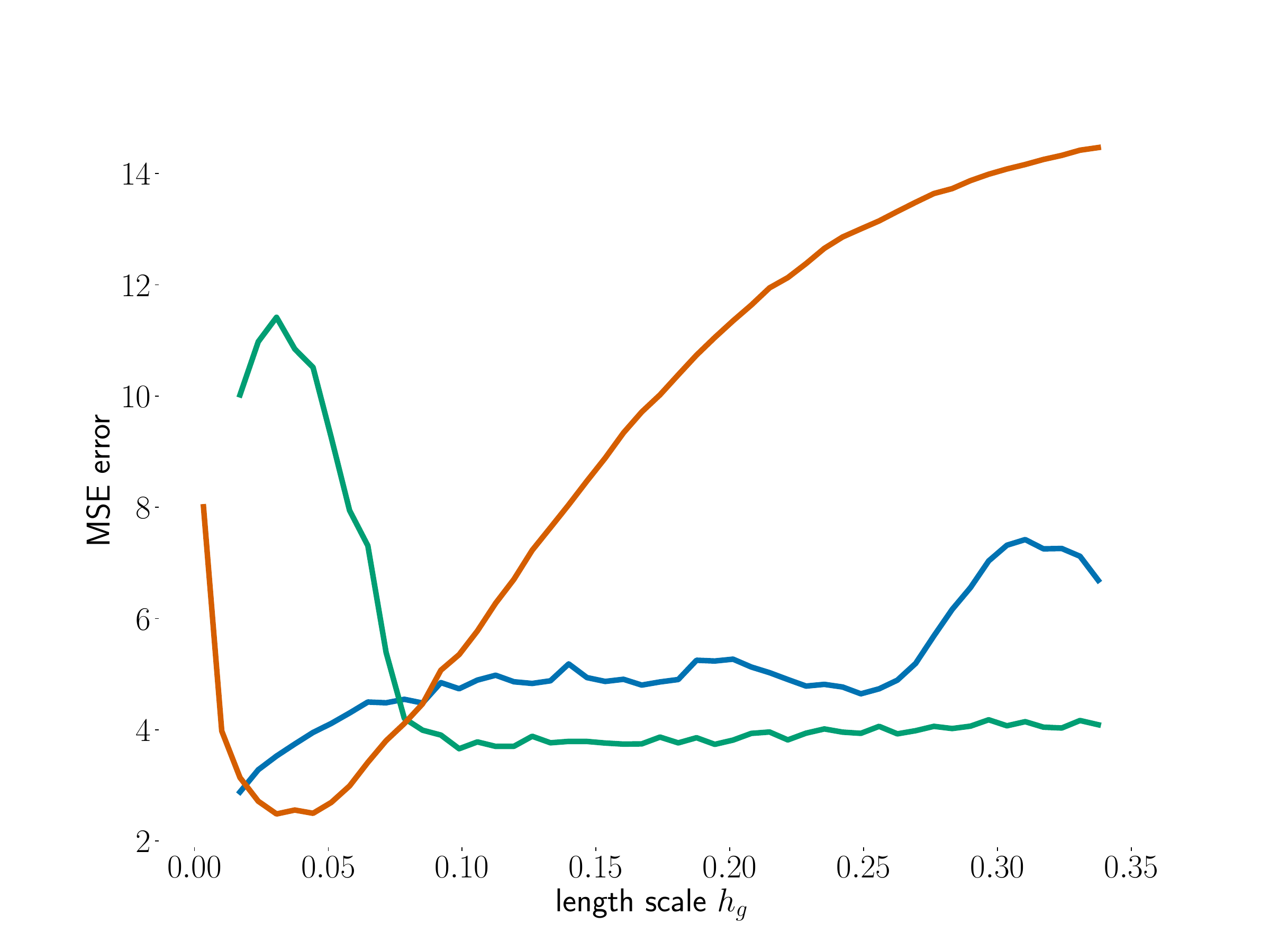}
        \caption{$m=2$, $\sigma^2 = 0.5$}\label{fig:enw_d}
    \end{subfigure}
    \caption{ 
    ``Bias-Variance'' tradeoff curves for $\mathcal{B}_{sp}$-ENW, $\mathcal{B}_{spectral}$-ENW and GNW, based on the 
        length-scale $h_g$. Parameter setup: $n=500$, $\mathrm{NUM}_{\MC}=20$, $\mathrm{NUM}_{\PTS}=50$. The frequency $m$ and 
        the label noise $\sigma^2$ vary as specified in the caption of the sub-figures.
   }\label{fig:enw_gnw_comp}
\end{figure}

We observe that 
when $h_g$ is close to $\tau_{\CV}$, the MSE error of GNW is generally lower than that of ENW, although for certain parameters (see Figure~\ref{fig:enw_a}), 
$\mathcal{B}_{sp}$-ENW and $\mathcal{B}_{spectral}$-ENW are competitive even in this scenario. In the large length-scale regime $\mathcal{B}_{spectral}$-ENW outperforms
the other algorithms by a significant margin, however $\mathcal{B}_{sp}$-ENW also shows significant improvement over GNW\@.
For the under-averaging regime, we generally observe that $\mathcal{B}_{sp}$-ENW is the dominant algorithm, provided that 
the label is sufficiently regular so that $\tau_{\CV}$ is above the connectivity threshold of the LPM\@.

\section{Conclusion}\label{sec:conclusion}

We have shown that in a LPM with kernel function~\eqref{eq:link_function}, GNW matches (up to multiplicative constant) the classical 
NW rate~\eqref{eq:NW_rate}. In particular, GNW is effective even in the extremely narrow length-scale regime, $h_g\ll \tau_{\star}$, as soon as 
$h_g\to 0$ and $nh_g^d\to\infty$ -\ the same assumptions needed on $\tau$ for asymptotic convergence of NW~\eqref{eq:NW}. Next, in Theorem~\ref{thm:perturbed_nw_thm},
we have shown that the Nadaraya-Watson estimator is stable with respect to perturbation of the design points, provided that these perturbations are not too large.
Using this result, we examined several papers from the literature on position recovery and we discussed the implications of their results for the node regression problem,
with a particular focus on optimal nonparametric rates. In order to  
construct an estimator that achieves optimal non-parametric rates for $a$-Hölder continuous regression function $f$,
it is sufficient to construct a position estimation algorithm $\mathcal{A}$ such that $p_{\tau_{\star}}(\mathcal{A},\mmat{X}_{n+1})\leq Cn^{-\frac{2a}{d+2a}}$, where $\tau_{\star}\sim n^{-\frac{1}{d+2a}}$.
This last question has not been treated in full generality in the literature, but rather it seems that it needs to be treated on a case-by-case basis: 
different algorithms work in different settings. We have the intuition that the results of position estimation
of~\citep{dani} based on the \emph{local, number of common neighbors approach} should be able to extend the optimality in $d>1$ in the under averaging regime. 
More detailed characterization of the possibility of obtaining standard rates~\eqref{eq:optimal_nw_rate} along 
with negative results in a minimax sense is left for a future work.

Empirically, we observe that if the ratio $h_g/\tau_{\star}$ is $\Theta(1)$ then GNW performs nearly optimally -\ indeed, in 
this case the risk GNW reaches the optimal NW-rate~\eqref{eq:optimal_nw_rate}. Additionally, 
ENW is to be preferred in the over-averaging regime $\frac{h_g}{\tau_{\star}}=\omega(1)$. For 
sparse graphs, the problem is not as clear. Existing results~\eqref{eq:breaking_omega_paper} imply convergence 
of position estimation in any regime of sparsity, however it seems that GNW is simply faster in this particular case.
 Due to its computational complexity ($\mathcal{O}(n)$) it should be preferred for extremely sparse graphs, 
specifically outside of the univariate case ($d>1$). Another case in which GNW might be a good choice is in the case where fast algorithm runtime is 
emphasized over the statistical quality of the result. Indeed, the computational cost of position estimation is often at least $\Omega(n^2)$, so that there is a tradeoff 
between runtime of the algorithm and its statistical performance.

\section*{Appendix: Proofs}
\renewcommand{\thesubsection}{\Alph{subsection}}

\subsection{GNW additional results}\label{variance_appendix}
Recall that in the analysis of GNW, unless explicitly stated otherwise, 
expectations are taken with respect to $\mmat{X}_{n}$, $\mmat{\mathcal{U}}$ and $\mvec{\epsilon}$.
In this section we provide details to support the theoretical results of Section~\ref{sec:GNW_concentration}.
In particular, we prove Proposition~\ref{prop:proxies}, Lemma~\ref{lemma:tech_1} and Lemma~\ref{lemma:variance_lower_bound}.
En route, we will compute the expectation of GNW explicitly. Being a quotient of two random variables, the exact value of $\mathbb{E}\left[\hat{f}_{\GNW}(\mvec{x})\right]$ may seem difficult to compute.
We are able to carry out this computation due to the decoupling trick, Lemma~\ref{lemma:decoupling}. 
\begin{lemma}\label{lemma_exp}
Recall that $R_i(\mvec{x})=\frac{1}{1+\sum_{j\neq i} a(\mvec{x},\mvec{x}_j)}$.
For all $i\in[n]$ we have
\begin{equation*}
    \mathbb{E}\left[R_i(\mvec{x})\right]=\frac{1-{(1-\frac{d(\mvec{x})}{n})}^n}{d(\mvec{x})}
\end{equation*}
\end{lemma}
\begin{proof}
Note that $R_i(\mvec{x})$, $i\in [n]$ are identically distributed. Therefore, for $i\in[n]$,
\begin{equation*}
    \mathbb{E}\left[R_i(\mvec{x})\right]=\mathbb{E}\left[R_1(\mvec{x})\right]
\end{equation*}
Recall Equation~\eqref{eqn_for_Ri}:
\begin{equation*}
    \sum_{i=1}^n a(\mvec{x},\mvec{x}_i)R_i(\mvec{x})=Z
\end{equation*}
Taking expectation and using the fact that $R_i(\mvec{x})$ and $a(\mvec{x},\mvec{x}_i)$ are independent, we get
\begin{equation}\label{eqn_for_ERi}
\begin{split}
    \mathbb{E}\left[\sum_{i=1}^n a(\mvec{x},\mvec{x}_i)R_i(\mvec{x})\right]&=\sum_{i=1}^n \mathbb{E}\left[a(\mvec{x},\mvec{x}_i)R_i(\mvec{x})\right]\\
    &=\sum_{i=1}^n \mathbb{E}\left[a(\mvec{x},\mvec{x}_i)\right]\mathbb{E}\left[R_i(\mvec{x})\right]\\
    &=nc(\mvec{x})\mathbb{E}\left[R_1(\mvec{x})\right]
\end{split}
\end{equation}
On the other hand,
\begin{equation}\label{eqn_for_Rempty}
    \mathbb{E}\left[Z\right]=\mathbb{P}\left[\sum_{i=1}^n a(\mvec{x},\mvec{x}_i)>0\right]=1-\mathbb{P}\left[\sum_{i=1}^n a(\mvec{x},\mvec{x}_i)=0\right]=1-{\left(1-c(\mvec{x})\right)}^n
\end{equation}
The result follows by combining Equations (\ref{eqn_for_Ri}), (\ref{eqn_for_ERi}) and (\ref{eqn_for_Rempty}).
\end{proof}

\begin{proposition}{(\textbf{Expectation of GNW})}\label{prop:expectation_comp}
    \begin{equation*}
        \mathbb{E}_{\mmat{X}_n,\mvec{\epsilon},\mmat{\mathcal{U}}}\left[\hat{f}_{\GNW}(\mvec{x})\right]=S(f,\mvec{x})\left(1-{\left(1-c(\mvec{x})\right)}^n\right)
    \end{equation*}
where $c(\mvec{x})$ and $S(f,\mvec{x})$ are given by~\eqref{eq:local_degree} and~\eqref{eq:operator_S}.
\end{proposition}

\begin{proof}{\textbf{of Proposition~\ref{prop:expectation_comp}}}
Recall the linearized expression for $\hat{f}_{\GNW}$~\eqref{eqn_for_gnw_Ri} i.e.\ we have

\begin{equation*}
    \hat{f}_{\GNW}(\mvec{x})=\sum_{i=1}^n y_i a(\mvec{x},\mvec{x}_i)R_i(\mvec{x})
\end{equation*}
Taking expectation and using Lemma~\ref{lemma_exp}, we get 

\begin{equation*}
\begin{split}
    \mathbb{E}\left[\hat{f}_{\GNW}(\mvec{x})\right]
    &=\sum_{i=1}^n\mathbb{E}\left[y_i a(\mvec{x},\mvec{x}_i)R_i(\mvec{x})\right]\\
    &=\sum_{i=1}^n \mathbb{E}\left[y_i a(\mvec{x},\mvec{x}_i)\right]\mathbb{E}\left[R_i(\mvec{x})\right]\\
    &=n\mathbb{E}[y_1a(\mvec{x},\mvec{x}_1)]\mathbb{E}\left[R_1(\mvec{x})\right]\\
    &=\frac{n(1-{(1-c(\mvec{x}))}^n)}{nc(\mvec{x})}\int f(\mvec{z})k(\mvec{x},\mvec{z})p(\mvec{z})dz\\
    &=S(f,\mvec{x})(1-{(1-c(\mvec{x}))}^n)
\end{split}
\end{equation*}
\end{proof}
Finally, the explicit computation of $\mathbb{E}[\hat{f}_{\GNW}]$ enables us to prove Proposition~\ref{prop:proxies}.
\begin{proof}{\textbf{of Proposition~\ref{prop:proxies}}}
In view of Proposition~\ref{prop:expectation_comp}, we have
\begin{equation}
\label{temp_lab_1}
    b(\mvec{x})-\Bias\left[\hat{f}_{\GNW}(\mvec{x})\right]=S(f,\mvec{x})-\mathbb{E}\left[\hat{f}_{\GNW}(\mvec{x})\right]=S(f,\mvec{x}){\left(1-\frac{d(\mvec{x})}{n}\right)}^n
\end{equation}
Next, we have
\begin{equation*}
    v(x)=\mathbb{E}\left[\hat{f}^2_{\GNW}(\mvec{x})\right]-2S(f,\mvec{x})\mathbb{E}\left[\hat{f}_{\GNW}(\mvec{x})\right]+S^2(f,\mvec{x})
\end{equation*}
Again, by using Proposition~\ref{prop:expectation_comp}, we get
\begin{equation}\label{temp_lab_2}
\begin{split}
    v(\mvec{x})-\Var(\hat{f}_{\GNW}(\mvec{x}))&={\left(\mathbb{E}\left[\hat{f}_{\GNW}(\mvec{x})\right]\right)}^2-2S(f,\mvec{x})\mathbb{E}\left[\hat{f}_{\GNW}(\mvec{x})\right]+S^2(f,\mvec{x})\\
    &={\left(S(f,\mvec{x})-\mathbb{E}\left[\hat{f}_{\GNW}(\mvec{x})\right]\right)}^2\\
    &=S^2(f,\mvec{x}){\left(1-\frac{d(\mvec{x})}{n}\right)}^{2n}
\end{split}
\end{equation}
The claim follows from Equations~\eqref{temp_lab_1} and\eqref{temp_lab_2} and the basic inequality $1-t\leq \exp(-t)$.
\end{proof}
Next, we prove Lemma~\ref{lemma:tech_1}.
\begin{proof}\label{proof_for_ref_1}{\textbf{ of Lemma~\ref{lemma:tech_1}}}
    We begin with a small lemma that also simplifies the notation for the actual calculation.
    
        \begin{lemma}\label{lemma:one_time_lemma}
            Suppose that $g\colon\mathbb{R}^{d+1}\to\mathbb{R}$ is a  measurable function such that for all $i\in [n]$, \\ $\mathbb{E}[g^2(\mvec{x}_i,\epsilon_i)]<\infty$. 
            For $i \in [n]$ set $F_i=g(\mvec{x}_i,\epsilon_i)$. 
            Then for all pairs of distinct indices $1\leq i,j \leq n$ we have
            \begin{equation*}
                \mathbb{E}\left[F_i F_j a(\mvec{x},\mvec{x}_i)a(\mvec{x},\mvec{x}_j)R_i(\mvec{x})R_j(\mvec{x})\right]=\mathbb{E}\left[F_i a(\mvec{x},\mvec{x}_i)\right]\mathbb{E}\left[F_j a(\mvec{x},\mvec{x}_j)\right]\mathbb{E}\left[{R^2_{\{i,j\}}(\mvec{x})}\right]
            \end{equation*}
            \end{lemma}
    
            \begin{proof}
            Using the decoupling trick~\ref{lemma:decoupling} we have
            \begin{equation*}
                F_i F_j a(\mvec{x},\mvec{x}_i)a(\mvec{x},\mvec{x}_j)R_i(\mvec{x})R_j(\mvec{x})=F_i F_j a(\mvec{x},\mvec{x}_i)a(\mvec{x},\mvec{x}_j){R^2_{\{i,j\}}(\mvec{x})}
            \end{equation*}
            and moreover $R_{\{i,j\}}(\mvec{x})$ is independent from $\left(\mvec{x}_i,\epsilon_i,a(\mvec{x},\mvec{x}_i)\right)$ and $\left(\mvec{x}_j,\epsilon_j,a(\mvec{x},\mvec{x}_j)\right)$. 
            Next, $\left(\mvec{x}_i,\epsilon_i,a(\mvec{x},\mvec{x}_i)\right)$ and $\left(\mvec{x}_j,\epsilon_j,a(\mvec{x},\mvec{x}_j)\right)$ are also independent by modeling assumptions\footnote{conditioning on $\mvec{x}_{n+1}=\mvec{x}$, i.e.\ 
            treating the latent position of node $(n+1)$ as a deterministic quantity is crucial in this part of the proof}. 
            As independent variables are uncorrelated, the conclusion follows.
            \end{proof}
    Set $g(\mvec{x}_i,\epsilon_i)=y_i-S(f,\mvec{x})$.
    Using Equation~\eqref{decomp}, we have
    
    \begin{equation}\label{tricky_eqn}
    \begin{split}
        \mathbb{E}\left[{\left(\hat{f}_{\GNW}(\mvec{x})-S(f,\mvec{x})Z\right)}^2\right]
        &=\mathbb{E}\left[{\left(\sum_{i=1}^n(y_i-S(f,\mvec{x}))a(\mvec{x},\mvec{x}_i)R_i(\mvec{x})\right)}^2\right]\\
        &=\sum_{i=1}^n\mathbb{E}\left[{\left(g(\mvec{x}_i,\epsilon_i)a(\mvec{x},\mvec{x}_i)R_i(\mvec{x})\right)}^2\right]\\
        &+\sum_{i\neq j}\mathbb{E}\left[ g(\mvec{x}_i,\epsilon_i)g(\mvec{x}_j,\epsilon_j)a(\mvec{x},\mvec{x}_i)a(\mvec{x},\mvec{x}_j)R_i(\mvec{x})R_j(\mvec{x}) \right]
    \end{split}
    \end{equation}
    For $i\neq j$, applying Lemma~\ref{lemma:one_time_lemma} with  $g\colon\mathbb{R}^{d+1}\to\mathbb{R}$ given by 
    $g(\cdot,\star)=(f(\cdot)+\star)-S(f,\mvec{x})$ along with the fact that $\mathbb{E}[g(\mvec{x}_i,\epsilon_i)a(\mvec{x},\mvec{x}_i)]=0$ gives
    
    \begin{equation}
    \label{dr_trick}
        \mathbb{E}\left[ g(\mvec{x}_i,\epsilon_i)g(\mvec{x}_j,\epsilon_j)a(\mvec{x},\mvec{x}_i)a(\mvec{x},\mvec{x}_j)R_i(\mvec{x})R_j(\mvec{x}) \right]=0
    \end{equation}
    Finally,
    \begin{equation*}\label{sr_trick}
    \begin{split}
        \sum_{i=1}^n \mathbb{E}\left[{\left(g(\mvec{x}_i,\epsilon_i)a(\mvec{x},\mvec{x}_i)R_i(\mvec{x})\right)}^2\right]
        &=\sum_{i=1}^n\mathbb{E}\left[{(y_i-S(f,\mvec{x}))}^2a(\mvec{x},\mvec{x}_i)R_i^2(\mvec{x})\right]
    \end{split}
    \end{equation*}
    \end{proof}
Finally, we conclude the variance analysis of GNW by a proof of Lemma~\ref{lemma:variance_lower_bound}.

\begin{proof}\textbf{of Lemma~\ref{lemma:variance_lower_bound}}
By Equation (\ref{eq:variance_decomp}), Lemma~\ref{prop:expectation_comp}, Lemma~\ref{lemma:tech_1}, as well as 
equation~\eqref{eq:eqn_for_clarity} 
and the basic inequality  $1-t\leq e^{-t}$ valid for all $t\geq 0$, we have

\begin{equation*}
\begin{split}
v(\mvec{x})&\geq \mathbb{E}\left[{\left(\hat{f}_{\GNW}(\mvec{x})-S(f,\mvec{x})Z\right)}^2\right]\\
& = \sum_{i=1}^n \mathbb{E}_{\mmat{\mathcal{U}}}\left[\mathbb{E}_{\mvec{\epsilon}}[{\left(y_i-S(f,\mvec{x}_i)\right)}^2]a(\mvec{x},\mvec{x}_i)R_i(\mvec{x})\right] \\
&\geq \sigma_0^2d(\mvec{x})\mathbb{E}\left[R_1^2(\mvec{x})\right]\\
&\geq \sigma_0^2d(\mvec{x}){\left(\mathbb{E}\left[R_1(\mvec{x})\right]\right)}^2\\
& = \sigma_0^2d(\mvec{x}){\left(\frac{{\left(1-(1-\frac{d(\mvec{x})}{n})\right)}}{d(\mvec{x})}\right)}^2\\
&\geq \sigma_0^2\frac{{\left(1-e^{-d(\mvec{x})}\right)}^2}{d(\mvec{x})}
\end{split}
\end{equation*}
\end{proof}

\subsection{Bias and Risk of GNW}\label{bias_risk_proofs}
\renewcommand{\thesubsection}{\Alph{subsection}}

\begin{proof}{\textbf{of Lemma~\ref{lemma:bias_control}}} 
    Our first claim is that under our assumptions, 
For $\mvec{x}\in Q$, we have $c(\mvec{x})=\int K\left(\frac{||\mvec{x}-\mvec{z}||}{h_g}\right)p(\mvec{z})dz>0$ and 
hence the operator $S(f,\mvec{x})$~\eqref{eq:operator_S} is non trivial. Indeed, suppose that Assumption~\ref{ass:K_1} holds. 
We will show that for every $\mvec{x}\in Q$, $c(\mvec{x})>0$ where $Q$ is the support of the distribution $p$.
Suppose that $c(\mvec{x})=0$.
Using Assumption~\ref{ass:K_1} we get 
\begin{equation*}
\begin{split}
    \alpha \int \mathbb{I}\left[||\mvec{x}-\mvec{z}||\leq M_1h_g\right] p(\mvec{z})dz &\leq 2\alpha \int \mathbb{I}\left[||\mvec{x}-\mvec{z}||\leq M_1h_g\right]K\left(\frac{||\mvec{x}-\mvec{z}||}{h_g}\right)p(\mvec{z})dz\\
    &\leq 2\alpha \int K\left(\frac{||\mvec{x}-\mvec{z}||}{h_g}\right)p(\mvec{z})dz\\
    &=2c(\mvec{x})\\
    &=0
\end{split}
\end{equation*}
As $\alpha >0$, $\mvec{x}\notin\supp(p)=Q$, and hence our claim follows form the contrapositive. 
Hence for $\mvec{x}\in Q$, we have
\begin{equation*}
\begin{split}
|S(f,\mvec{x})-f(\mvec{x})|&=\left|\frac{\int f(\mvec{z})K\left(\frac{||\mvec{x}-\mvec{z}||}{h_g}\right)p(\mvec{z})dz}{\int K\left(\frac{||\mvec{x}-\mvec{z}||}{h_g}\right)p(\mvec{z})dz}-f(\mvec{x})\right|\\
&=\left|\frac{\int f(\mvec{z})K\left(\frac{||\mvec{x}-\mvec{z}||}{h_g}\right)p(\mvec{z})dz}{\int K\left(\frac{||\mvec{x}-\mvec{z}||}{h_g}\right)p(\mvec{z})dz}-\frac{\int f(\mvec{x})K\left(\frac{||\mvec{x}-\mvec{z}||}{h_g}\right)p(\mvec{z})dz}{\int K\left(\frac{||\mvec{x}-\mvec{z}||}{h_g}\right)p(\mvec{z})dz}\right|\\
&=\left|\frac{\int_{Q}[f(\mvec{z})-f(\mvec{x})]K\left(\frac{||\mvec{x}-\mvec{z}||}{h_g}\right)p(\mvec{z})dz
}{\int_{Q} K\left(\frac{||\mvec{x}-\mvec{z}||}{h_g}\right)p(\mvec{z})dz}\right|\\
&\leq L\frac{ \int_{Q} ||\mvec{z}-\mvec{x}||^{\alpha }K\left(\frac{||\mvec{x}-\mvec{z}||}{h_g}\right)p(\mvec{z})dz}{
\int_{Q} K\left(\frac{||\mvec{x}-\mvec{z}||}{h_g}\right)p(\mvec{z})dz}\\
&\leq 2LM_2^{\alpha}h_g^{\alpha}
\end{split}
\end{equation*}
where we used Assumption~\ref{F_1} in the first and  Assumption~\ref{ass:K_1} in the second inequality.
\end{proof}

\begin{proof}{\textbf{of Lemma~\ref{lemma:local_degree_bound}}}
By Assumption~\ref{ass:K_1} and the assumption that $Q$ satisfies~\ref{ass:measure_retention} with parameters $(r_0,c_0)-$ we have 
\begin{equation*}
\begin{split}
\frac{c(\mvec{x})}{\alpha }&=\int K\left(\frac{||\mvec{x}-\mvec{z}||}{h_g}\right)p(\mvec{z})dz\\
&\geq \frac{1}{2}\int \mathbb{I}\left[||\mvec{x}-\mvec{z}||\leq M_1h_g\right]p(\mvec{z})dz\\
&\geq \frac{p_0(\mvec{x})}{2}\int \mathbb{I}\left[||\mvec{x}-\mvec{z}||\leq M_1h_g\right]\mathbb{I}\left[\mvec{z}\in Q\right]dz\\
&=\frac{p_0(\mvec{x})}{2}m(Q\cap B_{M_1h_g}(\mvec{x}))\\
&\geq \frac{p_0(\mvec{x})c_0}{2}m(B_{M_1h_g}(\mvec{x}))\\
&=c_0v_d M_1^d h_g^d p_0(\mvec{x})/2>0
\end{split}
\end{equation*}
The conclusion follows $d(\mvec{x})=nc(\mvec{x})$.
\end{proof}

\begin{proof}{\textbf{of Theorem~\ref{thm:pwriskthm} and Theorem~\ref{thm:final_result}}}
We use the bias and variance proxies to bound the risk via the following inequality   
    \begin{equation}
    \label{quasi_bias_variance}
    \mathcal{R}_g(\hat{f}_{\GNW}(\mvec{x}),f(\mvec{x}))\leq 2(v(\mvec{x})+b^2(\mvec{x}))
    \end{equation}
On one hand, from Lemma~\ref{lemma:bias_control} we see that under Assumptions~\ref{ass:K_1} and~\ref{F_1}, we have 
\begin{equation*}
|b(\mvec{x})|\leq 2LM_2^a h_g^a    
\end{equation*}
On the other hand, combining the bound in~\ref{thm:variance_theorem} along with Lemma~\ref{lemma:local_degree_bound}, and taking into 
account assumptions~\ref{ass:K_1},~\ref{ass:measure_retention} and Equation~\eqref{lbd}
we arrive at the following 
\begin{equation*}
    v(\mvec{x})\leq \frac{18B^2+4\sigma^2}{c_0v_d M_1^d n\alpha h_g^d p_0(\mvec{x})}
\end{equation*}
The conclusion for Theorem~\ref{thm:pwriskthm} follows form Equation~\eqref{quasi_bias_variance}.
Under Assumption~\ref{eq:density_ass}, Equation~\eqref{lbd} holds with $p_0(\mvec{x})\equiv p_0$. 
The conclusion for Theorem~\ref{thm:final_result} follows immediately from integrating the bound in Theorem~\ref{thm:pwriskthm}.
\end{proof}
Before we proceed with the proof of Theorem~\ref{thm:final_result_holder}, we need to show that the variance of GNW can not blow up.
The following lemma will be useful for ENW analysis as well.
\begin{lemma}\label{lemma:support_lemma}
    Suppose that $\epsilon_i$ satisfying Assumption~\ref{ass:additive_noise} and $0\leq w_i\leq 1$ are real numbers such that 
    $\sum_{i=1}^n w_i > 0$. Then
    \begin{equation}
        \mathbb{E}_{\mvec{\epsilon}}\left[{\left(\frac{\sum_{i=1}^n \epsilon_i w_i}{\sum_{i=1}^n w_i}\right)}^2\right]\leq \sigma^2 \min\{\frac{1}{\sum_{i=1}^n w_i},1\}
    \end{equation}
\end{lemma}

\begin{proof}
Note that 
\begin{equation}
    \mathbb{E}_{\mvec{\epsilon}}\left[{\left(\frac{\sum_{i=1}^n \epsilon_i w_i}{\sum_{i=1}^n w_i}\right)}^2\right]\leq\sigma^2\frac{\sum_{i=1}^n w_i^2}{{\left(\sum_{i=1}^n w_i\right)}^2}
\end{equation}
Using $w_i\leq 1$, we get the first inequality. The second inequality easily follows from the observation that $0\leq \frac{w_i}{\sum_{i=1}^n w_i}\leq 1$. 
Namely setting $v_i = \frac{w_i}{\sum_{i=1}^n w_i}$,
we have $v_i\geq 0$ and $\sum_{i=1}^n v_i = 1$, and hence $v_i\leq 1$, for all $i\in[n]$. Now
\begin{equation*}
    \frac{\sum_{i=1}^n w_i^2}{{(\sum_{i=1}^n w_i)}^2} = \sum_{i=1}^n v^2_i\leq \sum_{i=1}^n v_i = 1
\end{equation*}
Concluding the proof.
\end{proof}

\begin{proof}{\textbf{of Theorem~\ref{thm:final_result_holder}}}
    Using Lemma~\ref{lemma:support_lemma}, we get 
    \begin{equation*}
        \mathcal{R}_g\left(\hat{f}_{\GNW}(\mvec{x}),f(\mvec{x})\right)\leq 2\min\{b^2(\mvec{x})+v(\mvec{x}),4B^2+\sigma^2\}
    \end{equation*}
The idea is to split the integral in the global risk~\eqref{eq:node_reg_risk_global} in two parts, the first where 
the density is sufficiently high i.e.\ $p(\mvec{x})\geq 2SM_1^{b}h_g^{b}$, where we use the bounds from Theorem~\ref{thm:pwriskthm}
and the second, where the density is low and on which we use the bound $8B^2+2\sigma^2$.
From Assumption~\ref{ass:hcd_condition} we have 
\begin{equation*}
    \inf\limits_{\substack{\mvec{z}\in Q\\||\mvec{x}-\mvec{z}||\leq M_1h_g}} p(\mvec{z})\geq p(\mvec{x})-SM_1^{b}h_g^{b}
\end{equation*}
and hence, when $p(\mvec{x})\geq 2SM_1^{b}h_g^b$, we have 

\begin{equation*}
    \inf\limits_{\substack{\mvec{z}\in Q\\||\mvec{x}-\mvec{z}||\leq M_1h_g}} p(\mvec{z})\geq SM_1^{b}h_g^{b}
\end{equation*}
Therefore, for $\mvec{x}$ such that $p(\mvec{x})\geq 2SM_1^{b}h_g^b$, we have that Theorem~\ref{thm:pwriskthm} is satisfied with $p_0(\mvec{x})=SM_1h_g^b$.
We conclude as follows.
\begin{equation*}
\begin{split}
   \mathcal{R}_g\left(\hat{f}_{\GNW},f\right)&=\int\mathcal{R}_g\left(\hat{f}_{\GNW}(\mvec{x}),f(\mvec{x})\right)p(\mvec{x})dx\\
    &\leq 2\int\left(v(\mvec{x})+b^2(\mvec{x})\right)\mathbb{I}\left[p(\mvec{x})\geq 2 S M_1^{b}h_g^{b}\right] p(\mvec{x})dx\\
    &+(8B^2+2\sigma^2)\int \mathbb{I}\left[p(\mvec{x})\leq 2 S M_1^{b}h_g^{b}\right] p(\mvec{x})dx\\
    &\leq 4L^2M_2^{2a}h_g^{2a}+\frac{36B^2+8\sigma^2}{c_0v_d S M_1^{d+b}n\alpha h_g^{d+b}}\\
    &+(8B^2+2\sigma^2)S^{1/2}M_1^{b/2}h_g^{b/2}\int p^{1/2}(\mvec{x})dx 
\end{split}    
\end{equation*}
\end{proof}

\subsection{Estimated Nadaraya-Watson}
\renewcommand{\thesubsection}{\Alph{subsection}}

\begin{proof}{\textbf{of Theorem~\ref{thm:perturbed_nw_thm}}}
Using Assumption~\ref{ass:phi_conditions}, we have 
\begin{equation}\label{eq::nw_thm_eq_1}
    \sum_{i=1}^n \phi(\tilde{\delta}_i/\tau)\geq \frac{1}{2}\sum_{i=1}^n \mathbb{I}( \tilde{\delta}_i \leq M_1\tau)
\end{equation}
Next, using the assumption that $\Delta(\mathcal{A},\mmat{X}_{n+1})\leq\frac{M_1\tau}{2}$, for all $i\in [n]$ we have 
\begin{equation}\label{eq::nw_thm_eq_2}
\frac{\tilde{\delta}_i}{\tau} = \frac{\delta_i}{\tau}+\frac{\tilde{\delta}_i-\delta_i}{\tau}\leq \frac{\delta_i}{\tau}+\frac{M_1}{2}        
\end{equation}
Hence if $\delta_i/\tau \leq M_1/2$ then $\tilde{\delta}_i/\tau \leq  M_1/2+M_1/2 = M_1$.
Consequently, using equation~\eqref{eq::nw_thm_eq_1}, we get that 
\begin{equation}\label{eq::nw_thm_eq_3}
    \sum_{i=1}^n \phi(\tilde{\delta}_i/\tau)\geq \frac{1}{2}M(\tau)>0
\end{equation}
In particular, under the assumptions of Theorem~\ref{thm:perturbed_nw_thm}, we do not need to worry about the degenerate case 
$\sum_{i=1}^n \phi(\tilde{\delta}_i/\tau)=0$, in which 
we assign value $0$ by default to the estimator.
We have 

\begin{align*}\
    \hat{f}^{\mathcal{A}}_{\ENW,\tau}(\mvec{x}_{n+1})&=\frac{\sum_{i=1}^n y_i\phi\left(\frac{\tilde{\delta}_i}{\tau}\right)}{\sum_{i=1}^n\phi\left(\frac{\tilde{\delta}_i}{\tau}\right)}\\
    & = \frac{\sum_{i=1}^n f(\mvec{x}_i)\phi\left(\frac{\tilde{\delta}_i}{\tau}\right)}{\sum_{i=1}^n\phi\left(\frac{\tilde{\delta}_i}{\tau}\right)}+\frac{\sum_{i=1}^n \epsilon_i\phi\left(\frac{\tilde{\delta}_i}{\tau}\right)}{\sum_{i=1}^n\phi\left(\frac{\tilde{\delta}_i}{\tau}\right)}\\
    & = f(\mvec{x}_{n+1})+\frac{\sum_{i=1}^n (f(\mvec{x}_i)-f(\mvec{x}_{n+1}))\phi\left(\frac{\tilde{\delta}_i}{\tau}\right)}{\sum_{i=1}^n\phi\left(\frac{\tilde{\delta}_i}{\tau}\right)} \numberthis\label{eq:enw_signal_part} \\
    & + \frac{\sum_{i=1}^n \epsilon_i\phi\left(\frac{\tilde{\delta}_i}{\tau}\right)}{\sum_{i=1}^n\phi\left(\frac{\tilde{\delta}_i}{\tau}\right)} \numberthis\label{eq:enw_noise_part}
\end{align*}
We take care of the two terms separately. If $\phi(\tilde{\delta}_i/\tau)>0$, then from Assumption~\ref{ass:phi_conditions} we get $\tilde{\delta}_i\leq M_2\tau$, so that equation~\eqref{eq::nw_thm_eq_2} implies

\begin{equation}\label{eq::nw_thm_after_align_2}
    \delta_i^a\phi(\tilde{\delta}_i/\tau)\leq{\left(\tilde{\delta}_i+\frac{M_1\tau}{2}\right)}^a \phi(\tilde{\delta}_i/\tau)\leq {(M_2+M_1/2)}^a\tau^a \phi(\tilde{\delta}_i/\tau)
\end{equation}
Finally, Equation~\eqref{eq::nw_thm_after_align_2} yields the following bound on~\eqref{eq:enw_signal_part}
\begin{align*}
    \abs{\frac{\sum_{i=1}^n \left(f(\mvec{x}_i)-f(\mvec{x}_{n+1})\right)\phi\left(\frac{\tilde{\delta}_i}{\tau}\right)}{\sum_{i=1}^n\phi\left(\frac{\tilde{\delta}_i}{\tau}\right)}}
    &\leq L\frac{\sum_{i=1}^n \delta_i^a\phi\left(\frac{\tilde{\delta}_i}{\tau}\right)}{\sum_{i=1}^n\phi\left(\frac{\tilde{\delta}_i}{\tau}\right)} \\
    &\leq L{(M_2+M_1/2)}^a\tau^a \numberthis\label{eq::nw_thm_after_align_bound}
\end{align*}
In order to bound the expression~\eqref{eq:enw_noise_part}, 
we apply Lemma~\ref{lemma:support_lemma} with $w_i = \phi(\tilde{\delta}_i/\tau)$, yielding 

\begin{equation}\label{eq::nw_thm_after_align_3}
\mathbb{E}_{\epsilon}\left[{\left(\frac{\sum_{i=1}^n \epsilon_i\phi\left(\frac{\tilde{\delta}_i}{\tau}\right)}{\sum_{i=1}^n\phi\left(\frac{\tilde{\delta}_i}{\tau}\right)}\right)}^2\right]\leq \frac{\sigma^2}{\sum_{i=1}^n \phi(\tilde{\delta}_i/\tau)}\leq \frac{2\sigma^2}{M(\tau)}
\end{equation}
where we have used~\eqref{eq::nw_thm_eq_3} in the last inequality. We get the claimed result by combining 
equations~\eqref{eq::nw_thm_after_align_bound} and~\eqref{eq::nw_thm_after_align_3} with the basic inequality ${(a+b)}^2\leq 2(a^2+b^2)$. 
\end{proof}

\begin{proof}{\textbf{of Theorem~\ref{thm_enw_fixed_design}}}
Our strategy is to analyze two cases separately: when the position recovery algorithm approximates the latent distances
within precision $M_1\tau/2$ and when it fails to do so.
We introduce the notation 
\begin{equation}\label{eq:algo_fail_shorthand}
    S_{\mathcal{A}} = \{\mathrm{\Delta}(\mathcal{A}(\mmat{A}),\mmat{X}_{n+1})\leq \frac{M_1\tau}{2}\}
\end{equation}
for the event that indicates success of the algorithm $\mathcal{A}$, and $S^c_{\mathcal{A}}$ for its complement.
We have 

\begin{align*}
    \mathbb{E}_{\mmat{\mathcal{U}},\mmat{\epsilon}}\left[{\left(\hat{f}^{\mathcal{A}}_{\ENW,\tau}(\mvec{x}_{n+1})-f(\mvec{x}_{n+1})\right)}^2\right]&=\mathbb{E}_{\mmat{\mathcal{U}},\mmat{\epsilon}}\left[{\left(\hat{f}^{\mathcal{A}}_{\ENW,\tau}(\mvec{x}_{n+1})-f(\mvec{x}_{n+1})\right)}^2 \mathbb{I}(S_{\mathcal{A}}) \right]\\
    &+\mathbb{E}_{\mmat{\mathcal{U}},\mmat{\epsilon}}\left[{\left(\hat{f}^{\mathcal{A}}_{\ENW,\tau}(\mvec{x}_{n+1})-f(\mvec{x}_{n+1})\right)}^2\mathbb{I}(S^c_{\mathcal{A}})\right]
\end{align*}
When the position recovery algorithm $\mathcal{A}$ estimates the latent distances $\mvec{\delta}$  
with precision $\frac{M_1\tau}{2}$, then the conditions of Theorem~\ref{thm:perturbed_nw_thm} are satisfied. 
Hence we have
\begin{align*}
    \mathbb{E}_{\mmat{\mathcal{U}},\mvec{\epsilon}}\left[{\left(\hat{f}^{\mathcal{A}}_{\ENW,\tau}(\mvec{x}_{n+1})-f(\mvec{x}_{n+1})\right)}^2\mathbb{I}\left[S_{\mathcal{A}}\right]\right]
    &=\mathbb{E}_{\mmat{\mathcal{U}}}\left[\mathbb{E}_{\mvec{\epsilon}}{\left(\hat{f}^{\mathcal{A}}_{\ENW,\tau}(\mvec{x}_{n+1})-f(\mvec{x}_{n+1})\right)}^2\mathbb{I}(S_{\mathcal{A}})\right]\\
    &\leq\left(C_1\tau^{2a}+\frac{12\sigma^2}{M(\tau)}\right)\mathbb{E}_{\mmat{\mathcal{U}}}\left[\mathbb{I}(S_{\mathcal{A}})\right]\\
    &\leq C_1\tau^{2a}+\frac{12\sigma^2}{M(\tau)}\numberthis\label{eq:bound_A_success}
\end{align*}
Next, we need to analyze what happens when the position recovery algorithm $\mathcal{A}$ fails. 
The idea will be to average over the additive noise of the label first, $\mvec{\epsilon}$.
In this case, we want to show that $\mathbb{E}_{\mvec{\epsilon}}{\left[\left(\hat{f}^{\mathcal{A}}_{\ENW,\tau}(\mvec{x}_{n+1})-f(\mvec{x}_{n+1})\right)\right]}^2$ remains bounded.
If $\sum_{i=1}^n \phi(\tilde{\delta}_i/\tau)=0$, then by definition $\hat{f}^{\mathcal{A}}_{\ENW,\tau}(\mvec{x}_{n+1})=0$ and 
\begin{equation*}
    \mathbb{E}_{\mvec{\epsilon}}\left[{\left(\hat{f}^{\mathcal{A}}_{\ENW,\tau}(\mvec{x}_{n+1})-f(\mvec{x}_{n+1})\right)}^2\right]\leq \max_{\mvec{x}\in Q} |f(\mvec{x})|^2 \leq B^2
\end{equation*}
Otherwise, we have $\sum_{i=1}^n \phi(\tilde{\delta}_i/\tau) > 0$ and Lemma~\ref{lemma:support_lemma} yields 
\begin{equation*}
    \mathbb{E}_{\mvec{\epsilon}}\left[{\left(\hat{f}^{\mathcal{A}}_{\ENW,\tau}(\mvec{x}_{n+1})-f(\mvec{x}_{n+1})\right)}^2\right]\leq 4B^2+2\sigma^2
\end{equation*}
Finally,
\begin{align*}
    \mathbb{E}_{\mmat{\mathcal{U}},\mvec{\epsilon}}\left[{\left(\hat{f}^{\mathcal{A}}_{\ENW,\tau}(\mvec{x}_{n+1})-f(\mvec{x}_{n+1})\right)}^2\mathbb{I}(S^c_{\mathcal{A}})\right]&=
    \mathbb{E}_{\mmat{\mathcal{U}}}\left[\mathbb{E}_{\mvec{\epsilon}}{\left(\hat{f}^{\mathcal{A}}_{\ENW,\tau}(\mvec{x}_{n+1})-f(\mvec{x}_{n+1})\right)}^2\mathbb{I}(S^c_{\mathcal{A}})\right]\\
    &\leq\left(4B^2+2\sigma^2\right)\mathbb{E}_{\mmat{\mathcal{U}}}\left[\mathbb{I}(S^c_{\mathcal{A}})\right]\\
    &=\left(4B^2+2\sigma^2\right)p_{\tau}(\mathcal{A},\mmat{X})\numberthis\label{eq:bound_A_failure}
\end{align*}
Combining Equations~\eqref{eq:bound_A_success} and~\eqref{eq:bound_A_failure} we get the claimed result.
\end{proof}

\begin{proof}{\textbf{of Corollary~\ref{col:arias-castro-col}}} We will investigate for which parameters $h_g$, $a$ and $d$ 
    the result~\eqref{eq:arias_castro_proper_kernel} yields standard nonparametric rates for $\mathcal{A}_{sp}$-ENW\@.
    An i.i.d.\ sample $\mmat{X}_n$ with distribution $p$ satisfying 
    Assumptions~\ref{ass:measure_retention} and~\ref{eq:density_ass} will satisfy $\epsilon = {(\frac{\log(n)}{n})}^{1/d}$ with overwhelming 
    probability\footnote{This can be shown by a covering number argument.}  over $\mmat{X}_n$. 
    Since the probability $p_{\tau}(\mathcal{A},\mmat{X}_{n+1})$~\eqref{algo_rec_prob_failure} is decreasing in $\tau$, we get that for 
    \begin{equation}\label{eq:arias-castro-sp-bound-general}  
        \tau\geq C_2\left(h_g+{\left(\frac{\epsilon}{h_g}\right)}^{\frac{1}{1+A}}\right)
    \end{equation}
    we have $p_{\tau}(\mathcal{A}_{sp},\mmat{X}_{n+1})\leq \frac{1}{n}$.
    Specifically, if we want to achieve NW optimality for $\tau_{\star}=c_{\star}n^{-\frac{1}{d+2a}}$, we need $\tau_{\star}$ to satisfy
    inequality~\eqref{eq:arias-castro-sp-bound-general}. This yields $h_g\lesssim \tau_{\star}$ and ${(\frac{\epsilon}{h_g})}^{\frac{1}{1+A}}\lesssim \tau_{\star}$,
    which further limits the interval of admissible length-scales $h_g$:

    \begin{equation*}
        \frac{\epsilon}{\tau_{\star}^{1+A}}\lesssim h_g\lesssim \tau_{\star}
    \end{equation*}
    In order for this interval to be non-empty we need $\epsilon\lesssim \tau_{\star}^{2+A}$.
    Hence, we need to have

    \[\frac{\log(n)}{n}\lesssim n^{-\frac{d(2+A)}{d+2a}}\]
    Keeping in mind that we are constrained to $0<a\leq 1$, this 
    yields
    \begin{equation*}
        d(2+A) < d+2a
    \end{equation*}
    which is only possible for $d=1$ and $A>1/2$. 
    Conversely, it is easy to check that when $d=1$ and $a>(1+A)/2$, and $h_g$ is as in the assumption,
    $\tau_{\star}$ satisfies Equation~\eqref{eq:arias-castro-sp-bound-general}, which concludes the theorem.

\end{proof}

\begin{proof}{\textbf{of Corollary~\ref{col:dani-corollary}}}
    After some simple calculations, it is easy to see that for the specified values $h_g$, we have
    \begin{equation*}
        \mathrm{D}\left(\mathcal{B}_{rgg},\mmat{X}_{n+1}\right)\leq \frac{M_1\tau_{\star}}{2}=c_{\star}n^{-\frac{1}{2a+d}}    
    \end{equation*}
    and hence $p_{\tau_{\star}}(\mathcal{B}_{rgg},\mmat{X}_{n+1})=0$ (with high probability over
    the drawn points $\mmat{X}_{n+1}$) i.e. $\mathcal{B}_{rgg}$-ENW achieves optimal rates for $(h_g,\tau_{\star})$.

\end{proof}

\bibliography{references}

\end{document}